\documentclass[twoside,11pt]{article}
\usepackage{jmlr2e}
\usepackage{reductions,natbib,eucal,graphics,ifpdf,bm,colortbl,rotating}
\usepackage{graphicx}
\usepackage[a4paper]{geometry}
\usepackage{ifpdf}
\ifpdf
\fi

\ifx\theorem\undefined
\newtheorem{theorem}{Theorem}
\newtheorem{lemma}[theorem]{Lemma}

\newtheorem{corollary}[theorem]{Corollary}
\newtheorem{definition}[theorem]{Definition}

\fi

\ifx\BlackBox\undefined
\newcommand{\BlackBox}{\rule{1.5ex}{1.5ex}}  % end of proof
\fi

\ifx\proof\undefined
\newenvironment{proof}{\par\noindent{\bf Proof\ }}{\hfill\BlackBox\\[2mm]}
\fi

\newcommand{\Br}{{B}}    % Had previously distinguished
\newcommand{\Lr}{{L}}
\newcommand{\BB}{\mathbb{B}}

\newcommand{\EE}{\mathbb{E}}

\newcommand{\II}{\mathbb{I}}
\newcommand{\JJ}{\mathbb{J}}
\newcommand{\LL}{\mathbb{L}}
\newcommand{\PP}{\mathbb{P}}
\newcommand{\RR}{\mathbb{R}}
\newcommand{\RRbar}{\overline{\RR}}

\newcommand{\rv}[1]{\mathsf{#1}}
\newcommand{\Ssf}{\rv{S}}
\newcommand{\Xsf}{\rv{X}}
\newcommand{\Ysf}{\rv{Y}}

\newcommand{\Scal}{\mathcal{S}}
\newcommand{\Wb}{\overline{W}}

\newcommand{\KL}{\mathrm{KL}}

\newcommand{\inner}[2]{\left\langle #1,#2 \right\rangle}
\newcommand{\csiszar}{\Diamond}
\newcommand{\cdual}[1]{{#1}^{\csiszar}}
\newcommand{\lf}{\star}
\newcommand{\lfdual}[1]{{#1}^{\lf}}
\newcommand{\ri}{\mathrm{ri}}

\newcommand{\etimes}{\otimes}

\newcommand{\hh}{\hat{h}}
\newcommand{\heta}{{\hat{\eta}}}
\newcommand{\minimal}[1]{\underline{#1}}
\newcommand{\minL}{\minimal{L}}
\newcommand{\minLL}{\minimal{\LL}}
\newcommand{\SI}{\Delta\minLL}

\newcommand{\sgn}{\operatorname{sgn}}
\newcommand{\aco}{\operatorname{aco}}
\newcommand{\abold}{\mathbf{a}}

\newcommand{\E}[2]{\EE_{#1}\left[ #2 \right]}

\newcommand{\obs}{M}	% Distribution over observations
\newcommand{\ROC}{\mathrm{ROC}}
\newcommand{\AUC}{\mathrm{AUC}}

\title{Information, Divergence and Risk for Binary Experiments} 
\author{Mark D. Reid
\email{Mark.Reid@anu.edu.au}\\
\addr{Australian National University}\\ 
\addr{Canberra ACT 0200, Australia} 
\AND
Robert C. Williamson
\email{Bob.Williamson@anu.edu.au}\\
\addr{Australian National University and NICTA}\\ 
\addr{Canberra ACT 0200, Australia} 
}
\editor{\ }

\begin{document}

\maketitle

\begin{abstract}
We unify $f$-divergences, Bregman divergences, surrogate loss
bounds (regret bounds), proper scoring rules, matching losses, cost curves, 
ROC-curves and information.  We
do this by systematically studying integral and variational representations of
these objects and in so doing identify their primitives  which all are
related to cost-sensitive binary classification.  As well as clarifying
relationships between generative and discriminative views of learning, the new
machinery leads to tight and more general surrogate loss bounds and generalised
Pinsker inequalities relating $f$-divergences to variational divergence.
The new viewpoint illuminates existing algorithms:
it provides a new derivation of Support Vector Machines in terms of divergences
and relates Maximum Mean Discrepancy to Fisher Linear Discriminants.  It also
suggests new techniques for estimating $f$-divergences.
\end{abstract}

%%%%%%%%%%%%%%%%%%%%%%%%%%%%%%%%%%%%%%%%%%%%%%%%%%%%%%%%%%%%%%%%%%%%%%
\section{Introduction}

Machine learning problems often concern binary experiments. 
There it is assumed that observations are drawn
from a mixture of two distributions (one for each class). These
distributions determine many important objects related to the learning problems
they underpin such as risk, divergence and information.  Our aim in this paper
is to present all of these objects in a coherent framework explaining exactly
how they relate to each other.

%%%%%%%%%%%%%%%%%%%%%%%%%
\subsection{Motivation}
There are many different notions that underpin the definition of machine
learning problems. These include information, loss, risk, regret, ROC curves
and the area under them, matching loss functions, Bregman divergences and
distance or divergence between probability distributions. On the surface, the
problem of estimating whether two distributions are the same (as measured by,
say, their Kullback-Leibler divergence) is different to the
minimisation of expected risk in a prediction problem. One of the purposes of
the present paper is to show how this superficial difference is indeed only
superficial --- deeper down they are the same problem and analytical and
algorithmic insights for one can be transferred to the other.

Machine learning as a engineering discipline is still in its 
infancy\footnote{\cite{Bousquet06} has articulated
the need for an agreed vocabulary, a clear statement of the main problems, and
to ``revisit what has been done or discovered so far with a fresh
look''.}. There is
no agreed language to describe  machine learning problems (such is usually
done with an informal mixture of English and mathematics). There is very little
in the way of composability of machine learning solutions. That is, given the
solution to one problem, use it to solve another. Of course one would like to
not merely be able to do this, but to be certain what one might lose in doing
so. In order to do that, one needs to be able to provide theoretical guarantees
on how well the original problem will be solved by solving the surrogate
problem. Related to these issues is the fact that there are no well understood
{\em primitives} for machine learning. Indeed, what does that even mean? All of
these issues are the underlying motivation for this paper.

Our long term goal (towards which this paper is but the first step) is to turn
the field of machine learning into a more well founded engineering discipline
with an agreed language and well understood composition rules. Our motivation
is that until one can start building systems modularly, one is largely
restricted to starting from scratch for each new problem, rather than obtaining
the efficiency benefits of re-use\footnote{
    \cite{AbelsonSussmanSussman1996} described the principles of constructing
    software with the aid of 
	\cite[Chapter 12, paragraph 1]{Locke1690}:
    \begin{quote}
	The acts of the mind, wherein it exerts its power over simple ideas,
	are chiefly these three: (1) {Combining} several {simple ideas} into
	one compound one; and thus all \emph{complex ideas} are made.  (2) The
	second is bringing two ideas, whether simple or complex, together, and
	setting them by one another, so as to take a view of them at once,
	without uniting them into one; by which it gets all its \emph{ideas of
	relations}.  (3) The third is separating them from all other ideas that
	accompany them in their real existence; this is called {abstraction}:
	and thus all its \emph{general ideas} are made 
    \end{quote}
    Modularity is central to computer hardware  \citep{BaldwinClarkxx,
    BaldwinClark06} and other engineering disciplines 
    \citep{GershensonPrasadZhang2003}.
}.

We are comparing problems, not solutions or algorithms. Whilst there have been
attempts to provide a degree of unification at the level of algorithms
\citep{AltunSmola2006}, there are intrinsic limits to such a research program.
The most fundamental is that (surprisingly!) there is no agreed formal
definition of what an algorithm really is, nor how two algorithms can be
compared with a view to determining if they are the same
\citep{BlassGurevich2003}.

We have started with binary experiments because they are simple and widely
used.  As we will show, by pursuing the high level research agenda summarised
above, we have managed to unify all of the disparate concepts mentioned and
furthermore have simultaneously simplified and generalised two fundamental
results: Pinsker inequalities between $f$-divergences and surrogate-loss regret
bounds. The proofs of these new results rely essentially on the decomposition
into primitive problems.

%%%%%%%%%%%%%%%%%%
\subsection{Novelty and Significance}

Our initial goal was to present existing material in a unified way.  We
have indeed done that. In doing so we have developed new (and simpler) proofs
of existing results. Additionally we have developed some novel technical
results:
1) a link between the weighted integral 
representations for proper scoring rules and those for $f$-divergences; 2) a
unified derivation of the integral representations in terms of Taylor series; 3)
use of these representations to derive new bounds for divergences, Bayes
risks and regrets (``surrogate loss bounds'' and Pinsker inequalities); 
4) showing that statistical information (and hence $f$-divergence)
are both Bregman informations; 5) showing connections between variational
representation of risks and divergences; 6) the derivation of SVMs from a
variational perspective; 7) results relating AUC (Area under the ROC Curve) to
divergences.

The significance of these new connections is that they show that the choice of
loss function (scoring rule), $f$-divergence and Bregman divergence (regret)
are intimately related --- choosing one implies choices for the others.
Furthermore we show there are more intuitively usable parameterisations for
$f$-divergences and scoring rules (their corresponding weight functions). The
weight functions have the advantage that if two weight functions match, then the
corresponding objects are identical. That is not the case for the $f$
parametrising an $f$-divergence or the convex function parametrising a Bregman
divergence. As
well as the theoretical interest in such connections, these alternate
representations suggest new algorithms for empirically estimating such
quantities.

%%%%%%%%%%%%%%%%%
\subsection{Background}

Specific results are referred to in the body of the paper. We briefly indicate
the broad sweep of prior work along the lines of the present paper. 

The most
important precursors and inspiration are the three nearly 
simultaneous\footnote{
    \citep{Nguyen:2005}
    is dated 13 October, 2005, \citep{LieVaj06} was received on 26 October
    2005 and \citep{Buja:2005} is dated 3 November 2005.  Shen's PhD thesis 
    \citep{Shen2005},
    which contains most of the material in \citep{Buja:2005}, is dated 16
    October 2005.
}
works by \cite{Buja:2005}, \cite{LieVaj06} and \cite{Nguyen:2005}.  The work by 
\citet{Dawid2007} is very similar in spirit to that presented here. A crucial
difference is that he relies on a parametric viewpoint, and can utilise the
machinery of Riemannian geometry\footnote{
\cite{Zhang:2004,ZhangMatsuzoe2008} have developed a number of connections
between convex functions, the Bregman divergences they induce, and Riemannian
geometry.}. All of the results in the present paper are,
in contrast, ``coordinate-free.''  The \emph{motivation} of the present work is
closely aligned with that of  \citet{Hand1994} whose avowed aim was to
``stimulate debate about the need to formulate research questions sufficiently
precisely that they may be unambiguously and correctly matched with statistical
techniques.''\footnote{\citet{HandVinciotti2003} develop some refined machine
learning tasks that can be viewed  as weighted problems; confer 
\cite{Buja:2005}.}

The paper presents a unification of sorts. This, in itself, is hardly new in
machine learning.
There are different {\em approaches to unification}. One distinction is between
{\em Monistic} and {\em Pluralistic} approaches \citep{James1909,
TurklePapert1992}.

\emph{Monistic} approaches aim for a single all encompassing
theory\footnote{Monistic approaches can be categorised into at least four
distinct categories. They are briefly summarised in
Appendix~\ref{section:monistic}.}.  A problem with most monistic approaches is
that  you have to
accept it ``all or nothing.'' There are many unifying approaches developed
in Statistics and Machine learning that have left  little 
trace\footnote{For example: Nelson's
use of non-standard analysis \citep{Nelson1987, LutzMusio2005} as the
foundations for probability;
Tops{\o}e's \citep{Topsoe:2006},  Shafer and Vovk's \citep{ShaferVovk2001} game
theory as a basis, and Le Cam's use  of Riesz measures on a vector 
lattice to replace the traditional sample space~\citep{LeCam1964}.}.

\emph{Pluralistic} approaches are closer to what is proposed here (where,
instead of searching for a single master representation, we study relationships
and translations between a range of different representations). It 
resonates with Kiefer's assertion that
	``Statistics is too complex to be codified in terms of a simple
	prescription that is a panacea for all settings, and \dots one must look
	as carefully as possible at a variety of possible procedures\dots''
	\citep{Kiefer1977}.
Examples of existing pluralistic attempts include limited problem catalogs such
as for different notions of {\em cost} \citep{Turney2000} or a restricted set
of problems \citep{Raudys:2001}.

The decision theoretic approach \citep{DeGroot1970,Berger1985, Kiefer1987} due
to \cite{Wald1950,Wald1949} is central to the present paper.  The idea of
seeking {\em primitives} for statistics dates back at least to the elementary
experiments of \cite{Birnbaum:1961}.  The relationship between risks and
Bregman divergences is studied by \cite{GrunwaldDawid2004,Buja:2005}.
Summaries of earlier work on surrogate regret bounds and Pinsker bounds are
given in Appendices \ref{section:surrogate} and \ref{section:history}
respectively.

There are numerous possible definitions of information. Many of them are sterile;
\cite{Csiszar1978}  and \cite{Aczel1984} provide a critical analysis. 
\citet{Floridi2004} 
discusses pluralistic versus monistic approach: is there one single
definition of information, or should there be many different definitions
depending on the particular problem? Our view, like \cite{Shannon1948} is that
there are many types. Shannon information was developed with communications
problems in mind --- there is no fundamental reason why it is the only notion
of information that makes sense for learning and inference.

There are many known relationships between risks and
    divergences between distributions many of which we explicitly discuss later
    in the paper\footnote{General results include those due to 
    \cite{Osterreicher2003,
    OsterreicherVajda1993a, Gutenbrunner:1990,LieVaj06,Goel:1979,
    Golic:1987}.  Particular relations between risk in binary
    classification problems and $f$-divergences are not
    new \citep{PoorThomas1977,Kailath:1967}.  Some more general results that
    relate  the choice of loss function in a binary learning problem to
    particular $f$-divergences between the class-conditional distributions have
    been (re)-discovered \citep{EguchiCopas2001, Nguyen:2005,
    OsterreicherVajda1993a}. Known results relating different distances between
    probability distributions are summarised by
    \cite{GibbsSu2002}
    }.
The idea of solving a machine learning problem by using a solution to some
other learning problem is now called {\em machine
learning reductions} 
\citep{BeygelzimerLangfordZadrozny2008,BeygelzimerDaniHayesLangfordZadrozny}\footnote{
The idea is not new.
Equivalences are a natural structuring device and were explicit in Ashby's
foundational work on cybernetics \citep{Ashby1956}, a precursor to Machine
Learning.
\cite{Ben-Bassat:1978} studied the concept of 
    $\epsilon$-equivalence,
    \cite{conover1981rtb} showed how rank
    tests can be derived by applying nonparametric tests to order statistics,
    and \cite{GoldmanRivestSchapire1989,BartlettLongWilliamson1996}  used 
    reductions for  
    theoretical purposes.
    However recently there has been a large number
    of explicit constructions of reductions
    \citep{ZadroznyLangfordAbe2003,
    Langford:2004,
    BeygelzimerDaniHayesLangfordZadrozny,
    Langford:2005,
    LangfordZadrozny2005,
    Langford:2006, 
    LiLin,
    BeygelzimerLangfordRavikumar2007, Langford2007,
    Scott:2007},or development of
    results which although not explicitly called reductions are effectively so
    \citep{Brown:2002,Brown:1996,brown2003dae, Chaudhuri:2002,
    CossockZhang2006,CuevasFraiman1997,Domingos:1999,steinwart2005dld,
    Tasche:2001}.}. 
    Two key differences between the recent machine learning reductions 
    literature
    and the present paper is that our relationships between problems are
    (usually) exact (instead of approximate) and we work with the true
    underlying distributions (rather than finite sample distributions).
The theory of Comparison of Experiments,
 developed by  \cite{Blackwell1951, blackwell1953ece}, and significantly
 extended by \cite{LeCam1964,LeCam1986} is also related to the overall goal set
 out here\footnote{It has been used to define notions of isomorphism for
     statistical problem settings \citep{morse1966si,Sacksteder1967} and is the
     subject of three books \citep{Strasser1985, torgersen1991cse,Heyer1982}
     and a recent review \citep{GoelGinebra2003}.  The key difference with the
     present work is that
     the comparison of experiments theory seeks results that hold 
     for {\em all} loss functions rather
     than for a particular one; with a few exceptions \citep[Chapter
     10]{torgersen1991cse}.  Blackwell related comparisons to {\em
     sufficient statistics} and characterised comparisons.
     \cite{LeCam1964} quantified comparisons in terms of the degree to
     which one experiment is ``better than'' another (the {\em deficiency
     distance}).  There are very few known examples of deficiency distance
     \citep{Carter2002}.  Furthermore LeCam's theory is formulated in a
     particularly abstract way to make its theorems elegant
     \citep{YangCam1999}.  Renowned probabilists concur that its arcane
     formulation has made it inaccessible \citep{vandervaart2002swl,
     Pollard2000, MazanecStrasser2000}.  Consequently the subject has had
     relatively limited impact.  
     }.

     Graphical representations have been used for a long while to better
     understand binary experiments\footnote{In this paper we draw 
     connections between
     Receiver Operating Characteristic (ROC) curves,
     \citep{Fawcett:2006,Fawcett:2004,Flach:2003,Flach:2005,Maxion:2004}
     the Area Under ROC Curve (AUC),
     \citep{cortes2004aov,Hand2008,Hand:2001,HanNei82} and
     Cost Curves \citep{DrummondHolte2006,torgersen1991cse}.}. These can be seen
     as {\em representations} of Binary Experiments.

%%%%%%%%%%%%%%%%%%
\subsection{Outline}
The following is an outline of the main structure of this paper.

Many of the properties of the quantities studied in this paper are directly
derived from well-known properties of \emph{convex functions}. In particular, 
a generalised form of Taylor's theorem and Jensen's inequality underpin many
of the results in this paper.

One of the simplest type of statistical problems is distinguishing between two 
distributions. Such a problem is known as a \emph{binary experiment}. Two 
classes of \emph{measures of divergence} between the distributions are 
introduced: the class of Csisz\'{a}r  $f$-divergences and the class of Bregman 
divergences. 

When additional assumptions are made about a binary experiment ---
specifically, a prior probability for each of the two distributions --- it
becomes possible to talk about \emph{risk and statistical information} of an
experiment that is defined with respect to a loss function.

First, we present a number of results connecting
risk, statistical information, $f$-divergence, and Bregman divergence and
information that are scattered about the literature 
These results show that all of these concepts are intimately related. Second,
we exploit a result that shows that proper scoring rules --- a natural class
of losses for probability estimation --- have a Choquet representation;
\emph{i.e.}~they are expressible as weighted integrals of a family of 
``primitive'' losses, namely the \emph{cost-weighted misclassification losses}.

By combining this characterisation of proper scoring rules with the results
relating risk, information and divergence we are able to identify similar
primitives and weighted integral representations for $f$-divergences,
statistical information, Bregman divergences, and Bregman information. 
These representations simplify the study of these concepts by identifying each
with its corresponding weight function.
These weight functions elucidate several properties of the risks, divergences
and informations they characterise, including their optimality and their
convexity or concavity.
We provide a  ``translation'' between weight functions that
clarifies the relationships between these concepts.  The weight function view
also illuminates various ``graphical representations'' of binary experiments,
such as ROC curves.

Finally, we present two insights obtained from this unification. 
The first is a template for deriving Pinsker-like bounds
on arbitrary $f$-divergences in terms of variational divergence and surrogate
loss bounds which bound the regret of an hypothesis under an arbitrary scoring
rules in terms of its regret under the cost-sensitive misclassification loss.
The bounds we derive are more general than those previously presented.  The
second insight concerns the apparent difference between the Bayes risk (which
involves an optimization) and the $f$-divergence (which does not). Both of
these are equivalent in ways we show. Thus we consider ``variational''
approaches to divergences. One specific  consequence of this is that maximum
mean discrepancy (MMD) --- a kernel approach to hypothesis testing and
divergence estimation --- is essentially SVM learning in disguise.

%%%%%%%%%%%%%%%%%%%%%%%%%%%%%%%%%
\subsection{Notational Conventions}

The substantive objects are defined within the body of the paper. Here we 
collect
elementary notation and the conventions we adopt throughout.
We write $x\wedge y:=\min(x,y)$, $x\vee y :=\max(x,y)$, $(x)_+:= x\vee 0$,
$(x)_-:=x\wedge 0$ and $\test{p}=1$ if $p$ is true and $\test{p}=0$ otherwise.
The generalised function $\delta(\cdot)$ is defined by $\int_a^b \delta(x)
f(x)dx=f(0)$ when $f$ is continuous at $0$ and $a<0<b$. The unit step
$U(x)=\int_{-\infty}^x \delta(t)dt$.
The real numbers are denoted $\reals$, the non-negative reals $\reals^+$ and
the extended reals $\bar{\reals}=\reals\cup\{\infty\}$; the rules of arithmetic
with extended real numbers and the need for them in convex analysis are
explained by \cite{Rockafellar:1970}.
Random variables are written in sans-serif
font: $\mathsf{S}$, $\mathsf{X}$, $\mathsf{Y}$.  Sets are in calligraphic font:
$\mathcal{X}$ (the ``input'' space), $\mathcal{Y}$ (the ``label'' space).  
Vectors are written in bold font: $\abold,\bm{\alpha},\bm{x}\in\reals^m$. We
will often have cause to take expectations ($\EE$) over the random variable
$\mathsf{X}$. We write such quantities in blackboard bold:
$\II$, $\LL$, $\BB$, $\JJ$ etc. The elementary loss is $\ell$, its conditional
expectation w.r.t.  $\mathsf{Y}$ is $L$ and the full expectation (over the
joint distribution $\mathbb{P}$ of
$(\mathsf{X},\mathsf{Y})$) is $\LL$. The
lower bound on quantities with an intrinsic lower bound (e.g. the Bayes optimal
loss) are written with an underbar: $\minL$, $\minLL$. Quantities related by
double integration recur in this paper and we notate the starting point in
lower case, the first integral with upper case, and the second integral in
upper case with an overbar: $w$, $W$, $\Wb$. Estimated quantities are hatted:
$\heta$.  In several places we overload the notation. In all cases careful
attention to the type of the arguments or subscripts reliably disambiguates.

%%%%%%%%%%%%%%%%%%%%%%%%%%%%%%%%%%%%%%%%%%%%%%%%%%%%%%%%%%%%%%%%%%%%%%%%
\section{Convex functions and their representations}\label{sec:convex}
Many of the properties of divergences and losses are best understood through 
properties of the convex functions that define them. One aim of this paper is
to explain and relate various divergences and losses by understanding the
relationships between their primitive functions. The relevant definitions and
theory of convex functions will be introduced as required. Any terms not
explicitly defined can be found in books by \cite{hiriarturruty2001fca} or
\cite{Rockafellar:1970}.

A set $\Scal \subseteq \RR^d$ is said to be \emph{convex} if it is closed 
under linear interpolation. 
That is, for all $\lambda \in [0,1]$ and for all points $s_1,s_2\in\Scal$ the
point $\lambda s_1 + (1-\lambda)s_2 \in \Scal$.
A function $\phi : \Scal \to \RR$ defined on a convex set $\Scal$ is said to be
a (proper) \emph{convex function} if all lines between points on the 
graph of $\phi$ never lie below $\phi$.\footnote{
	The restriction of the values of $\phi$ to $\RR$ will be assumed 
	throughout unless explicitly stated otherwise. This implies the properness
	of $\phi$ since it cannot take on the values $-\infty$ or $+\infty$.
}
That is, for all $\lambda\in[0,1]$ and points $s_1, s_2 \in \Scal$ the function 
$\phi$ satisfies
\begin{equation*}
	\phi(\lambda s_1 + (1-\lambda) s_2) 
	\geq \lambda \phi(s_1) + (1-\lambda) \phi(s_2).
\end{equation*}
A function is said to be \emph{concave} if its additive inverse is convex. That 
is, $\phi : \Scal \to \RR$ is concave if $-\phi$ is concave.

The remainder of this section presents  properties, representations 
and transformations of convex functions that will be used throughout this paper.

\subsection{The Perspective Transform and the Csisz\'{a}r Dual}
\label{sub:perspective}
When $\Scal = \RR^+$ we can define a transformation of a convex 
function $\phi : \RR^+ \to \RRbar $ called 
the \emph{perspective transform} of $\phi$, denoted $I_\phi$ and defined 
for $\tau \in \RR^+$ by
\begin{equation}
	I_\phi(s, \tau) := 
	\begin{cases}
		\tau \phi( s/\tau ), & \tau > 0 , s > 0\\
		0, & \tau = 0, s = 0 \\
		\tau \phi(0), & \tau > 0 , s = 0\\
		s \phi'_{\infty}, & \tau = 0, s > 0
	\end{cases}
	\label{eq:perspective}
\end{equation}
where $\phi(0) := \lim_{s\to 0} \phi(s) \in \RRbar$ and 
$\phi'_{\infty}$ is the \emph{slope at infinity} defined as
\begin{equation}\label{eq:slopeatinfty}
	\phi'_{\infty} 
	:= \lim_{s\to +\infty} \frac{\phi(s_0 + s) - \phi(s_0)}{s}
	= \lim_{s\to +\infty} \frac{\phi(s)}{s}
\end{equation}
for every $s_0\in\Scal$ where $\phi(s_0)$ is finite. This slope at infinity is 
only finite when $\phi(s) = O(s)$, that is, when $\phi$ grows at most linearly 
as $s$ increases. When $\phi'_{\infty}$ is finite it measures the slope of the
linear asymptote.
The function $I_\phi : [0,\infty)^2 \to \RRbar$ is convex 
in both arguments \citep{Hiriart-UrrutyLemarechal1993} and may take on the value 
$+\infty$ when $s$ or $\tau$ is zero. 
It is introduced here as it will form the basis of the $f$-divergences described 
in the next section.\footnote{
	The perspective transform is closely related to \emph{epi-multiplication} 
	which is defined for all
	$\tau\in[0,\infty)$ and (proper) convex functions $\phi$ to be
	$\tau \etimes \phi := s\mapsto \tau \phi(s/\tau)$ for $\tau > 0$ and
	is $0$ when $\tau=s=0$ and $+\infty$ otherwise.
	\cite{Bauschke:2008} provides an excellent summary of the properties 
	of this operation along with its relationship to other operations on convex
	functions.
}

The perspective transform can be used to define the \emph{Csisz\'{a}r dual} 
$\cdual{\phi} : [0,\infty) \to \RRbar$ of a convex function $\phi : \RR^+\to\RR$ 
by letting
\begin{equation}\label{eq:csiszardual}
	\cdual{\phi}(\tau) 
	:= I_{\phi}(1,\tau) 
	= \tau\phi\left(\frac{1}{\tau}\right)
\end{equation}
for all $\tau\in\RR^+$ and $\cdual{\phi}(0) := \phi'_{\infty}$. 
Note that the original $\phi$ can be recovered from $I_\phi$ as 
$\phi(s) = I_f(s,1)$.

The convexity of the perspective transform $I_\phi$ in 
both its arguments guarantees the convexity of the dual $\cdual{\phi}$. Some
simple algebraic manipulation shows that for all $s,\tau \in \RR^+$
\begin{equation}\label{eq:cdualpersp}
	I_{\phi}(s,\tau) = I_{\cdual{\phi}}(\tau,s).
\end{equation}
This observation leads to a natural definition of symmetry for convex functions.
We will call a convex function \emph{$\csiszar$-symmetric} (or simply 
\emph{symmetric} when the context is clear) when its perspective transform
is symmetric in its arguments. That is, $\phi$ is $\csiszar$-symmetric when 
$I_\phi(s,\tau) = I_\phi(\tau, s)$ for all $s,\tau\in[0,\infty)$. 
Equivalently, $\phi$ is symmetric if and only if $\cdual{\phi} = \phi$.

\subsection{The Legendre-Fenchel Dual Representation}
\label{sub:lfdual}
A second important dual operator for convex functions is the  
\emph{Legendre-Fenchel (LF) dual}.
The LF dual $\lfdual{\phi}$ of a function $\phi : \Scal \to \RR$ is a function
defined by
\begin{equation}\label{eq:lfdual}
	\lfdual{\phi}(\lfdual{s})
	:= \sup_{s\in\Scal} \{ \inner{s}{\lfdual{s}} - \phi(s) \}.
\end{equation}
The LF dual of any function is convex and, if the function $\phi$ is convex then
the \emph{LF bidual} is a faithful representation of the original function. That
is,
\begin{equation}\label{eq:bidual}
	\phi^{\lf\lf}(s) 
	= \sup_{\lfdual{s}\in\lfdual{\Scal}} \{ 
		\inner{\lfdual{s}}{s} - \lfdual{\phi}(\lfdual{s}) 
	\}
	= \phi(s).
\end{equation}

When $\phi(s)$ is a function of a real argument $s$ and the derivative 
$\phi'(s)$ exists, the Legendre-Fenchel conjugate $\lfdual{\phi}$ is given by the 
\emph{Legendre transform} \citep{hiriarturruty2001fca,Rockafellar:1970}
\begin{equation}
    \lfdual{\phi}(s)
    = s\cdot (\phi')^{-1}(s) - \phi\left( (\phi')^{-1}(s)\right).
    \label{eq:legendre-transform}
\end{equation}

\subsection{Integral Representations}\label{sub:intrep}
In this paper we are primarily concerned with convex and concave functions 
defined on subsets of the real line. 
A central tool in their analysis is the integral form of their Taylor expansion.
Here, $\phi'$ and $\phi''$ denote the first and second derivatives of $\phi$
respectively.

\begin{theorem}[Taylor's Theorem]
	Let $\Scal = [s_0,s]$ be a closed interval of $\RR$ and let 
	$\phi : \Scal \to \RR$ be differentiable on $[s_0,s]$ and twice 
	differentiable on $(s_0,s)$. Then
	\begin{equation}\label{eq:taylor}
		\phi(s) 
		= \phi(s_0) + \phi'(s_0)(s - s_0) + \int_{s_0}^s (s-t)\,\phi''(t)\,dt .
	\end{equation}
\end{theorem}

The argument $s$ appears in the limits of integral in the above theorem and
consequently can be 
awkward to work with. Also, it will be useful to expand $\phi$ about some point
not at the end of the interval of integration. The following corollary of
Taylor's theorem removes these problems by introducing piece-wise linear terms 
of the form $(s-t)_+ = (s-t)\vee 0$.

\begin{corollary}[Integral Representation I]\label{cor:intrep1}
	Let $\phi:[a,b] \to \RR$ be a twice differentiable function. Then, for all 	
	$s,s_0 \in [a,b]$ we have
	\begin{equation}\label{eq:intrep1}
		\phi(s)
		= \phi(s_0) + \phi'(s_0)(s - s_0) 
		+ \int_a^b \phi_{s_0}(s,t)\,\phi''(t)\,dt ,
	\end{equation}
	where 
	\begin{equation}
		\phi_{s_0}(s,t) := 
			\begin{cases}
				(s-t)_+ & s \leq s_0 \\
				(t-s)_+ & s > s_0
			\end{cases}
			\label{eq:phi-s-0-def}
	\end{equation}
	is a piece-wise linear and convex in $s$ for each $s_0, t \in [a,b]$.
\end{corollary}
This result is a consequence of the way in which the terms $\phi_t$ effectively 
restrict the limits of integration to the interval $(s_0, s) \subseteq [a,b]$
or $(s, s_0) \subseteq [a,b]$ depending on whether $s_0 < s$ or $s_0 \geq s$ 
with appropriate reversal of the sign of $(s-t)$.

\cite{LieVaj06} proved a general version of the above theorem that holds for
functions with discontinuous first derivatives. Since convex functions are 
necessarily continuous, they replace the first derivative $\phi'$ with a 
right-hand derivative $\phi'_+$ (which is guaranteed to exist) and the
second derivative $\phi''$ with the measure $d\phi'_+$. To make the exposition
simpler we will generally assume that the functions we study are suitably 
differentiable (but see the comment below on distributional derivatives).

When $a = 0$ and $b = 1$ a second integral representation for the unit interval can 
be derived from (\ref{eq:intrep1}) that removes the term involving $\phi'$. 
\begin{corollary}[Integral Representation II]\label{cor:intrep2}
	A twice differentiable function $\phi:[0,1]\to\RR$ can be expressed as
	\begin{equation}\label{eq:intrep2}
		\phi(s)
		= \phi(0) + (\phi(1) - \phi(0))s
		- \int_0^1 \psi(s,t)\,\phi''(t)\,dt ,
	\end{equation}
	where $\psi(s,t) = (1-t)s \wedge (1-s)t$ is piece-wise linear and
	concave in $s\in[0,1]$ for each $t\in[0,1]$.
\end{corollary}
The result follows by integration by parts of 
$t\phi''(t)$. The proof can be found in Appendix~\ref{sub:proof-of-intrep2}.
It is used in Section~\ref{sec:intrep} below to obtain an integral 
representation of losses for binary class probability estimation.
This representation can be traced back to \cite{Temple:1954} who
notes that the kernel $\psi(s,t)$ is the Green's function for the differential
equation $\psi'' = 0$ with boundary conditions $\psi(a) = \psi(b) =0$.

Both these integral representations state that the non-linear part of $\phi$ can 
be expressed as a weighted integral of piece-wise linear terms $\phi_{s_0}$ or
$\psi$. 
When we restrict our attention to convex $\phi$ we are guaranteed the
``weights'' $\phi''(t)$ for each of these terms are non-negative. 
Since the measures of risk, information and divergence we examine below do not 
depend on the linear part of these expansions we are able to identify 
convex functions with the weights $w(t) = \phi''(t)$ that define their 
non-linear part.
The sets of piece-wise linear functions $\{\phi_{s_0}(s,t)\}_{t\in[a,b]}$ 
and $\{\psi(s,t)\}_{t\in[0,1]}$ can be though of as families of ``primitive'' 
convex functions from which others can be built through their weighted 
combination.
Representations like these are often called \emph{Choquet representations}
after work by \cite{Choquet:1953} on the representation of compact convex 
spaces \citep{Phelps2001}.

Equation \ref{eq:intrep2} is also valid when $\phi''$ only exists in a
distributional sense \citep{AntosikMikusinskiSikorski1973,Friedlander1982}. 
In fact all of the integral representation results in
this paper are so valid. being able to deal with distributions is essential in
order to understand the weight functions corresponding to the primitive
$f$-divergences and loss functions.

%%%%%%%%%%%%%%%%%%%%%%%%
\subsection{Bregman Divergence}\label{sub:Bregman}
Bregman divergences are a generalisation of the notion of distances between 
points. 
Given a differentiable\footnote{
	Technically, $\phi$ need only be differentiable on the relative interior 
	$\ri(\Scal)$ of $\Scal$. We omit this requirement for simplicity and 
	because it is not relevant to this discussion.
} convex function $\phi : \Scal \to \RR$ 
and two points $s_0, s\in\Scal$
the \emph{Bregman divergence\footnote{
Named in reference to \cite{Bregman1967} although he was not the first 
to consider
such an equation, at least in the one dimensional case
\citep[p.838]{BrunkEwingUtz1957}.}
of $s$ from $s_0$} is defined to be
\begin{equation}\label{eq:bregman}
	B_{\phi}(s,s_0) 
	:= \phi(s) - \phi(s_0) - \inner{s - s_0}{\nabla\phi(s_0)} ,
\end{equation}
where $\nabla\phi(s_0)$ is the gradient of $\phi$ at $s_0$.
A concise summary of many of the properties of Bregman divergences is given by 
\citet[Appendix A]{Banerjee:2005a}.
In particular, Bregman divergences always satisfy $B_\phi(s,s_0)\ge 0$ and 
$B_\phi(s_0,s_0) = 0$ for all $s,s_0\in\Scal$, regardless of 
the choice of $\phi$.
They are not always metrics, however, as they do not always satisfy the triangle 
inequality and their symmetry depends on the choice of $\phi$.

When $\Scal = \RR$ and $\phi$ is twice differentiable, comparing the definition 
of a Bregman divergence in (\ref{eq:bregman}) to the integral representation in 
(\ref{eq:taylor}) reveals that Bregman divergences between real numbers can be 
defined as the non-linear part of the Taylor expansion of $\phi$.
Rearranging (\ref{eq:taylor}) shows that for all $s,s_0\in\RR$ 
\begin{equation}
	\int_{s_0}^s (s-t)\,\phi''(t) dt
	= \phi(s) - \phi(s_0) - (s-s_0)\phi'(s_0)
	= B_\phi(s,s_0)
\end{equation}
since $\nabla\phi = \phi'$ and the inner product is simply multiplication over 
the reals.
This result also holds for more general convex sets $\Scal$. Importantly, it 
intuitively shows why the following holds.
\begin{theorem}
	Let $\phi$ and $\psi$ both be real-valued, differentiable convex functions 
	over the convex set $\Scal$ such that $\phi(s) = \psi(s) + as + b$ for some 
	$a,b\in\RR$. Then, for all $s$ and $s_0$, $B_\phi(s,s_0) = 
	B_\psi(s,s_0)$.
\end{theorem}
A proof can be obtained directly by substituting and expanding $\psi$ in the
definition of a Bregman divergence. 

%%%%%%%%%%%%%%%%%%%%%%%%%%%%%%%%%%%%%%%%%%%%
\subsection{Jensen's Inequality and the Jensen Gap}
A central inequality in the study of convex functions is Jensen's inequality.
It relates the expectation of a convex function applied to random variable
to the convex function evaluated at its mean.
We will denote by $\E{\mu}{\cdot} := \int_{\Scal} \cdot\,d\mu$ expectation over 
$\Scal$ with respect to a probability measure $\mu$ over $\Scal$.

\begin{theorem}[Jensen's Inequality]\label{thm:jensen}
	Let $\phi : \Scal \to \RR$ be a convex function,  
	$\mu$ be a distribution and $\Ssf$ be an $\Scal$-valued 
	random variable (measurable w.r.t. $\mu$) such that 
	$\E{\mu}{|\mathsf{S}|} < \infty$.
	The following inequality holds
	\begin{equation}\label{eq:jensen}
		\JJ_{\mu}[\phi(\Ssf)] 
		:= \E{\mu}{\phi(\Ssf)} - \phi(\E{\mu}{\Ssf})
		\geq 0 .
	\end{equation}
\end{theorem}
The proof is straight-forward and can be found in \citep[\S10.2]{Dudley:2003}.
Jensen's inequality can also be used to characterise the class of convex 
functions.
If $\phi$ is a function such that (\ref{eq:jensen}) holds for all random 
variables and distributions then $\phi$ must be convex.\footnote{
	This can be seen by considering a distribution with a
	finite, discrete set of points as its support and applying Theorem~4.3
	of \cite{Rockafellar:1970}. 
}
Intuitively, this connection between expectation and convexity is natural since
expectation can be seen as an operator that takes convex combinations of random
variables.

We will call the difference $\JJ_\mu[\phi(\Ssf)]$ the 
\emph{Jensen gap for $\phi(\Ssf)$}.
Many measures of divergence and information studied in the subsequent
sections can be expressed as the Jensen gap of some convex function.
Due to the linearity of expectation, the Jensen gap is insensitive to the 
addition of affine terms to the convex function that defines it: 
\begin{theorem}\label{thm:jensen-affine}
	Let $\phi : \Scal \to \RR$ be convex function and $\Ssf$ and $\mu$
	be as in Theorem~\ref{thm:jensen}. Then for each $a, b \in \RR$
	the convex function $\psi(s) := \phi(s) + as + b$ satisfies
	$\JJ_{\mu}[\phi(\Ssf)] = \JJ_{\mu}[\psi(\Ssf)]$,
	where $\phi_{s_0}$ is as in (\ref{eq:phi-s-0-def}).
\end{theorem}
The proof is a consequence of the definition of the Jensen gap and the linearity
of expectations and can be found in Appendix~\ref{sub:proof-jensen-affine}.
An implication of this theorem is that when considering sets of convex functions
as parameters to the Jensen gap operator they only need be identified by their
non-linear part. Thus, the Jensen gap operator can be seen to impose an 
equivalence relation over convex functions where two convex 
functions are equivalent if they have the same Jensen gap, that is, if their
difference is affine. 

In light of the two integral representations in 
Section~\ref{sub:intrep}, this means the Jensen gap only depends on the integral
terms in (\ref{eq:intrep1}) and (\ref{eq:intrep2}) and so is completely 
characterised by the weights provided by $\phi''$. Specifically, for 
suitably differentiable $\phi : [a,b] \to \RR$ we have 
\[
	\JJ_\mu[\phi(\Ssf)] = \int_a^b \JJ_\mu[\phi_{s_0}(\Ssf,t)]\,\phi''(t)\,dt.
\]
Since several of the measures of divergence, information and risk we analyse can 
be expressed as a Jensen gap, this observation implies that these quantities can 
be identified with the weights provided by $\phi''$ as it is these that 
completely determine the measure's behaviour.

%%%%%%%%%%%%%%%%%%%%%%%%%%%%%%%%%%%%%%%%%%%%%%%%%%%%%%%%%%%%%%%%%%%%%%%%%%%
\section{Binary Experiments and Measures of Divergence}

The various properties of convex functions developed in the previous section
have many implications for the study of statistical inference.
We begin by considering \emph{binary experiments} $(P,Q)$ where $P$ and $Q$ are
probability measures\footnote{
	We intentionally avoid too many measure theoretic details for the sake of 
	clarity. Appropriate $\sigma$-algebras and continuity can be assumed where
	necessary. 
} over a common space $\Xcal$. We will often consider $P$ the distribution over 
\emph{positive} instances and $Q$ the distribution over \emph{negative} 
instances. The densities of $P$ and $Q$ with respect to some third reference 
distribution $M$ over $\Xcal$ will be defined by $dP = p\,dM$ and $dQ = q\,dM$ 
respectively. Unless stated otherwise we will assume that $P$ and $Q$ are both
absolutely continuous with respect to $M$. (One can always choose $M$
to ensure this by setting $M=(P+Q)/2$; but see the next section.)

There are several ways in which the ``separation'' of $P$ and $Q$ in a binary 
experiment can be quantified. Intuitively, these all measure the difficulty of 
distinguishing between the two distributions on the basis of instances drawn
from their mixture. The further apart the distributions are the easy 
discrimination becomes. This intuition is made precise through the connections
with risk and MMD later in Appendix \ref{sec:appendix-svm-mmd}.

A central statistic in the study of binary experiments and statistical 
hypothesis testing is the likelihood ratio $dP/dQ$. As the following section
outlines, the likelihood ratio is, in the sense of preserving the distinction 
between $P$ and $Q$, the ``best'' mapping from an arbitrary space $\Xcal$ to the 
real line.

%%%%%%%%%%%%%%%%%%%%%%%%

\subsection{Statistical Tests and the Neyman-Pearson Lemma}
\label{sub:NPLemma}

In the context of a binary experiment $(P,Q)$, a \emph{statistical test} is any 
function that assigns each instance $x\in\Xcal$ to either $P$ or $Q$. We will
use the labels 1 and 0 for $P$ and $Q$ respectively and so a statistical test
is any function $r : \Xcal \to \{0,1\}$. In machine learning, a function of this 
type is usually referred to as a \emph{classifier}. The link between tests and 
classifiers is explored further in Section~\ref{sec:risk}.

Each test $r$ partitions the instance space $\Xcal$ into  positive and 
negative \emph{prediction sets}:
\begin{eqnarray*}
	\Xcal^+_r & := & \{ x \in \Xcal\ :\ r(x) = 1 \}  \\
	\Xcal^-_r & := & \{ x \in \Xcal\ :\ r(x) = 0 \}.
\end{eqnarray*}
There are four \emph{classification rates} associated with these predictions 
sets: the true positive rate ($TP$), true negative rate ($TN$), false positive 
rate ($FP$) and the false negative rate ($FN$). For a given test $r$ they are 
defined as follows:
\begin{equation}\label{eq:contingency}
	\begin{array}{cc}
	TP_r := P(\Xcal^+_r)\ &\ FP_r := Q(\Xcal^+_r)\\
	FN_r := P(\Xcal^-_r)\ &\ TN_r := Q(\Xcal^-_r).
	\end{array}
\end{equation}
The subscript $r$ will often be dropped when the test made clear by the 
context.
Since $P$ and $Q$ are distributions over $\Xcal = \Xcal^+_r \cup \Xcal^-_r$ and
the positive and negative sets are disjoint we have that $TP + FN = 1$ and $FP +
TN = 1$.
As a consequence, the four values in (\ref{eq:contingency}) can be summarised by
choosing one from each column.

Often, statistical tests are obtained by applying a threshold $\tau_0$ to a 
real-valued \emph{test statistic} $\tau : \Xcal \to \RR$. 
In this case, the statistical test is $r(x) = \test{\tau(x) \ge \tau_0}$.
This leads to parametrised forms of prediction sets $\Xcal^y_{\tau}(\tau_0) :=
\Xcal^y_{\test{\tau \ge \tau_0}}$ for $y\in\{+,-\}$, and the classification
rates $TP_{\tau}(\tau_0)$, $FP_{\tau}(\tau_0)$, $TN_{\tau}(\tau_0)$, and
$TP_{\tau}(\tau_0)$ which are defined analogously. By varying the threshold 
parameter a range of classification rates can be achieved. 
This observation leads to a well known graphical representation of 
test statistics known as the ROC curve, which is discussed further in 
Section~\ref{sub:ROC}.

A natural question is whether there is a ``best'' statistical test or test 
statistic to use for binary experiments. This is usually formulated in terms of 
a test's power and size.
The \emph{power} $\beta_r$ of the test $r$ for a particular binary experiment 
$(P,Q)$ is a synonym for its true positive rate (that is, $\beta_r := TP_r$ 
and so $1-\beta_r := FN_r$\footnote{
	This is opposite to the usual definition of $\beta_r$ in the statistical
	literature. Usually, $1-\beta_r$ is used to denote the power of a test.
	We have chosen to use $\beta_r$ for the power (true positive rate) 
	as this makes it easier to compare with ROC curves.
}
) and the \emph{size} $\alpha_r$ of same test is just 
its false positive rate $\alpha_r := FP_r$.
Here, ``best'' is considered to be the \emph{uniformly most powerful} (UMP)
test of a given size. That is, a test $r$ is considered UMP of size 
$\alpha\in[0,1]$ if, $\alpha_r = \alpha$ and for all other tests $r'$ such that 
$\alpha_{r'} \le \alpha$ we have $1-\beta_r \ge 1-\beta_{r'}$.
We will denote by $\beta(\alpha):=\beta(\alpha, P,Q)$ the true positive 
rate of a UMP test between $P$ (the null hypothesis) and $Q$ 
(the alternative) at $Q$ with significance $\alpha$. \cite{torgersen1991cse} 
calls
$\beta(\cdot,P,Q)$ the {\em Neyman-Pearson function for the dichotomy} $(P,Q)$.
Formally, for each $\alpha \in [0,1]$, the Neyman-Pearson function $\beta$ 
measures the largest true positive rate $TP_r$ of any measurable classifier 
$r\colon \Xcal \to \{-1,1\}$ that has false positive rate $FP_r$ at most 
$\alpha$. 
That is, \[
	\beta(\alpha) 
	= \beta(\alpha,P,Q) 
	:= \sup_{r} \{TP_r\ \colon\ FP_r \le \alpha\}.
\]

The Neyman-Pearson lemma \citep{Neyman:1933} shows that the likelihood ratio
$\tau^*(x) = dP/dQ(x)$ is the uniformly most powerful test for each choice of 
threshold $\tau_0$. 
Since each choice of $\tau_0\in\RR$ results in a test 
$\test{dP/dQ \ge \tau_0}$ of some size $\alpha\in[0,1]$ we have that\footnote{
	Equation (\ref{eq:beta-in-terms-of-minLL}) in Section~\ref{sub:ROCandRisk} 
	below, shows that $\beta(\alpha)$ is the lower envelope of a family of 
	linear functions of $\alpha$ and is thus concave and continuous. Hence,
	the equality in (\ref{eq:beta-with-NP}) holds.
} 
\begin{equation}\label{eq:beta-with-NP}
	\beta(FP_{\tau^*}(\tau_0)) = TP_{\tau^*}(\tau_0)
\end{equation}
and so varying $\tau_0$ over $\RR$ results in a maximal ROC curve. This too is
discussed further in Section~\ref{sub:ROC}. 

The Neyman-Pearson lemma thus identifies the likelihood ratio $dP/dQ$ as a 
particularly useful statistic. Given an experiment $(P,Q)$ it is, in some sense, 
the best mapping from the space $\Xcal$ to the reals. The next section shows how 
this statistic can be used as the basis for a variety of divergences measures
between $P$ and $Q$.

%%%%%%%%%%%%%%%%%%%%%%%%%%%%%%%%%%%%

\subsection{Csisz\'{a}r $f$-divergences}
\label{sub:fdivergence}

The class of \emph{$f$-divergences} \citep{AliSilvey1966,Csiszar1967} 
provide a rich set of relations that can
be used to measure the separation of the distributions in a binary experiment.
An $f$-divergence is a function that measures the
``distance'' between a pair of distributions $P$ and $Q$ defined over a space
$\mathcal{X}$ of observations. 
% The set of convex functions $f : (0,\infty) \to \RR$ will be denoted $\Fcal$. 
Traditionally, the $f$-divergence of $P$ from $Q$ is defined for any 
convex $f:(0,\infty)\to\RR$ such that $f(1) = 0$. In this case, the 
$f$-divergence is
\begin{equation}\label{eq:fdiv1}
	\II_f(P,Q) = \EE_Q \left[ f\left(\frac{dP}{dQ}\right) \right]
			   = \int_{\Xcal} f\left(\frac{dP}{dQ}\right)\,dQ
\end{equation}
when $P$ is absolutely continuous with respect to $Q$ and equal $\infty$ 
otherwise.\footnote{
	\citet[pg. 34]{Liese2008} give a definition that does not require 
	absolute continuity. 
}

The above definition is not completely well-defined as the behaviour of $f$ is 
not specified at the endpoints of $(0,\infty)$. This is remedied via the 
perspective transform of $f$, introduced in Section~\ref{sub:perspective} above
which defines the limiting behaviour of $f$.
Given convex $f : (0,\infty) \to \RR$ such that $f(1) = 0$ the 
\emph{$f$-divergence of $P$ from $Q$} is 
\begin{equation}\label{eq:fdiv}
	\II_f(P,Q) := 
	\E{M}{I_f(p,q)} =
	\E{\Xsf\sim M}{I_f(p(\Xsf),q(\Xsf))} ,
\end{equation}
where $I_f$ is the perspective transform of $f$.

The restriction that $f(1)=0$ in the above definition is only present to 
normalise $\II_f$ so that $\II_f(Q,Q) = 0$ for all distributions $Q$.
We can extend the definition of $f$-divergences to all convex $f$ 
by performing the normalisation explicitly. 
Since $f\left(\E{Q}{dP/dQ}\right)=f(1)$ this is done most conveniently through the 
definition of the Jensen gap for the function $f$ applied to the random variable 
$dP/dQ$ with distribution $Q$. 
That is, for all convex $f : (0,\infty) \to \RR$ and for all distributions $P$ 
and $Q$
\begin{equation}\label{eq:fdivjensen}
	\JJ_Q\left[ f\left(\frac{dP}{dQ}\right) \right]
	= \II_f(P,Q) - f(1).
\end{equation}

Due to the issues surrounding the behaviour of $f$ at 0 and $\infty$ the 
definitions in (\ref{eq:fdiv1}), (\ref{eq:fdiv}) and (\ref{eq:fdivjensen}) are 
not entirely equivalent. When it is necessary to deal with the limiting
behaviour, the definition in (\ref{eq:fdiv}) will be used. However, the version
in (\ref{eq:fdivjensen}) will be most useful when drawing connections between 
$f$-divergences and various definitions of information in 
Section~\ref{sec:risk} below.

%\subsubsection{Properties}
Several properties of $f$-divergence can be immediately obtained from the above 
definitions. 
The symmetry of the perspective $I_f$ in (\ref{eq:csiszardual}) means that 
\begin{equation}\label{eq:fdivdual}
	\II_f(P,Q) = \II_{\cdual{f}}(Q,P)
\end{equation}
for all distributions $P$ and $Q$, where $\cdual{f}$ is the Csisz\'{a}r dual
of $f$.
The non-negativity of the Jensen gap ensures that $\II_f(P,Q)\ge 0$ for all $P$ 
and $Q$. Furthermore, the affine invariance of the Jensen gap 
(Theorem~\ref{thm:jensen-affine}) implies the same affine invariance for
$f$-divergences.

%\subsubsection{Examples}
Several well-known divergences correspond to specific choices of the function 
$f$ \cite[\S 5]{AliSilvey1966}.
One divergence central to this paper is the \emph{variational divergence} 
$V(P,Q)$ which is obtained by setting $f(t) = |t-1|$ in Equation~\ref{eq:fdiv}.
It is the only $f$-divergence that is a true metric on the space of 
distributions over $\Xcal$ \citep{Khosravifard:2007} and gets its name from its 
equivalent definition in the variational form
\begin{equation}
	\label{eq:fdivvar}
	V(P,Q) = 2 \| P - Q \|_{\infty} 
	:= 2 \sup_{A \subseteq \Xcal} | P(A) - Q(A) |.
\end{equation}
(Some authors define $V$ without the 2 above.)
This form of the variational divergence leads is discussed further in 
Section~\ref{sec:variational}.
Furthermore, the variational divergence is one of a family of ``primitive'' 
$f$-divergences discussed in Section~\ref{sec:intrep}. These are primitive in 
the sense that all other $f$-divergences can be expressed as a weighted sum of 
members from this family. 
 
Another well known $f$-divergence is the Kullback-Leibler (KL) divergence 
$\mathrm{KL}(P,Q)$, obtained by setting $f(t) = t \ln(t)$ in 
Equation~\ref{eq:fdiv}. Others are given in 
Table~\ref{table:symmetric-divergences} in 
Section~\ref{sub:relating-SI-and-fdiv}.

%%%%%%%%%%%%%%%%%%%%%%%%%%%%%%
\subsection{Generative Bregman Divergences}
\label{sub:genBregDiv}
Another measure of the separation of distributions can defined as the expected
Bregman divergence between the densities $p$ and $q$ with respect to the
reference measure $M$. Given a convex function $\phi : \RR^+ \to \RR$ the
\emph{generative Bregman divergence} between the distributions $P$ and $Q$ is
(confer (\ref{eq:fdiv}))
\begin{equation}
	\BB_{\phi}(P,Q) := \E{M}{B_{\phi}(p,q)} =
	\E{\Xsf\sim M}{B_{\phi}(p(\Xsf),q(\Xsf))}.
\end{equation}
We call this Bregman divergence ``generative'' to distinguish it from the 
``discriminative'' Bregman divergence introduced in Section~\ref{sec:risk}
below, where the adjectives ``generative'' and ``discriminative'' are explained
further. 

\cite{Csiszar:1995} notes that there is only one divergence common to the class 
of $f$-divergences and the generative Bregman divergences. 
In this sense, these two classes of divergences are ``orthogonal'' to each 
other.
Their only common point is when the respective convex functions satisfy 
$f(t) = \phi(t) = t \ln t - at + b$ (for $a,b\in\reals$) in which case 
both $\II_f$ and $\BB_\phi$ are the KL divergence. 

%%%%%%%%%%%%%%%%%%%%%%%%%%%%%%%%%%%%%%%%%%%%%%%%%%%%%%%%%%%%%%%%%%%%%%%%%
\section{Risk and Statistical Information}
\label{sec:risk}

The above discussion of $f$-divergences assumes an arbitrary reference measure 
$M$ over the space $\Xcal$ to define the densities $p$ and $q$. 
In the previous section, the choice of reference measure was irrelevant since 
$f$-divergences are invariant to this choice.

In this section an assumption is made that adds additional structure 
to the relationship between $P$ and $Q$. 
Specifically, we assume that the reference measure $M$ is a mixture of these
two distributions. 
That is, $M = \pi P + (1-\pi) Q$ for some $\pi\in(0,1)$. 
In this case, by construction, $P$ and $Q$ are absolutely continuous with 
respect to $M$.
Intuitively, this can be seen as defining a distribution over the observation 
space $\Xcal$ by first tossing a coin with a bias $\pi$ for heads and drawing 
observations from $P$ on heads or $Q$ on tails.

This extra assumption allows us to interpret a binary experiment $(P,Q)$ as an
generalised \emph{supervised binary task} $(\pi,P,Q)$ where the 
positive ($y = 1$) and negative ($y = -1$) \emph{labels} 
$y \in \mathcal{Y} := \{-1,1\}$ are paired with \emph{observations} 
$x\in\mathcal{X}$ through a joint distribution $\PP$ over 
$\mathcal{X}\times\mathcal{Y}$. (We formally define a task later in terms of an
experiment plus loss function.)
Given an observation drawn from $\Xcal$ according to $M$, it is natural to try 
to predict its corresponding label or estimate the probability it was drawn from 
$P$. 

Below we will introduce risk, regret, and proper scoring rules and show 
how these
relate to discriminative Bregman divergence. We then show the connection
between the generative view ($f$-divergence between the class conditional
distributions) and Bregman divergence.

%%%%%%%%%%%%%%%%%%%%%%%%%%%%%%%%%%%%%%%%%%%%%%%%%%%%%%%%%%%%%%%
\subsection{Generative and Discriminative Views}
\label{sub:generative-discriminative}
Traditionally, the joint distribution $\PP$ of inputs $x\in\mathcal{X}$ and
labels $y\in\mathcal{Y}$ is used as the starting point for 
analysing risk in statistical learning theory. To better link risks to 
divergences, our analysis we will consider two related representations of $\PP$.

The \emph{generative} view decomposes $\PP$ into two \emph{class-conditional
distributions} defined as $P(X) := \PP(X|y=1)$, $Q(X) := \PP(X|y=-1)$ for 
all $X\subseteq\mathcal{X}$ and a mixing probability or \emph{prior} $\pi :=
\PP(\mathcal{X},y=1)$. 
The \emph{discriminative} representation decomposes the joint distribution into 
an \emph{observation distribution} $\obs(X) := \PP(X,\mathcal{Y})$ for all 
$X\subseteq\mathcal{X}$ and an \emph{observation-conditional density} or
\emph{posterior} $\eta(x) = \frac{dH}{d\obs}(x)$ where $H(X) := \PP(X,y=1)$.
The terms ``generative'' and ``discriminative'' are used here to suggest a distinction
made by \cite{Ng:2002}: in the generative case, the aim is to model the 
class-conditional distributions $P$ and $Q$ and then use Bayes rule to compute 
the most likely class; in the discriminative case the focus is on estimating 
$\eta(x)$ directly. Although we are not interested in this paper 
in the problems of modelling 
or estimating we find the distinction a useful one\footnote{
    The generative-discriminative distinction usually refers to whether one is
    modelling the process that generates each class-conditional distribution,
    or instead wishes solely to perform well on a discrimination task 
    \citep{Drummond,LasserreCambridgeBishopMinka2006, Minka:2005,
    RubinsteinHastie1997}.
    There has been some recent work relating the two in the sense 
    that if the class
    conditional distributions are well estimated then will one perform well in
    discrimination \citep{LongServedio2006, LongServedioSimon,
    goldberg2001ctu, palmer2006pcb}.
}.

Both these decompositions are exact since $\PP$ can be reconstructed from 
either.
Also, translating between them is straight-forward, since
\[
    \obs = \pi P + (1-\pi)Q\ \ \ \mbox{and}\ \ \  \eta = \pi \frac{dP}{d\obs},
\]
so we will often
swap between $(\eta, \obs)$ and $(\pi, P, Q)$ as arguments to functions for 
risk, divergence and information. 
A graphical representation of the generative and discriminative views of a 
binary task is shown in Figure~\ref{fig:discrim-gen}.

\begin{figure}[t]
	\centering
	%\subfigure[Generative]{
		\includegraphics[width=0.4\textwidth]{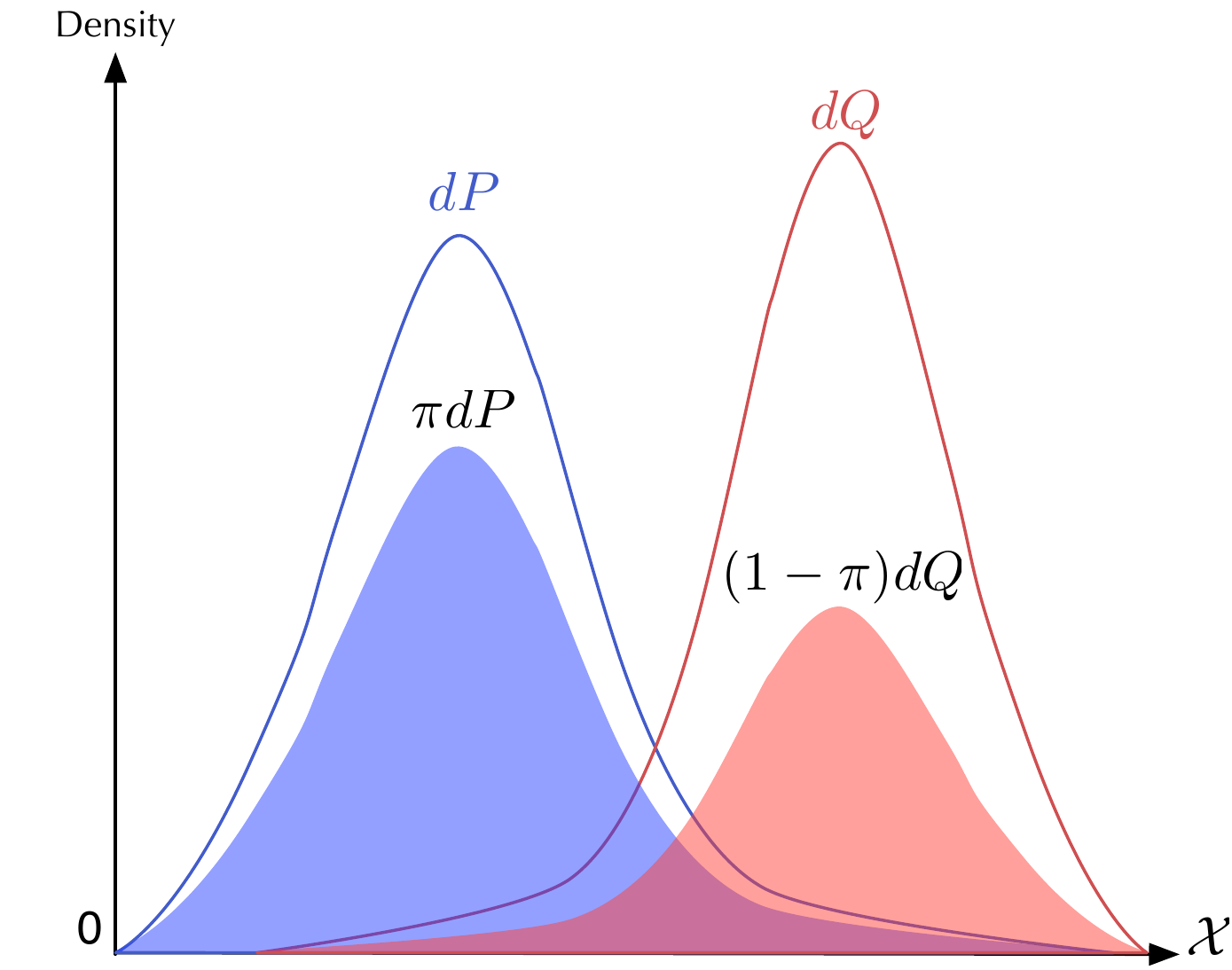}
		\label{fig:generative}
	%}
	%\subfigure[Discriminative]{
		\includegraphics[width=0.4\textwidth]{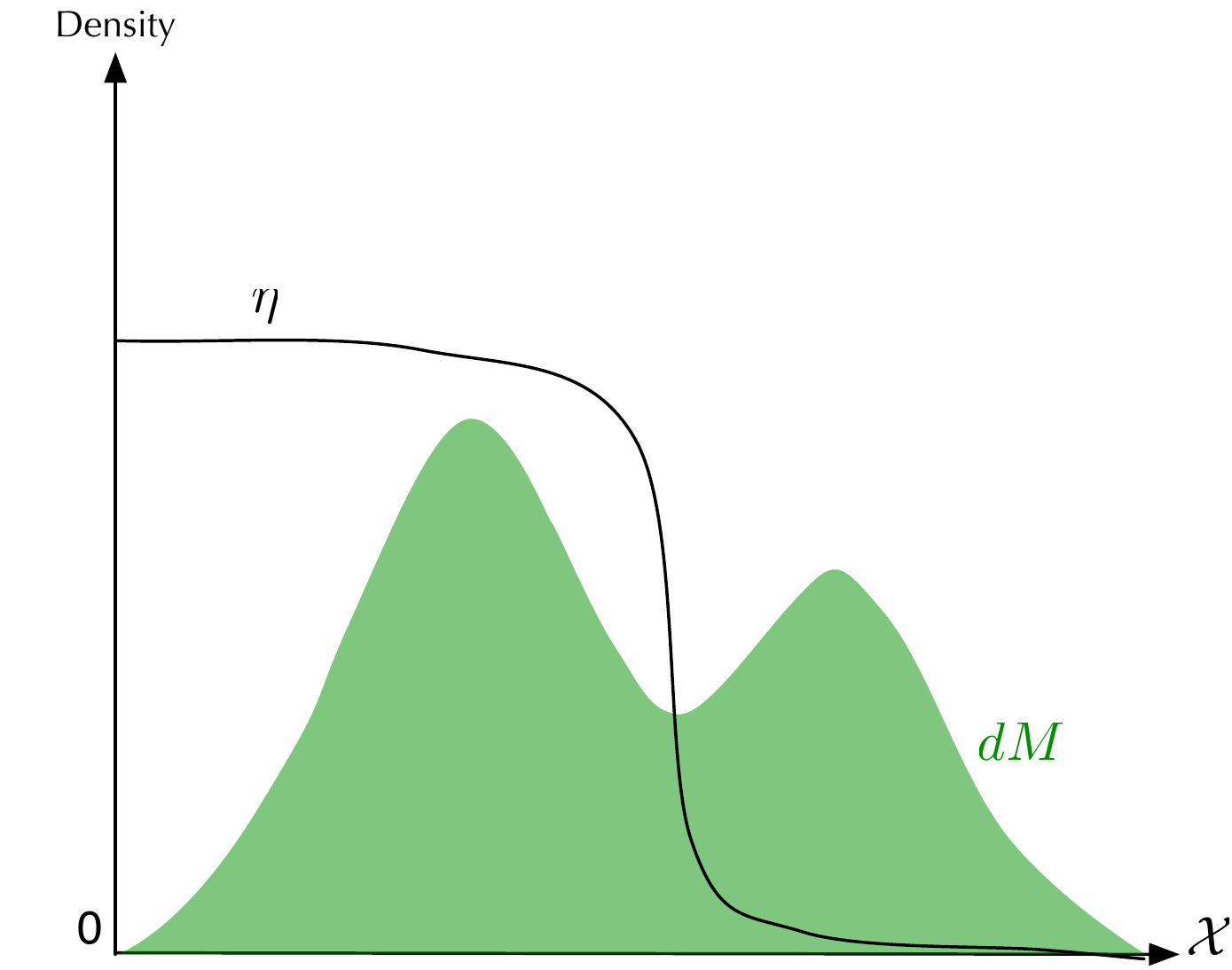}
		\label{fig:discriminative}
	%}
\caption{The generative and discriminative view of binary experiments.
\label{fig:discrim-gen}}
\end{figure}

The posterior $\eta$ is closely related to the likelihood ratio $dP/dQ$ in 
the supervised binary task setting.
For each choice of $\pi\in(0,1)$ this relationship can be expressed a mapping 
$\lambda_\pi : [0,1] \to [0,\infty]$ and its inverse $\lambda^{-1}_\pi$ defined 
by
\begin{eqnarray}
	\lambda_{\pi}(c) & := &
	\frac{1-\pi}{\pi}\frac{c}{1-c}\label{eq:lambda-pi-def} \\
	\lambda^{-1}_{\pi}(t) & = & \frac{\pi t}{\pi t + 1 - \pi}
\end{eqnarray}
for all $c\in[0,1)$ and $t\in[0,\infty)$ and $\lambda_{\pi}(1) := \infty$.
Thus
\[
	\eta = \lambda^{-1}_\pi\left(\frac{dP}{dQ}\right)
	\mbox{\ \ and, conversely,\ \ }
	\frac{dP}{dQ} = \lambda_{\pi}(\eta).
\]
These will be used later when connecting $f$-divergences and risk. 

%%%%%%%%%%%%%%%%%%%%%%%%%%%
\subsection{Estimators, Classifiers and Risk}
We will call a ($\obs$-measurable) function $\heta : \mathcal{X} \to [0,1]$ a
class probability \emph{estimator}. Overloading the notation slightly, we will
also use $\heta = \heta(x) \in [0,1]$ to denote an \emph{estimate}
for a specific observation $x\in\mathcal{X}$. Much of the subsequent arguments
rely on this conditional perspective.

Estimate quality is assessed using a \emph{loss function} 
$\loss : \mathcal{Y}\times [0,1]  \to \RR$ and the loss of the estimate $\heta$
with respect to the label $y\in\mathcal{Y}$ is denoted $\loss(y,\heta)$.
If $\eta\in[0,1]$ is the probability of observing the label $y=1$ the 
\emph{point-wise risk} of the estimate $\heta\in[0,1]$ is defined to be the 
$\eta$-average of the point-wise loss for $\heta$:
\begin{equation}
L(\eta,\heta) := \EE_{\mathsf{Y}\sim\eta}[\loss(\mathsf{Y},\heta)] 
				= \loss(0,\heta) (1-\eta) + \loss(1,\heta) \eta .
	\label{eq:L-eta-heta}
    \end{equation}
(This is what \cite{Steinwart2006} calls the \emph{inner risk}.)
When $\eta : \mathcal{X} \to [0,1]$ is an observation-conditional density, 
taking the $\obs$-average of the point-wise risk gives the \emph{(full) 
risk} of the estimator $\heta$:
\[
	\LL(\eta,\heta,\obs) := \E{M}{L(\eta,\heta)} 
	=\EE_{\Xsf\sim M} [L(\eta(\Xsf),\heta(\Xsf))]
	= \int_{\mathcal{X}} L(\eta(x),\heta(x))\,d\obs(x) 
	=: \LL(\pi,\heta,  P,Q).
\]
The convention of using $\loss$, $L$ and $\mathbb{L}$ for the loss, point-wise 
and full risk is used throughout this paper.

We call the combination of a loss $\loss$ and the distribution $\PP$ a 
\emph{task} and denote it discriminatively as $T = (\eta, \obs; \loss)$ 
or generatively as $T = (\pi, P, Q; \loss)$. A natural measure of the 
difficulty of a task is its minimal achievable risk, or \emph{Bayes risk}: 
\[
	\minLL(\eta, \obs) = \minLL(\pi,P,Q)
	:= \inf_{\heta \in[0,1]^{\mathcal{X}}} \LL(\eta,\heta,\obs)
	= \EE_{\mathsf{X}\sim\obs}\left[\minL(\eta(\mathsf{X}))\right],
\]
where 
\[
[0,1]\ni\eta\mapsto\minL(\eta) := \inf_{\heta\in[0,1]} L(\eta,\heta)
\]
is 
the \emph{point-wise Bayes risk}. Note the use of the underline on $\minLL$ and
$\minL$ to indicate that the corresponding functions $\LL$ and $L$ are 
minimised.  

%%%%%%%%%%%%%
\subsubsection{Proper Scoring Rules}
If $\heta$ is to be interpreted as an estimate of the true positive class
probability $\eta$ then it is desirable to require that $L(\eta,\heta)$ be 
minimised by $\heta = \eta$ for all $\eta\in[0,1]$. Losses that satisfy this
constraint are said to be \emph{Fisher consistent} and are known as 
\emph{proper scoring rules}
\citep{Buja:2005,GneitingRaftery2007}.
That is, a proper scoring rule $\loss$ satisfies $\minL(\eta) = L(\eta,\eta)$ 
for all $\eta\in[0,1]$. 

Proper scoring rules for probability estimation and surrogate margin losses
(confer \citet{BartlettJordanMcAuliffe2006})
for classification are closely related. (Surrogate margin losses are considered
in more detail in Appendix~\ref{section:surrogate}.)
\citet{Buja:2005} note that ``the surrogate criteria of classification
are exactly the primary criteria of class probability estimation'' and that
most commonly used surrogate margin losses are just proper scores mapped from 
$[0,1]$ to $\RR$ via a link function. The main exceptions are hinge 
losses\footnote{
	And powers of absolute divergence $|y-r|^{\alpha}$ for $\alpha \ne 2$.
} which means SVMs are ``the only case that truly bypasses estimation of 
class probabilities and directly aims at classification'' 
\cite[pg. 4]{Buja:2005}.
However, commonly used margin losses of the form $\phi(y F(x))$ are a more 
restrictive class than proper scoring rules since, as 
\citet[\S 23]{Buja:2005} note, ``[t]his dependence on the margin 
limits all theory and practice to a symmetric treatment of class 0 and 
class 1''.

The following important property of proper scoring rules is originally 
attributed to \cite{Savage:1971}.

\begin{theorem}\label{pro:minrisk_concave}
	The point-wise Bayes risk $\minL(\eta)$ for a proper scoring rule
	$\loss$ is concave function. Conversely, given a concave function
	$\Lambda : [0,1] \to \RR$ there exists a proper scoring rule $\loss$ so
	that $\minL(\eta) = \Lambda(\eta)$ and 
	\begin{equation}
		L(\eta,\heta) = \minL(\heta) - (\heta - \eta)\minL'(\heta).
		\label{eq:savage-rep}
	\end{equation}
\end{theorem}

\citet[\S 17]{Buja:2005} provide a proof which relies on Fisher consistency and
the linearity of $L(\eta,\heta)$ in $\eta$ which means the functions $\eta
\mapsto L(\eta,\heta)$ are upper tangents to $\minL(\heta)$ for all
$\heta\in[0,1]$. We provide an alternate proof below immediately 
after the proof of Theorem~\ref{theorem:L-w-general-form} which provides a
general explicit formula for $\minL(\eta)$.

This characterisation of the concavity of $\minL$ means proper scoring rules 
have a natural connection to Bregman divergences.

%%%%%%%%%%%%%%%%%%%%%%%%%%%%%%%%%%%%%%%%%%%%%%%%%%%%%%
\subsection{Discriminative Bregman Divergence}\label{sub:discBregDiv}
Recall from Section~\ref{sub:Bregman} that if $\mathcal{S} \subseteq \RR^d$ is a 
convex set, then a convex function $\phi : \mathcal{S} \to \RR$ defines a 
\emph{Bregman divergence}
\[
	B_\phi(s, s_0) 
	:= \phi(s) - \phi(s_0) - \inner{s - s_0}{\nabla\phi(s_0)}.
\]
When $\mathcal{S} = [0,1]$, the concavity of $\minL$ means $\phi(s) = -\minL(s)$
is convex and so induces the Bregman divergence\footnote{
	Technically, $\mathcal{S}$ is the 2-simplex 
	$\{(s_1,s_2) \in [0,1]^2 : s_1 + s_2 = 1\}$ but we
	identify $s\in[0,1]$ with $(s,1-s)$.}
\[
	B_\phi(s, s_0) 
	= -\minL(s) + \minL(s_0) - (s_0 - s)\minL'(s_0)
	 = L(s,s_0) - \minL(s)
\]
by Theorem~\ref{pro:minrisk_concave}.
The converse also holds. Given a Bregman divergence $B_\phi$ 
over $\mathcal{S}=[0,1]$ the convexity of $\phi$ guarantees that $\minL = -\phi$ 
is concave. 
Thus, we know that there is a proper scoring rule $\loss$ with Bayes risk equal 
to $-\phi$. 
As noted by \citet[\S 19]{Buja:2005}, 
the difference 
\[
	B_\phi(\eta,\heta) = L(\eta,\heta) - \minL(\eta)
\] 
is also known as the \emph{point-wise regret} of the estimate $\heta$ w.r.t. 
$\eta$. 
The corresponding \emph{(full) regret} is the $\obs$-average point-wise regret
\[
\EE_{\Xsf\sim \obs}[B_\phi(\eta(\Xsf),\heta(\Xsf))] = 
    \LL(\eta,\heta) - \minLL(\eta).
\]

%%%%%%%%%%%%%%%%%%%%%%%%%%%%%%%%%%%%%%%%%%%%%%%%%%%%%%%%%%%
\subsection{Bregman Information}\label{sub:BregInfo}
\cite{Banerjee:2005} recently introduced the notion of the \emph{Bregman 
information} $\BB_{\phi}(\mathsf{S})$ of a random variable $\mathsf{S}$ drawn 
according to some distribution $\sigma$ over $\mathcal{S}$. 
It is the minimal $\sigma$-average Bregman divergence that can be achieved by an 
element $s^* \in \mathcal{S}$ (the \emph{Bregman representative}).  In symbols,
\[
\BB_\phi(\mathsf{S}) 
:= \inf_{s\in\mathcal{S}}\EE_{\Ssf\sim\sigma}\left[B_\phi(\mathsf{S}, s)\right]
= \EE_{\Ssf\sim\sigma}\left[ B_\phi(\mathsf{S}, s^*) \right].
\]
The authors show that the mean $\bar{s} := \EE_{\Ssf\sim\sigma}[\mathsf{S}]$, is the 
unique Bregman representative. That is, 
\(
\BB_\phi(\mathsf{S}) 
= \EE_{\sigma}[B_{\phi}(\mathsf{S}, \bar{s})].
\)
Surprisingly, this minimiser \emph{only} depends on $\mathcal{S}$ and $\sigma$, 
not the choice of $\phi$ defining the divergence and is a consequence of 
Jensen's inequality and the form of the Bregman divergence.

Since regret is a Bregman divergence, it is natural to ask what is the 
corresponding Bregman information. In this case, $\phi = -\minL$ and the
random variable $\mathsf{S} = \eta(\mathsf{X}) \in [0,1]$ where 
$\mathsf{X}\in\mathcal{X}$ is distributed 
according to the observation distribution $\obs$. Noting that 
$\EE_{\Xsf\sim \obs}[\eta(\mathsf{X})] = \pi$, the proof of the following theorem stems from the 
definition of Bregman information and some simple algebra showing that 
$\LL(\pi , \eta, \obs) = \minLL(\pi,\obs)$, since by assumption $\loss$ is
proper scoring rule.
\begin{theorem}
	\label{pro:breg_statinfo}
	Suppose $\loss$ is proper scoring rule.
	Given a discriminative task $(\eta, \obs)$ and letting
	$\phi = -\minL$, the corresponding Bregman information of 
	$\eta(\Xsf)$ satisfies
	\[
		\BB_\phi(\eta(\Xsf)) 
		= \minLL(\pi,\obs) - \minLL(\eta,\obs).
	\]
\end{theorem}

%%%%%%%%%%%%%%%%%%%%%%%%%%%%%%
\subsubsection{Statistical Information}
The reduction
of risk 
\begin{equation}\label{eq:statinfo}
	\SI(\eta,\obs) = \SI(\pi, P, Q) := \minLL(\pi,\obs) - \minLL(\eta,\obs)
\end{equation}
is known as \emph{statistical information}
and was introduced by \cite{DeGroot1962}.
This reduction can be interpreted as how much risk is removed by knowing
observation-specific class probabilities $\eta$ rather than just the average
$\pi$. 

DeGroot originally introduced statistical information in terms of what he
called an \emph{uncertainty function} which, in the case of binary experiments,
is any 
function $U : [0,1] \to [0,\infty)$. The statistical information is then the 
average reduction in uncertainty which can be expressed as a concave Jensen gap 
\[
	-\JJ_M[U(\eta)] = \JJ_M[-U(\eta)] = U(\E{\Xsf\sim M}{\eta(\Xsf)}) - 
	\E{\Xsf\sim M}{U(\eta(\Xsf))}.
\]
DeGroot noted that Jensen's inequality implies that for this quantity to be 
non-negative the uncertainty function must be concave, that is, $-U$ must be
convex. 

Theorem~\ref{pro:breg_statinfo} shows that statistical
information is a Bregman information and corresponds to the Bregman divergence 
obtained by setting $\phi = -\minL$. This connection readily shows that  
$\SI(\eta, \obs) \ge 0$ \cite[Thm 2.1]{DeGroot1962} since  
the minimiser of the Bregman information is 
$\pi = \EE_{\Xsf\sim \obs}[\eta(\Xsf)]$ 
regardless of loss and $B_\phi(\eta,\pi) \geq 0$ since it is a regret.   

%%%%%%%%%%%%%%%%%%%%%%%%%%%%%%%%%%%%%%%%%%%%%%%%%%
\subsubsection{Unifying Information and Divergence}
From a generative perspective, $f$-divergences can be used to assess the
difficulty of a learning task by measuring the divergence between the 
class-conditional distributions $P$ and $Q$. The more divergent the 
distributions for the two classes, the easier the classification task.
\citet[Thm. 2]{OsterreicherVajda1993a} made this relationship precise by 
showing that $f$-divergence and statistical information have a one-to-one 
correspondence:
\begin{theorem}\label{thm:duality}
If $(\pi, P, Q; \loss)$ is an arbitrary task then defining 
\begin{equation}
	f^\pi(t) 
	:= \minL(\pi) - (\pi t + 1 - \pi)
	\minL\left(\frac{\pi t}{\pi t + 1 - \pi}\right)
	\label{eq:OV-L-to-f}
    \end{equation}
for $\pi\in[0,1]$ implies $f^\pi$ is convex and $f^\pi(1)=0$ and 
\[
	I_{f^\pi}(P,Q) = \SI(\pi, P, Q)
\]
for all distributions $P$ and $Q$.
Conversely, if $f$ is convex and $f(1)= 0$ and $\pi\in[0,1]$ then defining
\[
	\minL^{\pi}(\eta) 
	:=
	-\frac{1-\eta}{1-\pi}f\left(\frac{1-\pi}{\pi}\frac{\eta}{1-\eta}\right)
\]
implies 
\[
	I_{f}(P, Q) = \SI^\pi(\pi, P, Q)
\]
for all distributions $P$ and $Q$.
\end{theorem}

The proof, given in Appendix~\ref{app:infodiv-duality}, is a straight-forward 
calculation that exploits the relationships between the generative and 
discriminative views presented earlier.  
Combined with the link between Bregman and statistical information, 
this result means that they and $f$-divergences are \emph{interchangeable} as 
measures of task difficulty. The theorem leads to some correspondences between
well known losses and divergence:
log-loss with $\mathrm{KL}(P,Q)$; square loss with triangular 
discrimination; and 0-1 loss with $V(P,Q)$. (See
Section~\ref{section:example-squared-loss} for an explicitly
worked out example.)

This connection generalises the link between $f$-divergences 
and $F$-errors (expectations of concave functions of $\eta$) in 
\cite{DevGyoLug96} and can be compared the more recent work 
of \cite{Nguyen:2005} who show that each $f$-divergence 
corresponds to the negative Bayes risk for a \emph{family} of surrogate margin 
losses. The one-to-many nature of their result may seem at odds with the 
one-to-one relationship here. However, the family of margin losses given in 
their work can be recovered by combining the proper scoring rules with link 
functions. Working with proper scoring rules also addresses a limitation
pointed out by \citet[pg. 14]{Nguyen:2005}, namely that 
``asymmetric $f$-divergences cannot be generated by 
\emph{any} (margin-based) surrogate loss function'' and extends their analysis 
``to show that asymmetric $f$-divergences can be realized by general 
(asymmetric) loss functions''. 

\subsection{Summary}
The main results of this section can be summarised as follows.
\begin{theorem}
	Let $f : [0,\infty) \to \RR$ be a convex function and for each 
	$\pi \in [0,1]$ define for $c\in[0,1)$:
	\begin{eqnarray}
		\phi(c) & := & \frac{1-c}{1-\pi}f\left(\lambda_{\pi}(c)\right) \\
		\minL(c) & := & -\phi(c)
	\end{eqnarray}
	where $\lambda_\pi$ is  defined by (\ref{eq:lambda-pi-def}).
	Then for every binary experiment $(P,Q)$ we have
	\begin{equation}
		\II_f(P,Q) = \SI(\eta,M) = \BB_{\phi}(\eta, M)
	\end{equation}
	where $M := \pi P + (1-\pi) Q$, $\eta := \pi dP/dM$ and $\minLL$ is the
	expectation (in $\mathsf{X}$) of the conditional Bayes risk $\minL$.
	Equivalently,	
	\begin{equation}
		\JJ_{Q}[f(dP/dQ)] = \JJ_{M}[-\minL(\eta)] = \JJ_{M}[\phi(\eta)].
	\end{equation}
\end{theorem}
What this says is that for each choice of $\pi$ the classes 
of $f$-divergences $\II_f$, statistical informations $\SI$ and (discriminative) 
Bregman informations $\BB_\phi$ can all be defined in terms of the Jensen gap of 
some convex function. Additionally, there is a bijection between each of these
classes due to the mapping $\lambda_{\pi}$ that identifies likelihood ratios
with posterior probabilities.

It is important to note that the class of $f$-divergences is more ``primitive''
than the other measures since its definition does not require the extra 
structure that is obtained by assuming that the reference measure $M$ can be
written as the convex combination of the distributions $P$ and $Q$. 
Indeed, each $\II_f$ is invariant to choice of reference measure and so is
invariant to the choice of $\pi$.
The results in the next section provide another way of looking at this 
invariance of $\II_f$. 
In particular, we see that every $f$-divergence is a weighted ``average'' of
statistical informations or, equivalently, $\II_{f_\pi}$ divergences.

\section{Primitives and Weighted Integral Representations}
\label{sec:intrep}

When given a class of functions like $f$-divergences, risks and measures of 
information it is natural to ask what the ``simplest'' elements of these classes 
are. 
We would like to know which functions are ``primitive'' in the sense that they 
can be used to express other measures but themselves cannot be so expressed.

The main result of this section is that risks and $f$-divergences (and therefore 
also statistical and Bregman information) can be expressed as weighted integrals 
of these primitive elements. In the case of $f$-divergences and information the 
weight function in these integrals completely determines their behaviour. This 
means the weight functions can be used as a proxy for the analysis of these
measures, or as a ``knob'' the user can adjust in choosing what to measure.

We also show that the close
relationships between information and $f$-divergence can be directly translated
into a relationship between the weight functions of these measures. 
That is, given the weight function that determines an $f$-divergence there is,
for each choice of the prior $\pi$, a simple transformation that yields the
weight function for the corresponding statistical information, and \emph{vice
versa}.

%%%%%%%%%%%%%%%%%%%%%%%%%%%%%%%%%%%%%%%%%%%%%%%%%%%%%%%%%%%%%
\subsection{Integral Representations of $f$-divergences}

The following result shows that the class of $f$-divergences (and, by the result 
of the previous section, statistical and Bregman information) is closed under 
linear combination.
\begin{theorem}
	For all convex functions $f_1, f_2 \colon (0,\infty)\to\RR$ and 
	all $\alpha_1, \alpha_2 \in [0,\infty)$, the function
	\begin{equation}\label{eq:convexg}
		(0,\infty)\ni t\mapsto g(t) := \alpha_1 f_1(t) + \alpha_2 f_2(t)
	\end{equation} 
	is convex. 
	Furthermore, for all distributions $P$ and $Q$, we have
	\begin{equation}\label{eq:convex-fdiv}
		\II_g(P,Q) = \alpha_1 \II_{f_1}(P,Q) + \alpha_2 \II_{f_2}(P,Q).
	\end{equation}
	Conversely, given $f_1$, $f_2$, $\alpha_1$ and $\alpha_2$, if 
	(\ref{eq:convex-fdiv}) holds for all $P$ and $Q$ then $g$ must be, up to 
	affine additions, of the form (\ref{eq:convexg}). 
\end{theorem}
The proof is a straight-forward application of the definition of convexity and
of $f$-divergences.

One immediate consequence of this result is that the set of $f$-divergences is 
closed under finite linear combinations $\sum_i \alpha_i \II_{f_i}$.
Furthermore, the integral representations discussed in Section~\ref{sub:intrep} 
extend this observation beyond finite linear combination to generalised weight 
functions $\alpha$. 
By Corollary~\ref{cor:intrep1}, if $f$ is a convex function then 
expanding it about $1$ in (\ref{eq:intrep1}) and setting $\alpha(s) = f''(s)$ 
means that
\begin{equation}\label{eq:If-alpha}
	\II_f(P,Q) = \int_0^{\infty} \II_{F_s}(P,Q)\,\alpha(s)\,ds
\end{equation}
where $F_s(t) = \test{s \le 1}(s-t)_+ + \test{s > 1}(t-s)_+$.\footnote{
	Technically, one must assume that $f$ is twice differentiable for this 
	result to hold. However, the convexity of $f$ implies it has well-defined
	one-sided derivatives $f'_{+}$ and $\alpha(s)$ can be expressed as the 
	measure corresponding to $df'_{+}/d\lambda$ for the Lebesgue measure 
	$\lambda$. Details can be found in \citep{LieVaj06}.  The
	representation of a general $f$-divergence in terms of elementary ones
	is not new; see for example \cite{FeldmanOsterreicher1989}.
}
The set of functions $\{F_s\}_{s=0}^{\infty}$ can therefore be seen as the 
generators of the class of primitive $f$-divergences. 
As a function of $t$, each $F_s$ is piece-wise linear, with a single ``hinge'' 
at $s$.
Of course, any affine translation of any $F_s$ is also a primitive. In fact,
each $F_s$ may undergo a different affine translation without changing the 
$f$-divergence $\II_f$. The weight function $\alpha$ is what completely 
characterises the behaviour of $\II_f$.

The integral in (\ref{eq:If-alpha}) need not always exist since the integrand
may not be integrable. When the Cauchy Principal Value diverges we say the
integral takes on the value $\infty$. We note that many (not all)
$f$-divergences can sometimes take on infinite values.

The integral form in (\ref{eq:If-alpha}) can be readily transformed into an
integral representation that does not involve an infinite integrand. 
This is achieved by mapping the interval $[0,\infty)$ onto $[0,1)$ via the
change of variables 
$\pi = \frac{1}{1+s} \in[0,1]$.
In this case, $s = \frac{1-\pi}{\pi}$ and so $ds = -\frac{d\pi}{\pi^2}$ and the 
integral of (\ref{eq:If-alpha}) becomes
\begin{eqnarray}
	\II_f(P,Q) 
	&=& -\int_1^0
			\II_{F_{\frac{1-\pi}{\pi}}}(P,Q)\,\alpha(\tfrac{1-\pi}{\pi})
			\,\pi^{-2}
		\,d\pi \nonumber \\
	&=& \int_0^1
			\II_{\tilde{f}_\pi}(P,Q)\,\gamma(\pi)\,d\pi \label{eq:fintrep01}
\end{eqnarray}
where 
\begin{equation}\label{eq:fpi}
	\tilde{f}_\pi(t) 
		:= \pi F_{\frac{1-\pi}{\pi}}(t) 
		= 
		\begin{cases}
			(1-\pi(1+t))_{+} \ ,& \pi \ge \thalf \\
			(\pi(1+t)-1)_{+} \ ,& \pi < \thalf
		\end{cases}	
\end{equation}
and 
\begin{equation}\label{eq:gamma}
	\gamma(\pi) := \frac{1}{\pi^3}f''\left(\frac{1-\pi}{\pi}\right).
\end{equation}
This observation forms the basis of the following theorem which will be used 
to discuss the connection between $f$-divergences and statistical 
information.\footnote{
	The $1/\pi^3$ term in the definition of $\gamma$ seems a little unusual at 
	first glance. However, it is easily understood as the product of two terms: 
	$1/\pi^2$ from the second derivative of $(1-\pi)/\pi$, and $1/\pi$ from a 
	transformation of variables within the integral to map the limits of 
	integration from $(0,\infty)$ to $(0,1)$ via $\lambda_{\pi}$.
}

\begin{theorem}\label{thm:fdivchoquet}
	Let $f$ be convex such that $f(1)=0$. Then there exists a 
	(generalised) function $\gamma:(0,1)\to\RR$
	such that, for all $P$ and $Q$: 
	\[
		\II_f(P,Q) = \int_0^1 \II_{f_\pi}(P,Q)\,\gamma(\pi)\,d\pi,
		\ \mbox{where}\  
		f_\pi(t) = (1-\pi)\wedge \pi - (1-\pi)\wedge (\pi t).
	\]
\end{theorem}
\begin{proof}
	The earlier discussion giving the derivation of equation 
	(\ref{eq:fintrep01}) implies the result. The only discrepancy is over the
	form of $f_\pi$. However, this is remedied by noting that the family of
	$\tilde{f}_\pi$ given in (\ref{eq:fpi}) can be transformed by affine 
	addition without affecting the representation of $\II_f$. Specifically, 
	\begin{eqnarray*}
		f_\pi(t) 
		&:=& (1-\pi)\wedge\pi - (1-\pi)\wedge(\pi t) \\
		& =& 
			\begin{cases}
				(1-\pi(1+t))_{+}             \ , & \pi \ge \thalf \\
				(\pi(1+t) - 1)_{+} + \pi(1-t)\ , & \pi < \thalf
	 		\end{cases} \\
		& =& \tilde{f}_\pi(t) + \test{\pi < \thalf}\pi(1-t)
	\end{eqnarray*}
	and so $\tilde{f}_\pi$ and $f_\pi$ are in the same affine equivalence
	class for each $\pi\in[0,1]$. Thus, by Theorem~\ref{thm:jensen-affine}
	we have $\II_{f_\pi} = \II_{\tilde{f}_\pi}$ for each $\pi\in[0,1]$, proving 
	the result. 
\end{proof}

The specific choice of $f_\pi$ in the above theorem from all of the affine 
equivalents was made to make simpler the connection between integral 
representations for losses and $f$-divergences, discussed in
Section~\ref{sub:relating-SI-and-fdiv}.

One can easily verify that $f_\pi$ are convex hinge functions
of $t$ with a hinge at $\frac{1-\pi}{\pi}$ and $f_\pi(1)=0$.
Thus $\{\II_{f_\pi}\}_{\pi\in(0,1)}$
is a family
of primitive $f$-divergences; confer~\cite{OsterreicherFeldman1981,
FeldmanOsterreicher1989}.
This theorem implies an existing representation of
$f$-divergences due to \citet[Thm. 1]{OsterreicherVajda1993a} and 
\citet{Gutenbrunner:1990}. 
They show that an $f$-divergence can be represented as a weighted integral  of 
statistical informations for 0-1 loss: for all $P,Q$
\begin{eqnarray}
    \II_f(P,Q) &=& \int_0^1 \SI^{0-1}(\pi,P,Q) \gamma(\pi) d\pi\
    \label{eq:LV-integral-rep}\\
    \gamma(\pi)&=&\frac{1}{\pi^3} f''\left(\frac{1-\pi}{\pi}\right) .
    \label{eq:gamma-f-relationship}
\end{eqnarray}

An $f$ divergence is \emph{symmetric} if $\II_f(P,Q)=\II_f(Q,P)$ for all $P,Q$.
The representation of $\II_f$ in terms of
$\gamma$ and Theorem~\ref{pro:c_vs_pi} provides an easy test for 
symmetry: 
\begin{corollary}
    \label{corollary:symmetric-f-div}
    Suppose $I_f$ is an $f$-divergence with corresponding weight function 
    $\gamma$ given by (\ref{eq:gamma-f-relationship}).
    Then $\II_f$ is symmetric 
    iff
     $\gamma(\pi)=\gamma(1-\pi)$ for all $\pi\in[0,1]$.
\end{corollary}

\begin{proof}
    Let $\cdual{f}(t):=t f(1/t) $ denote the  Csisz\'{a}r-dual of $f$
as described in Section~\ref{sub:perspective} above. It is known 
(see (\ref{eq:fdivdual}) and e.g.~\cite{LieVaj06}) that 
\[
\II_f(P,Q) = \II_{\cdual{f}}(Q,P) \ \ \ \mbox{if and only if}\ \ \ \
f(t)=\cdual{f}(t)+c(t-1)
\]
for some $c\in\reals$. Since $f$ and $\gamma$ are related by
$f''\left(\frac{1-\pi}{\pi}\right)=\pi^3 \gamma(\pi)$ we can argue as follows.
Observe that ${\cdual{f}}'(t)=f(1/t)-f'(1/t)/t$ and ${\cdual{f}}''(t)=f''(1/t)/t^3$.
Hence
${\cdual{f}}''\left(\frac{1-\pi}{\pi}\right)=f''\left(\frac{\pi}{1-\pi}\right)\left(\frac{\pi}{1-\pi}\right)^3$.
Let $\pi'=1-\pi$.  Thus $\frac{1-\pi}{\pi}=\frac{\pi'}{1-\pi'}$. Hence
\begin{eqnarray*}
    {\cdual{f}}''\left(\frac{1-\pi}{\pi}\right) &=&
    f''\left(\frac{1-\pi'}{\pi'}\right)\left(\frac{\pi}{1-\pi}\right)^3\\
    &=& {\pi'}^3 \gamma(\pi') \left(\frac{\pi}{1-\pi}\right)^3\\
    &=& \pi^3\gamma(1-\pi) .
\end{eqnarray*}
Thus if $\gamma(1-\pi)=\gamma(\pi)$, we have shown $\pi\mapsto\gamma(1-\pi)$ is
the weight corresponding to $\cdual{f}$. Observing that
$\frac{\partial^2}{\partial t^2} (\cdual{f}(t)+c(t-1))={\cdual{f}}''$ concludes 
the proof.

\end{proof}

Corollary \ref{corollary:symmetric-f-div} provides a way of generating
all convex $f$ such that $\II_f$ is symmetric that is simpler than that proposed
by \cite{Hiriart-Urruty2007}: let $\gamma(\pi)=\beta(\pi\wedge(1-\pi))$ where 
$\beta\in(\reals^+)^{[0,\thalf]}$ and generate $f$
from $\gamma$ by inverting (\ref{eq:gamma-f-relationship}); explicitly,
\[
    f(s)= \int_0^s \left(\int_0^t \frac{1}{(\tau+1)^3}
    \gamma\left(\frac{1}{\tau+1}\right) d\tau\right) dt,\ \ s\in\reals^+.
\]

%%%%%%%%%%%%%%%%%%%%%%%%%%%%%%%%%%%%%%%%%%%%%%%%%%
\subsection{Proper Scoring Rules and Cost-Weighted Risk}
We now consider a representation of proper scoring rules in terms of primitive 
losses that originates with  \cite{ShufordAlbertMassengill1966}. Our 
discussion follows that of \cite{Buja:2005} and then examines its implications 
in light of the connections between information and divergence just presented.

The \emph{cost-weighted losses} are a family of losses parametrised by a false 
positive cost $c\in[0,1]$ that defines a loss for $y\in\{\pm 1\}$ and 
$\heta\in[0,1]$ by
\begin{equation} \label{eq:costloss}
    \ell_c(y,\heta) 
    = c \test{y = -1}\test{\heta \ge c} + (1-c)\test{y = 1}\test{\heta < c}.
\end{equation}
Intuitively, a cost-weighted loss thresholds $\heta$ at $c$ and assigns a cost
if the resulting classification disagrees with $y$. 
These correspond to the ``signatures'' for eliciting the probability $\eta$ as 
described by \cite{Lambert:2008}.
Substituting $c=\half$ will verify that $2l_{\half}$ is equivalent to 0-1 
misclassification loss $\loss^{0-1}$. 

We will use $L_c$, $\LL_c$ and $\SI_c$ to denote the cost-weighted point-wise 
risk, full risk and statistical information associated with each cost-weighted 
loss.
The following theorem collect some useful observations about these primitive 
quantities. The first shows that the point-wise Bayes risk is a simple, concave 
``tent'' function. 
The second shows that cost-weighted statistical information is invariant under 
the switching of the classes provided the costs are also switched and that $\pi$ 
and $1-c$ are interchangeable. 

\begin{theorem} \label{pro:mincostrisk}
For all $\eta,c\in[0,1]$ the point-wise Bayes risk 
$\minL_c(\eta) = (1-\eta)c \wedge (1-c)\eta$ and is therefore concave in both 
$c$ and $\eta$.
\end{theorem}
\begin{proof}
From the definition of $\loss_c$ in equation~\ref{eq:costloss} and the 
definition of point-wise Bayes risk, we have for $\eta\in[0,1]$
\begin{eqnarray*}
    \minL_c(\eta)
    & = & \inf_{\heta\in[0,1]} L_c(\eta,\heta) \\
    & = & \inf_{\heta\in[0,1]}
            \{
                (1-\eta) c \test{\heta \geq c} + \eta (1-c)\test{\heta < c}
            \}\\
    & = & \inf_{\heta\in[0,1]}
            \{
                \eta(1-c) + (c-\eta)\test{\heta \geq c},
            \}
\end{eqnarray*}
where the last step makes use of the identity
$\test{\heta < c} = 1 - \test{\heta \geq c}$.
Since $(c-\eta)$ is negative if and only if $\eta > c$ the infimum is obtained
by having $\test{\heta \geq c} = 1$ if and only if $\eta \geq c$, that is, by
letting $\heta = \eta$. In this case, when $\heta \geq c$ we have
$\minL_c(\eta) = c(1-\eta)$ and when $\heta < c$ we have
$\minL_c(\eta) = (1-c)\eta$. The concavity of $\minL_c$ is evident as this
function is the minimum of two linear functions of $c$ and $\eta$.
\end{proof}

\begin{theorem} \label{pro:c_vs_pi}
For all $c\in[0,1]$ and tasks $(\eta,\obs; \loss_c) = (\pi, P, Q; \loss_c)$
the statistical information satisfies 1) 
\[
\SI_c(1-\eta, \obs) = \SI_{1-c}(\eta, \obs),
\]
or equivalently, 
\[
\SI_c(1-\pi, Q, P) = \SI_{1-c}(\pi, P, Q);
\]
and 2) 
\[
\SI_\pi(1-c, P, Q) = \SI_c(1-\pi, P, Q).
\]
\end{theorem}
\begin{proof}
By Theorem~\ref{pro:mincostrisk} we know
$\minL_c(\eta) = \min\{(1-\eta)c, (1-c)\eta\}$
and so $\minL_c(1-\eta) = \minL_{1-c}(\eta)$ for all $\eta, c\in[0,1]$.
Therefore, $\minLL_c(1-\eta, M) = \minLL_{1-c}(\eta, M)$ for any
$\eta : \mathcal{X} \to [0,1]$ including the constant function $\EE_M[\eta]$.
By definition, $\SI_c(\eta, M) = \minLL(\EE_M[\eta], M) - \minLL(\eta, M)$
and so $\SI_{1-c}(\eta, M) = \SI_c(1-\eta, M)$ proving part 1 of the lemma.

Part 2 also follows from Theorem~\ref{pro:mincostrisk} by noting that
$\minL_c(1-\pi) = \minL_{\pi}(1-c)$ and that
$\EE_M[\minL_c(\eta)] = \int_{\mathcal{X}}\min\{(1-c)\pi\,dP, (1-\pi)c\,dQ\}$.
\end{proof}

%%%%%%%%%%%%%%%%%%%%%%%%%%%%%%%%%%%%%%%%%%%%%%%%%%
\subsection{Integral Representations of Proper Scoring Rules}
The cost-weighted losses are primitive in the sense that they form the basis
for a Choquet integral representation of proper scoring rules. This representation is essentially a consequence of Taylor's theorem and was 
originally studied by \cite{ShufordAlbertMassengill1966} and later generalised by \cite{Schervish1989}. 
The recent presentation of this result by 
\cite{Lambert:2008} gives yet a more general formulation in terms of the 
elicitability of properties of distributions, along with a geometric derivation.
An historical summary of decompositions of scoring rules
is given by \citet[Section 4]{Winkler:1990}.

\begin{theorem}\label{thm:choquet}
A function $\loss : \mathcal{Y}\times[0,1]\to\RR$ is a proper scoring rule
that is not everywhere infinite and satisfies 
\begin{eqnarray}\label{schervish-condition}
	\loss(y,y) = \lim_{\heta \to y} \loss(y,\heta)
\end{eqnarray}
for $y\in {0,1}$
iff 
for each $\heta\in[0,1]$ and $y\in\mathcal{Y}$
\begin{equation}
	\loss(y,\heta) = \int_0^1 \loss_c(y,\heta)\,w(c)\,dc
	\label{eq:schervish-representation}
    \end{equation}
where 
\begin{equation}\label{eq:w-in-terms-of-L}
	w(c) = -\minL''(c)
\end{equation} and $\minL$ is the conditional Bayes risk for $\loss$.
\end{theorem}

The conditions on the scoring rule are required to avoid meaningless losses that assign infinite costs regardless of the estimate, and those rules which 
jump to infinity at the endpoints of $[0,1]$.
As is the case throughout this paper, the second derivative of $\minL$ is to
be interpreted distributionally. That is, $\minL''$ may be a generalised 
function such as the Dirac~$\delta$. The proof is technical and has been
presented by
\cite{Schervish1989} and \cite{Lambert:2008}. 

This is a powerful result that effectively identifies all the Fisher consistent
losses $\loss$ for probability estimation (and hence most surrogate margin 
losses) with a weight function $w$. 
This shift from ``losses as functions from estimates to costs'' to 
``losses as sums of primitive losses'' is (loosely!) analogous to the way 
the Fourier 
transform represents functions as sums of simple, periodic signals.

We will write $\ell_w$, $L_w$ and $\LL_w$ to explicitly indicate the
parametrisation of the loss, conditional loss and expected loss by the weight
function $w$.
We will also make use of the expression for $B_c$ derived by \cite{Buja:2005}:
\begin{lemma}\label{lem:int-rep-regret}
For any loss $c\in[0,1]$ the cost-weighted regret 
$B_c(\eta,\heta) := L_c(\eta,\heta) - \minL_c(\eta)$ can be expressed as
\begin{equation}
    B_c(\eta,\heta)=|\eta-c|\test{\eta\wedge\heta < c \le \eta\vee\heta}.
    \label{eq:B-c-formula}
\end{equation}
\end{lemma}

\begin{proof}
	From Theorem~\ref{pro:mincostrisk} we know that 
	$\minL_c(\eta) = \min\left\{ (1-\eta)c, (1-c)\eta \right\}$ and 
	note that $(1-\eta) c \le (1-c)\eta \iff c \le \eta$. Then, by
	the definition of $L_c$ and the identity $1-\test{p} = \test{\neg p}$
	we have
	\begin{eqnarray*}
		B_c(\eta,\heta)
		& = & (1-\eta)c\test{\heta \ge c} + (1-c)\eta\test{\heta < c} - 
			\min\left\{ (1-\eta)c, (1-c)\eta \right\} \\
		& = & 
			(1-\eta)c\test{\heta \ge c} + (1-c)\eta\test{\heta < c} - 
			(1-\eta)c\test{\eta \ge c} - (1-c)\eta\test{\eta < c} \\
		& = &
			(1-\eta)c (\test{\heta \ge c} - \test{\eta \ge c}) 
			+ (1-c)\eta (\test{\heta < c} - \test{\eta < c}) .
	\end{eqnarray*}
	Note that $\test{\heta \ge c} - \test{\eta \ge c}$ is either 1 or -1  
	depending on whether $\heta \ge c > \eta$ or $\heta < c \le \eta$ and
	is zero otherwise. Similarly, $\test{\heta < c} - \test{\eta < c}$
	is 1 when $\heta < c \le \eta$, is -1 when $\heta \ge c > \eta$ and is
	zero otherwise. This means
	\begin{eqnarray*}
		B_c(\eta,\heta) 
		& = &
		\begin{cases}
			(1-\eta)c - (1-c)\eta ,& \heta \ge c > \eta \\
			-(1-\eta)c + (1-c)\eta ,& \eta \ge c > \heta
		\end{cases} \\
		& = &
		\begin{cases}
			c - \eta ,& \heta \ge c > \eta \\
			\eta - c ,& \eta \ge c > \heta			
		\end{cases} \\
		& = &
		|\eta - c|\test{\min\{\eta,\heta\} \le c < \max\{\eta,\heta\}}
	\end{eqnarray*}  
	as required.
\end{proof}

\begin{theorem}
    Suppose $w\colon [0,1]\rightarrow\reals^+$ is a weight function and let
    \begin{eqnarray}
	W(t)&:=&\int^t w(x) dx\label{eq:W-def}\\
	\Wb(t)&:=&\int^t W(x) dx \label{eq:Wb-def}.
    \end{eqnarray}
    Then the regret of $\heta$ with respect to a true $\eta$ under the proper
    scoring rule induced by $w$ satisfies
    \begin{equation}
	B_w(\eta,\heta)= \Wb(\eta) -\Wb(\heta) -(\eta-\heta) W(\heta) .
	\label{eq:B-w-general-form}
    \end{equation}
    \label{theorem:B-w}
\end{theorem}
One can easily check that the arbitrary constants of integration in
(\ref{eq:W-def}) and (\ref{eq:Wb-def}) cancel out in (\ref{eq:B-w-general-form})
and thus do not matter.
\begin{proof}
    From (\ref{eq:schervish-representation}) and (\ref{eq:B-c-formula})
    we have 
    \begin{equation}
	B_w(\eta,\heta)= \int_{\eta\wedge\heta}^{\eta\vee\heta} |\eta-c| w(c)
	dc
	= \int_{\eta\wedge\heta}^\eta (\eta-c) w(c) dc + 
	\int_\eta^{\eta\vee\heta} (c-\eta) w(c) dc.\label{eq:B-w-split-integral}
    \end{equation}
    Now using integration by parts we have
    \[
	\int(c-\eta) w(c) dc = (c-\eta) W(c) - \int W(c) dc
	  = (c-\eta) W(c) -\Wb(c).
    \]
    Similarly
    \[
	 \int (\eta-c) w(c) dc = -(c-\eta) W(c) +\Wb(c).
    \]
    Thus from (\ref{eq:B-w-split-integral}) we have
    \begin{eqnarray}
	B_w(\eta,\heta)&=& 
	\left.\left[(c-\eta)W(c)-\Wb(c)\right]\right|_{\eta\wedge\heta}^\eta -
	\left.\left[(c-\eta)W(c)-\Wb(c)\right]\right|_\eta^{\eta\vee\heta}
	\nonumber\\
	&=& 2\Wb(\eta)
	-\Wb(\eta\wedge\heta)-\Wb(\eta\vee\heta)-
	(\eta-\eta\wedge\heta)W(\eta\wedge\heta)-
	(\eta-\eta\vee\heta)W(\eta\vee\heta).\nonumber
    \end{eqnarray}
    If $\eta\le\heta$, then $\eta\wedge\heta=\eta$ and $\eta\vee\heta=\heta$ and 
     we obtain
    \begin{eqnarray*}
	B_w(\eta,\heta)&= &
	2\Wb(\eta)-\Wb(\eta)-\Wb(\heta)-(\eta-\eta)W(\eta)-(\eta-\heta)W(\heta)\\
	&=& \Wb(\eta)-\Wb(\heta)-(\eta-\heta) W(\heta).
    \end{eqnarray*}
    If instead $\eta>\heta$, then $\eta\wedge\heta=\heta$ and
    $\eta\vee\heta=\eta$ and 
    we have
    \begin{eqnarray*}
	B_w(\eta,\heta)&=& \Wb(\eta)-\Wb(\heta)-\Wb(\eta)-(\eta-\heta)W(\heta)\\
	&=& \Wb(\eta)-\Wb(\heta)-(\eta-\heta) W(\heta).
    \end{eqnarray*}
    Thus in either case we obtain (\ref{eq:B-w-general-form}).
\end{proof}
Using (\ref{eq:B-w-general-form}) we can take a Taylor series expansion of 
$B_w(\eta,\heta)$ in $\heta$ about $\eta$ to obtain
\[
B_w(\eta,\heta)= \frac{1}{2} w(\eta) (\heta-\eta)^2 +\frac{1}{3}
    w'(\eta)(\heta-\eta)^3 + \frac{1}{8}w''(\eta)(\heta-\eta)^4 + \cdots
\]
This matches the second order result presented by \cite{Buja:2005}.

We consider three examples.
First, consider $w(c)=1$ for $c\in(0,1)$. Thus $W(c)=c$ and
$\Wb(c)=c^2/2$ and thus
\[
B_w(\eta,\heta)=\frac{\eta^2}{2}-\frac{\heta^2}{2}-(\eta-\heta)\heta =
\frac{(\eta-\heta)^2}{2}
\]
which is also apparent from the above Taylor series.
Second, consider  $w(c)=\frac{1}{c(1-c)}$. We have
$W(c)=\ln\left(\frac{c}{1-c}\right)$ and
$\Wb(c)=(1-c)\ln(1-c)+c\ln(c)$ and thus
\begin{eqnarray*}
    B_w(\eta,\heta) &=& (1-\eta)\ln(1-\eta)+\eta\ln \eta
	-(1-\heta)\ln(1-\heta)-
	\heta\ln\heta-(\eta-\heta)\ln\left(\frac{\heta}{1-\heta}\right)\\
    &=&(1-\eta)\ln\left(\frac{1-\eta}{1-\heta}\right)+
    \eta\ln\left(\frac{\eta}{\heta}\right)
\end{eqnarray*}
which agrees with the expression given by \cite{Buja:2005}. Finally consider
$w(c)=\delta(c-c_0)$, $c_0\in(0,1)$. We have $W(c)=U(c-c_0)$ and
$\Wb(c)=(c-c_0)_+$ and thus substituting into (\ref{eq:B-w-general-form}) we
obtain
\[
B_{c_0}(\eta,\heta)=(\eta-c_0)_+ -(\heta-c_0)_+ -(\eta-\heta)U(\heta-c_0),
\]
which agrees with (\ref{eq:B-c-formula}).
This can be written as
\begin{equation}
    B_{c_0}(\eta,\heta)= \left\{
    \begin{array}{ll}
	c_0-\eta & \ \ \mbox{if $\eta\le c_0$ and $\heta > c_0$}\\
	\eta-c_0 & \ \ \mbox{if $\eta>c_0$ and $\heta\le c_0$}\\
	0        & \ \ \mbox{otherwise.}
    \end{array}
    \right.
    \label{eq:B-c-0}
\end{equation}
In the special case that $c_0=\thalf$ we have
\begin{equation}
    B_{\frac{1}{2}}(\eta,\heta) = \left\{\begin{array}{ll} (\eta-\frac{1}{2})
	\sgn(\eta-\frac{1}{2}) & \ \ \mbox{if\ }
	\sgn(\eta-\frac{1}{2})\ne\sgn(\heta-\frac{1}{2})\\
	0 & \ \ \mbox{otherwise}.
    \end{array}\right.
    \label{eq:B-half}
\end{equation}

A similar approach allows a direct calculation of a general form for the
$w$-weighted conditional loss $L_w(\eta,\heta)$. 
\begin{theorem}
    Let $w$, $W$ and $\Wb$ be as in Theorem \ref{theorem:B-w}. Then for all
    $\eta,\heta\in[0,1]$,
    \begin{equation}
	L_w(\eta,\heta) =-\Wb(\heta)+W(\heta)(\heta-\eta)+\eta(\Wb(1)+
	\Wb(0))-\Wb(0).
	\label{eq:L-w-formula}
    \end{equation}
    Furthermore the conditional Bayes risk satisfies
    \begin{equation}
	\minL_w(\eta)=L_w(\eta,\eta)= -\Wb(\eta)+\eta(\Wb(1)+\Wb(0))-\Wb(0).
	\label{eq:L-bar-general-form}
    \end{equation}
    \label{theorem:L-w-general-form}
\end{theorem}
\begin{proof}
Starting from the expression given by 
\citet[Equation 17]{Buja:2005} and again integrating by parts we have
\begin{eqnarray}
    L_w(\eta,\heta)&=& \int_0^1
    \left[\eta(1-t)\test{t\ge\heta}+(1-\eta)t\test{t<\heta}\right] w(t)
    dt\nonumber\\
    &=& \eta\int_\heta^1 (1-t)w(t)dt + (1-\eta)\int_0^\heta t w(t)
    dt\nonumber\\
    &=& \eta \left.\left[(1-t) W(t)+\Wb(t)\right]\right|_\heta^1 +
        (1-\eta)\left.\left[tW(t)-\Wb(t)\right]\right|_0^\heta\nonumber\\
    &=&
    W(\heta)[(1-\eta)\heta-\eta(1-\heta)]-\Wb(\heta)+\eta\Wb(1)-
      (1-\eta)\Wb(0)\nonumber\\
    &=&-\Wb(\heta)+W(\heta)(\heta-\eta)+\eta(\Wb(1)+\Wb(0))-\Wb(0) .
    \label{eq:R-general-form}
\end{eqnarray}
Since $w$ is everywhere non-negative, $W$ and $\Wb$ are too 
(we deal with the constants of
integration shortly --- see e.g.~(\ref{eq:constants-of-integration})).
Consequently (\ref{eq:R-general-form}) is minimised by setting $\heta=\eta$
in which case we obtain (\ref{eq:L-bar-general-form}).
\end{proof}

The above theorem leads to a simple direct proof of
Theorem~\ref{pro:minrisk_concave}.

\begin{proof}{\bf{}(Theorem \ref{pro:minrisk_concave})}
    The concavity of $\minL_w(\eta)$ follows immediately from
    (\ref{eq:L-bar-general-form}) since $\Wb$ is convex, being the integral of a
    monotonically increasing function $W$, the integral of a non-negative
    function $w$.
    From (\ref{eq:L-bar-general-form}) we have
    \[
    \minL_w'(\heta)=-W(\heta)+\Wb(1)+\Wb(0).
    \]
    Thus
    \begin{eqnarray*}
	& & \minL_w(\heta)-(\heta-\eta) \minL_w'(\heta)\\
	&=& -\Wb(\heta)+\heta(\Wb(1)+\Wb(0))-\Wb(0) -
	(\heta-\eta)[-W(\heta) +\Wb(1)+\Wb(0)]\\
	&=& -\Wb(\heta) +W(\heta)(\heta-\eta)+\eta(\Wb(1)+\Wb(0))-\Wb(0)\\
	&=& L_w(\eta,\heta)  
    \end{eqnarray*}
    where the last step follows from  (\ref{eq:L-w-formula}).
    This proves (\ref{eq:savage-rep}).
\end{proof}

%%%%%%%%%%%%%%%%%%%%%%%%%%
\subsubsection{Convexity, Matching Losses and Canonical Links}
Recall from Section~\ref{sub:lfdual} that the Legendre-Fenchel dual of $f$ can
be expressed in terms of its derivative and inverse. 
Furthermore in this case (writing $Df:=f'$) $f'=(D\lfdual{f})^{-1}$. Thus with
$w$, $W$, and $\Wb$ defined as above,
\begin{equation}
    W=(D(\lfdual{\Wb}))^{-1},\ \ \ W^{-1}=D(\lfdual{\Wb}), \ \
    \lfdual{\Wb}=\int W^{-1} .
    \label{eq:W-star-relationships}
\end{equation}
We now further consider $B_w$ as given by (\ref{eq:B-w-general-form}). It will
be convenient to parametrise $B$ by $W$ instead of $w$. Note that the standard
parametrisation for a Bregman divergence is in terms of the convex function
$\Wb$. Thus will write $B_{\Wb}$, $B_W$ and $B_w$  to all represent
(\ref{eq:B-w-general-form}).
The following theorem is well known (e.g \cite{Zhang:2004}) 
but as will be seen, stating it
in terms of $B_W$ provides some advantages.
\begin{theorem}
    Let $w$, $W$, $\Wb$ and $\Br_W$ be as above. Then for all
    $x,y\in[0,1]$,
    \begin{equation}
	\Br_W(x,y) = \Br_{W^{-1}}(W(y),W(x)).
	\label{eq:B-W-dual}
    \end{equation}
\end{theorem}
\begin{proof}
    Using (\ref{eq:legendre-transform}) we have
    \begin{eqnarray}
    & & 	\lfdual{\Wb}(u) = u\cdot W^{-1}(u) -\Wb(W^{-1}(u)) \nonumber\\
    &\Rightarrow& \Wb(W^{-1}(u)) = u\cdot W^{-1}(u) - \lfdual{\Wb}(u).
    \label{eq:W-bar-W-inverse}
    \end{eqnarray}
    Equivalently (using (\ref{eq:W-star-relationships}))
    \begin{equation}
	\lfdual{\Wb}(W(u)) = u\cdot W(u) -\Wb(u).
	\label{eq:W-bar-star-W}
    \end{equation}
    Thus substituting and then using (\ref{eq:W-bar-W-inverse}) we have
    \begin{eqnarray}
	\Br_W(x,W^{-1}(v)) &=& \Wb(x) - \Wb(W^{-1}(v)) - (x-W^{-1}(v))\cdot
	    W(W^{-1}(v))\nonumber\\
	&=& \Wb(x)+\lfdual{\Wb}(v) -v W^{-1}(v) - (x-W^{-1}(v))\cdot v
	\nonumber\\
	&=& \Wb(x) +\lfdual{\Wb}(v) - x\cdot v. \label{eq:B-W-int-res-1}
    \end{eqnarray}
    Similarly (this time using (\ref{eq:W-bar-star-W}) we have
    \begin{eqnarray}
	 \Br_{W^{-1}}(v,W(x))&=&\lfdual{\Wb}(v)-\lfdual{\Wb}(W(x))
	 -(v-W(x))\cdot W^{-1}(W(x))\nonumber\\
	 &=&\lfdual{\Wb}(v)-x W(x)+\Wb(x)-v\cdot x +x W(x)\nonumber\\
	 &=& \lfdual{\Wb}(v)+\Wb(x) - v\cdot x\label{eq:B-W-int-res-2}
    \end{eqnarray}
    Comparing (\ref{eq:B-W-int-res-1}) and (\ref{eq:B-W-int-res-2}) we  see
    that
    \[
       \Br_W(x,W^{-1}(v)) = \Br_{W^{-1}}(v,W(x)) 
    \]
    Let $y=W^{-1}(v)$. Thus subsitituting $v=W(y)$ leads to
    (\ref{eq:B-W-dual}).
\end{proof}
The weight function corresponding to $\Br_{W^{-1}}$ is 
$
\frac{\partial}{\partial x} W^{-1}(x) = \frac{1}{w(W^{-1}(x))} .
$

Often in estimating $\eta$ one uses a parametric representation of
$\heta\colon\Xcal\rightarrow$[0,1]
which has a natural scale not matching $[0,1]$. In such cases it is common to
use a \emph{link function} \citep{McCullaghNelder1989, KivinenWarmuth2001,
HelmboldKivinenWarmuth1999}.
Traditionally one writes
$
\heta = \psi^{-1}(\hh)
$
where $\psi^{-1}$ is the ``inverse link'' (and $\psi$ is of course the forward
link). The function $\hh\colon\Xcal\rightarrow\reals$ is the 
{\em hypothesis}. Often $\hh=\hh_{\bm\alpha}$
is parametrised linearly in a parameter vector $\bm{\alpha}$. In such a 
situation it is
computationally convenient if $\Lr_W(\eta,\psi^{-1}(\hh))$ is convex in $\hh$
(which implies it is convex in $\bm{\alpha}$ when $\hh$ is linear in 
$\bm{\alpha}$). The
following result provides a simple sufficient condition for the ``composite
loss'' $\Lr_W(\eta,\psi^{-1}(\hh))$ to be convex in $\hh$. It was previously
shown (with a more intricate proof) by \cite{Buja:2005}. The result also
corresponds to the notion of ``matching loss'' as developed by 
\citet{HelmboldKivinenWarmuth1999} and
\citet{KivinenWarmuth2001}.
\begin{theorem}
    Let $w$, $W$, $\Wb$ and $\Br_W$ be as above. Denote by $\Lr_W$ the
    $w$-weighted conditional loss parametrised by $W=\int w$.
    If the inverse link $\psi^{-1}=W^{-1}$ (and thus $\heta=W^{-1}(\hh)$) then
    \begin{eqnarray*}
	\Br_W(\eta,\heta) &=& \Br_W(\eta,W^{-1}(\hh)) =
	\Wb(\eta)+\lfdual{\Wb}(\hh) -\eta\cdot\hh\\
	\Lr_W(\eta,\heta) &=& \Lr_W(\eta,W^{-1}(\hh)) = \lfdual{\Wb}(\hh) -
	\eta\cdot\hh + \eta(\Wb(1)+\Wb(0))-\Wb(0)\\
	\frac{\partial}{\partial\hh} \Lr_W(\eta,W^{-1}(\hh)) & =& \heta-\eta
    \end{eqnarray*}
    and furthermore $\Br_W(\eta,W^{-1}(\hh))$ and $\Lr_W(\eta,W^{-1}(\hh))$ are
    convex in $\hh$.
\end{theorem}
\begin{proof}
    The first two expressions follow immediately from (\ref{eq:B-W-int-res-1})
    and (\ref{eq:B-W-int-res-2}) by substitution. The derivative follows from
    calculation: $\frac{\partial}{\partial\hh} \Lr_W(\eta,W^{-1}(\hh)) 
    = \frac{\partial}{\partial\hh}(\lfdual{\Wb}(\hh)-\eta\cdot\hh) 
    = W^{-1}(\hh)-\eta=\heta-\eta$.  The convexity follows from the fact that 
    $\lfdual{\Wb}$ is
    convex (since it is the LF dual of a convex function $\Wb$) and the overall
    expression is the sum of this and a linear term, and thus convex.
\end{proof}
\cite{Buja:2005} call $W$ the \emph{canonical link}.  

Importantly, the linearity of expectation means that the same weight function
can be used to write a loss's risk and statistical information as a weighted
integral of the primitives $\LL_c$ and $\SI_c$, respectively. When combined
with Theorem~\ref{thm:duality}, these results give a similar 
weighted integral representation for  Bregman divergences. 

%%%%%%%%%%%%%%%%%%%%%%%%%%%%%%%%%%%%%%%%%%%%%%%%%%%%%%%%%%%%%%%%%%%
\subsection{Relating Integral Representations for $\LL$ and $\II_f$}
\label{sub:relating-SI-and-fdiv}
We 
can also give a translation between the weight functions $\gamma$ for an 
$f$-divergence and $w$ for the corresponding statistical information. 

\begin{theorem}\label{pro:wgamma}
Let $f$ be convex (with $f(1)=0$) define $\II_f$ and the weight function 
$\gamma$. Then
for each $\pi\in(0,1)$ the weight function $w_\pi$ in Theorem 
\ref{thm:choquet} for the loss
$\loss^\pi$
given by Theorem~\ref{thm:duality} satisfies
\[
	w_\pi(c) 
	= \frac{\pi(1-\pi)}{\nu(\pi,c)^3}
		\gamma\left(\frac{(1-c)\pi}{\nu(\pi,c)}\right)
\] 
or, inversely,
\[
	\gamma(c) = \frac{\pi^2(1-\pi)^2}{\nu(\pi,c)^3}
		w\left(\frac{\pi(1-c)}{\nu(\pi,c)}\right)
\]
where $\nu(\pi,c) = (1-c)\pi + (1-\pi)c$.
\end{theorem}

\begin{proof}
Theorem~\ref{thm:duality} shows that 
\begin{equation}\label{eq:propwgamma1}
	\minL^\pi(\eta) 
	= -\frac{1-\eta}{1-\pi}f\left(\frac{1-\pi}{\pi}\frac{\eta}{1-\eta}\right).
\end{equation}
and we have seen from (\ref{eq:w-in-terms-of-L}) that
$w^\pi(c) = -(\minL^\pi)''(c)$. 
The remainder of this proof involves taking the second derivative of $\minL$, 
doing some messy algebra and matching the result to the relationship
between $\gamma$ and $f''$ in (equation~\ref{eq:gamma-f-relationship}).

Letting $r_\pi = \frac{1-\pi}{\pi}\frac{\eta}{1-\eta}$ and taking derivatives of 
(\ref{eq:propwgamma1}) yields
\begin{eqnarray*}
	-(\minL^\pi)'(\eta) 
	& = & (1-\pi)^{-1}[-f(r_\pi) + (1-\eta)f'(r_\pi)r_{\pi}'] \\
	-(\minL^\pi)''(\eta)
	& = & (1-\pi)^{-1}[-f'(r_\pi) r_{\pi}' 
			+ (1-\eta)(f'(r_\pi)r_{\pi}'' + f''(r_\pi)(r_{\pi}')^2)
			- f'(r_\pi)r_{\pi}'] \\
	& = & (1-\pi)^{-1}[(-2r_{\pi}' + (1-\eta)r_{\pi}'')f'(r_\pi)
			+ (1-\eta)(r_{\pi}'')^2f''(r_\pi) ].
\end{eqnarray*}
However, the form of $r_\pi$ means 
$r_{\pi}' = \frac{1-\pi}{\pi}\frac{1}{(1-\eta)^2}$
and so
$r_{\pi}'' = \frac{1-\pi}{\pi}\frac{2}{(1-\eta)^3}$. This means the coefficient
of $f'(r_\pi)$ in the above expression vanishes
\[
	(-2r_{\pi}' + (1-\eta)r_{\pi}'') 
	= \frac{1-\pi}{\pi}
		\left[
			\frac{-2}{(1-\eta)^2} + (1-\eta)\frac{2}{(1-\eta)^3}
		\right]
	= 0.
\]
Substituting this back into $-(\minL)''$ gives us
\begin{eqnarray*}
	-(\minL^\pi)''(\eta)
	& = & \frac{1-\eta}{1-\pi} f''(r_\pi)(r_{\pi}')^2 \\
	& = & \frac{1-\eta}{1-\pi} 
		f''\left(
			\frac{1-\pi}{\pi}\frac{\eta}{1-\eta}
		\right)
		\frac{(1-\pi)^2}{\pi^2}\frac{1}{(1-\eta)^4} \\
	w(\eta) & = & \frac{1-\pi}{\pi^2(1-\eta)^3}
		f''\left(
			\frac{1-\pi}{\pi}\frac{\eta}{1-\eta}
		\right) .
\end{eqnarray*}
By equation~\ref{eq:gamma-f-relationship} we have 
\begin{equation}\label{eq:propwgamma2}
	\gamma(t) = \frac{1}{t^3}f''\left(\frac{1-t}{t}\right).
\end{equation}
Letting $t = \frac{(1-c)\pi}{(1-c)\pi + (1-\pi)c}$ in that expression gives
\[
	\gamma\left(\frac{(1-c)\pi}{\nu(\pi,c)}\right) 
	= \frac{\nu(\pi,c)^3}{(1-c)^3\pi^3}
		f''\left(
			\frac{1-\pi}{\pi}\frac{c}{1-c}
		\right).
\]
Thus
\[
\frac{\pi(1-\pi)}{\nu(\pi,c)^3}\gamma\left(\frac{(1-c)\pi}{\nu(\pi,c)}\right) 
	= \frac{1-\pi}{\pi^2(1-c)^3}f''\left(\frac{1-\pi}{\pi}\frac{c}{1-c}\right)
	= w(c)
\]
as required.
The argument to show the inverse relationship is essentially the same.
\end{proof}

\begin{sidewaystable}
    \begin{center}
    \arrayrulecolor{tabgrey}
    \renewcommand{\arraystretch}{1.3}
    \begin{tabular*}{18cm}{lcl}
	\hline
	 Symbol & $\gamma(\pi)$ & $f(t)$ \hfill (Divergence Name)\hfill\\
	\hline\hline
	$\rule{0pt}{2.4ex} V(P,Q)$ & $4\delta\left(\pi-\frac{1}{2}\right)$ 
	   & $|t-1|$\hfill (Variational Divergence)\\[0.2mm]
	\hline
	$\rule{0pt}{2.4ex}\Delta(P,Q)$ &  8 & $(t-1)^2/(t+1)$\hfill  (Triangular 
	     Discrimination) \\
	\hline
	$\rule{0pt}{2.4ex}\mathrm{I}(P,Q)=\frac{1}{2}[\mathrm{KL}(P,\frac{P+Q}{2})+$ &
	     $\frac{1}{2\pi(1-\pi)}$ & $\frac{t}{2}\ln t - 
	     \frac{(t+1)}{2}\ln(t+1) +\ln 2$ \\
	     \hspace*{2cm}$\mathrm{KL}(Q,\frac{P+Q}{2})]$ & 
	       &\hfill (Jensen-Shannon Divergence)\\[0.2mm]
	\hline
	$\rule{0pt}{3ex} \mathrm{T}(P,Q)=\frac{1}{2}[\mathrm{KL}(\frac{P+Q}{2},P) +$
	&$\frac{(2\pi-\frac{1}{2})^{2}+\frac{1}{2}}{4\pi^2(\pi-1)^2}$ &
	$\left(\frac{t+1}{2}\right)\ln\left(\frac{t+1}{2\sqrt{t}}\right)$ \\
	\hspace*{2cm} $\mathrm{KL}(\frac{P+Q}{2},Q)]$ &
	  &\hfill(Arithmetic-Geometric Mean Divergence) \\[0.2mm]
	\hline
	$\rule{0pt}{2.4ex}\mathrm{U}(P,Q)$ \hfill(``Uninformative prior'')
	& $\frac{1}{\pi\wedge(1-\pi)}$ &
	     $\test{t\le 1}(t\ln(t/(t+1)) + \ln(2))$\\
	& & $+\test{t> 1}((t+1)\ln(t+1)-(\frac{1}{2}+2\ln 2)t+\frac{1}{2})$
	\\[0.2mm]
	     \hline
	     $\rule{0pt}{2.4ex}\mathrm{J}(P,Q)=\mathrm{KL}(P,Q)+\mathrm{KL}(Q,P)$ &$\frac{1}{\pi^2 (1-\pi)^2}$ & $(t-1)\ln(t)$
	         \hfill(Jeffreys Divergence) \\[0.2mm]
	     \hline
	     $\rule{0pt}{2.4ex}h^2(P,Q)$ & $\frac{1}{2[\pi(1-\pi)]^{3/2}}$ & $(\sqrt{t}-1)^2$
	     \hfill(Hellinger Divergence)\\[0.2mm]
	     \hline
	     $\rule{0pt}{3ex} \chi^2(P,Q)$ & $\frac{2}{\pi^3}$ & $(t-1)^2$
	     \hfill (Pearson $\chi$-squared)\\[0.2mm]
	     \hline
	     $\rule{0pt}{3ex}\Psi(P,Q)=\chi^2(P,Q)+\chi^2(Q,P) $ &
	     $\frac{2}{\pi^3}+\frac{2}{(1-\pi)^3}$ &
	     $\frac{(t-1)^2(t+1)}{t}$ \hfill(Symmetric $\chi$-squared)\\[0.2mm]
	     \hline
	     $\mathrm{KL}(P,Q) $ & $ \frac{1}{\pi^2 (1-\pi)}$ & $t\ln t$
	     \hfill (Kullback-Liebler Divergence)\\[0.2mm]
	     \hline
     $\mathrm{KL}_\epsilon(P,Q) $ & $\frac{1}{\pi^2 (1-\pi)}\indicator_{[\epsilon,1-\epsilon]}(\pi)$
	 & $\test{t<\frac{\epsilon}{1-\epsilon}}(t(\ln(
	    \frac{\epsilon}{1-\epsilon})+1)-\frac{\epsilon}{1-\epsilon})$\
	    \hfill
	    (Truncated KL)\\
	 & $\epsilon\in[0,1)$& $+\test{\frac{\epsilon}{1-\epsilon}\le t \le\frac{1-\epsilon}{\epsilon}} 
	     t\ln t$\\
	     & & $+\test{\frac{1-\epsilon}{\epsilon} <t}(t(\ln(
	     \frac{1-\epsilon}{\epsilon})+1)-\frac{1-\epsilon}{\epsilon})$\\
	     \hline
    \end{tabular*}
   
    \hspace{1cm}\parbox{16cm}{\caption{
     $f$-divergences and their corresponding weights;
    confer~\cite{Taneja2005,LieVaj06}. 
    \cite{Topsoe:2000} calls $C(P,Q)=2 \mathrm{I}(P,Q)$ and
    $\tilde{C}(P,Q)=2\mathrm{T}(P,Q)$ the Capacitory and Dual Capacitory
    discrimination respectively. 
    \label{table:symmetric-divergences}}}
\end{center}
\end{sidewaystable}

The representation (\ref{eq:LV-integral-rep},\ref{eq:gamma-f-relationship})
allows the determination of  weights for common $f$-divergences.
$\mathrm{KL}(P,Q)$ corresponds
to $\gamma(\pi)=\frac{1}{\pi^2(1-\pi)}$. Thus
$J(P,Q)=\mathrm{KL}(P,Q)+\mathrm{KL}(Q,P)$ corresponds to $\gamma(\pi)=\frac{1}{\pi^2(1-\pi)^2}$.
Several $f$-divergences are presented with their corresponding weight
function
in Table~\ref{table:symmetric-divergences}. 
The weight for $\mathrm{KL}(P,Q)$ has a double pole at $\pi=0$ which is why
KL-divergence is hard to estimate\footnote{
Considering KL-divergence from the weight
function perspective immediately suggests a scheme to estimate it: avoid
attempting to estimate the regions near zero and one where the weight function
diverges. A particular example of this is the divergence we have called
$\mathrm{KL}_\epsilon$ in 
Table~\ref{table:symmetric-divergences}.  This approach to regularizing the
KL-divergence was suggested by
\citet[page 454]{Gutenbrunner:1990}.}.

%%%%%%%%%%%%%%%%%%%%%%%%%%%%%%%%%%%%%%%%%
\subsection{Example --- Squared Loss}
\label{section:example-squared-loss}

We illustrate some of the above concepts with  a simple example. Consider
squared loss. We have
\[
    L(\eta,\heta)=\heta^2(1-\eta) +(\heta-1)^2 \eta 
\]
and thus $\minL(\eta)=L(\eta,\eta)=\eta(1-\eta)$ and $\minL''(\eta)=-2$ and
thus by (\ref{eq:w-in-terms-of-L})  $w(\eta)=2$. From (\ref{eq:OV-L-to-f})
we thus have
\[
f^\pi(t)=\frac{\pi(1-\pi)(\pi t +1 -\pi) -(1-\pi)\pi t}{\pi t+1-\pi} .
\]
Choosing $\pi=\thalf$ this becomes $f^{\frac{1}{2}}(t)=\frac{1-t}{4t+4}$. One
can check that $8\cdot f^{\frac{1}{2}}(t)+t-1=\frac{(t-1)^2}{t+1}$ which agrees
with the $f$ corresponding to Capacitory Discrimination in
Table~\ref{table:symmetric-divergences}. Scaling is just a question of
normalisation and we have already seen that $\II_f$ is insensitive to affine
offsets in $f$.  This illustrates the awkwardness of parameterising $\II_f$ in
terms of $f$: at first sight  $\frac{1-t}{4t+4}$ and $\frac{(t-1)^2}{t+1}$ seem
different. Using weight functions automatically filters out the effect of 
any affine
offsets --- if the weight functions corresponding to $f_1$ and $f_2$ match,
then $\II_{f_1}=\II_{f_2}$. Finally observe that substituting 
$\gamma(\pi)=8$ from the table into
Theorem~\ref{pro:wgamma} we obtain $w_{\frac{1}{2}}(c)=\frac{1/4}{\nu(\pi,c)^3}
\cdot 8 = 2$ consistent  with the weight obtained above.

%%%%%%%%%%%%%%%%%%%%%%%%%%%%%%%%%%%%%%%%%%%%%%%%%%%%%%%%%%%%%%%%%%%%%%
\section{Graphical Representations}
The last section described representations of risks and $f$-divergences in 
terms of weighted integrals of primitive functions. The weight functions and
values of the primitive functions lend themselves to a graphical interpretation
that is explored in this section. In particular, a diagram called a 
\emph{risk curve} is introduced. This is shown to be closely related to 
the \emph{cost curves} of \cite{DrummondHolte2006} as well as an idealised
\emph{receiver operating characteristic, or ROC curve} 
\citep{Fawcett:2004}.
Risk curves are useful aids to intuition when reasoning about 
risks, divergences and information and they are used extensively in 
Section~\ref{sec:inequalities} to derive bounds between various divergences
and risks.

\subsection{ROC Curves}\label{sub:ROC}
Plotting a \emph{receiver operating curve} or \emph{ROC curve} is a way of 
graphically summarising the performance of a test statistic. Recall from
Section~\ref{sub:NPLemma} that in the context of a binary experiment $(P,Q)$
on a space $\Xcal$, a test statistic $\tau$ is any function that maps points in 
$\Xcal$ to the real line. Each choice of threshold $\tau_0\in\RR$ results in
a classifier $r(x) = \test{\tau \ge \tau_0}$ and its corresponding 
classification rates. 
An ROC curve for the test statistic $\tau$ is simply a plot of the true positive
rate of these classifiers as a function of their false positive rate as 
the threshold $\tau_0$ varies over $\RR$.
Formally, 
\[
	\ROC(\tau) 
	:= \{ (FP_{\tau}(\tau_0), TP_{\tau}(\tau_0)) : \tau_0 \in \RR \}
	\subset [0,1]^2.
\]
A graphical example of an ROC curve is shown as the solid black line in 
Figure~\ref{fig:roc}.

\begin{figure}[t]
    \begin{center}
	\includegraphics[width=0.5\textwidth]{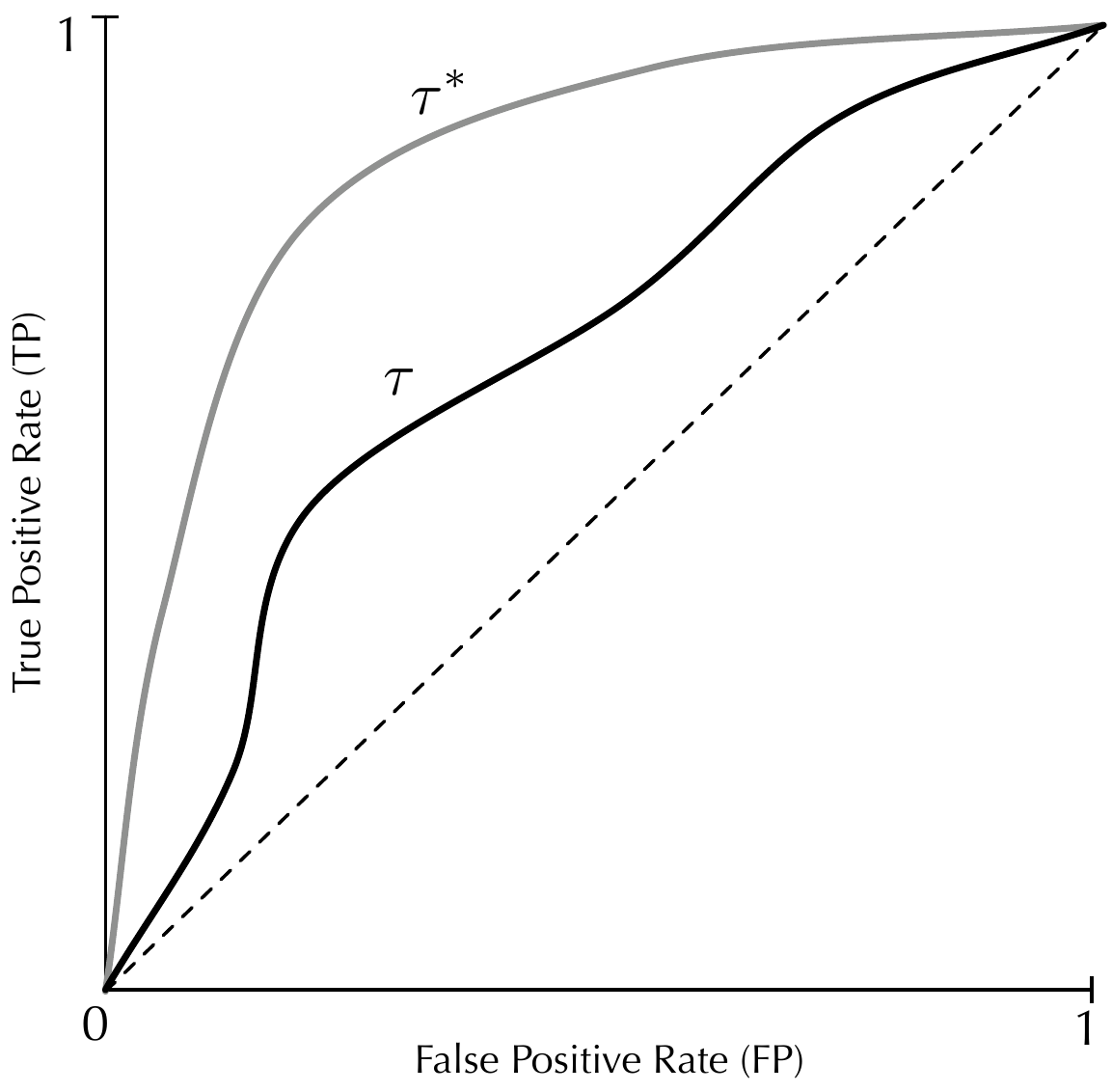}
    \end{center}
\caption{Example of an ROC diagram showing an ROC curve for an arbitrary 
statistical test $\tau$ (middle, bold curve) as well as an optimal statistical 
test $\tau^*$ (top, grey curve). The dashed line represents the ROC curve for a 
random, or uninformative statistical test.\label{fig:roc}}
\end{figure}

For a fixed experiment $(P,Q)$, the Neyman-Pearson lemma provides an upper 
envelope for ROC curves. It guarantees that the ROC curve for the likelihood
ratio $\tau^* = dP/dQ$ will lie above, or \emph{dominate}, that of any other 
test statistic $\tau$ as shown in Figure~\ref{fig:roc}. 
This is an immediate consequence of the 
likelihood ratio being the uniformly most powerful test since for each
false positive rate (or size) $\alpha$ it will have the largest true positive 
rate (or power) $\beta$ of all tests \citep{EguchiCopas2001}.

The performance of a test statistic $\tau$ shown in an ROC curve is commonly 
summarised by the \emph{Area Under the ROC Curve}, $\AUC(\tau)$, 
and is closely related to the Mann-Whitney-Wilcoxon statistic.
Formally, if $(P,Q)$ is a binary experiment and $\tau$ a test statistic
the AUC is
\begin{eqnarray}
	\AUC(\tau) &:=& \int_0^1 \beta_\tau(\alpha)\,d\alpha \label{eq:AUC}\\
			  &=& \int_{-\infty}^\infty
			 		TP_{\tau}(\tau_0)\,FP'_\tau(\tau_0)
				  \,d\tau_0, \label{eq:AUC-TP-FP}
\end{eqnarray}
where $\beta_\tau(\alpha) = TP_{\tau}(\tau_0)$ for a $\tau_0\in\RR$ such that
$FP_\tau(\tau_0) = \alpha$.

In Section~\ref{sub:NPLemma} the Neyman-Pearson lemma was used to argue that the 
curve $\beta(\alpha)$ for the likelihood ratio dominates all other curves. 
As the likelihood ratio is used to define $f$-divergences, it is natural to
ask whether the area under the maximal ROC curve is an $f$-divergence. 
That is, does there exist a convex $f$ such that $\II_f(P,Q) = \AUC(dP/dQ)$?
Interestingly, the answer is ``no''. 
To see this, note that an $f$-divergence's integral can be decomposed as follows
\begin{equation}\label{eq:If-for-AUC}
	\II_f(P,Q) = \int_0^\infty f(t) \int_{\Xcal_t} dQ\,dt
\end{equation}
where $\Xcal_t := \{ x \in \Xcal : dP/dP(x) = t \} = (dP/dQ)^{-1}(t)$. 
Compare this to the definition of AUC given in (\ref{eq:AUC-TP-FP}) when 
$\tau = dP/dQ$
\begin{eqnarray}
	\AUC(dP/dQ) 
	& = & \int_{-\infty}^\infty  TP_{\tau}(t)\,FP'_\tau(t) \,dt \nonumber \\ 
	& = & -\int_0^{\infty} (P\circ\tau^{-1})([t,\infty)) \label{eq:AUC-int}
			\int_{\Xcal_t} dQ\,dt
\end{eqnarray}
since 
\(
	FP'_\tau(t) 
	= d/dt \int_t^{\infty} \int_{\Xcal_t} dQ(x)\,dt
	= - \int_{\Xcal_t} dQ
\) 
and $dP/dQ \ge 0$.
If we assume there exists an $f$ such that for all binary experiments $(P,Q)$ 
that $\II_f(P,Q) = \AUC(dP/dQ)$ we would require the integrals in 
(\ref{eq:If-for-AUC}) and (\ref{eq:AUC-int}) to be equal for all $(P,Q)$. 
This would require $f(t) = -(P\circ(dP/dQ)^{-1})([t,\infty))$ for all 
$t\in[0,\infty)$ which is not possible for all binary experiments $(P,Q)$
simultaneously.

Interestingly, even though maximal AUC for $(P,Q)$ cannot be expressed as an 
$f$-divergence, \cite{torgersen1991cse} shows how it can be expressed as
the variational divergence between the \emph{product measures} 
$P\times Q$ and $Q \times P$. That is, 
$\AUC(dP/dQ) = V(dP\times dQ,dQ\times dP)$. Following up this connection and
considering other $f$-divergences of product measures is left as future work.

It is important to realise that AUC is not a particularly intrinsic measure --- 
just a common one. As the earlier discussion of integral representations
have shown, there is value in considering weighted versions of integrals such
as (\ref{eq:AUC}). As \cite{Hand2008} notes in his commentary on a recent
paper (outlining another type of performance curve):
``To use all the values of the diagnostic instrument,
when integrating to yield the overall AUC measure,
it is necessary to decide what weight to give to each value in
the integration. The AUC implicitly does this using a weighting
derived empirically from the data.''
Along these lines, \cite{xie2002wgm} and \citet{EguchiCopas2001} have suggested
generalisations of the AUC that incorporates weights and show that certain
choice of weight functions yield well-known losses. 

A closer investigation of these generalisations of AUC and their connection to
measures of divergence is also left as future work.

\subsection{Risk Curves}

Risk curves are graphical representation closely related to 
ROC curves that take into account a prior $\pi$ in addition to the binary
experiment $(P,Q)$. 
They provide a concise summary of the 
risk of an estimator $\hat{\eta}$ for the full range of costs $c\in[0,1]$ for
a fixed prior $\pi\in[0,1]$, or, alternatively, for the full range of priors 
$\pi$ given a fixed cost $c$.

Formally, \emph{a risk curve for costs} for the estimator $\heta$ is the set 
$\{ (c, \LL_c(\hat{\eta}, \pi, P, Q)) : c\in[0,1]\}$  
of points parametrised by cost\footnote{
	Unlike the cost curves originally described by \cite{DrummondHolte2006}, 
	the version presented here does not normalise the risk and plots the cost 
	on the horizontal axis rather than the product of the prior probability and
	cost.
}.
A \emph{risk curve for priors} for the estimator $\heta$ is the set 
$\{ (\pi, \LL^{\textrm{0-1}}(\heta, \pi,P,Q)) : \pi\in [0,1] \}$.

\begin{figure}[t]
    \begin{center}
	\includegraphics[width=0.5\textwidth]{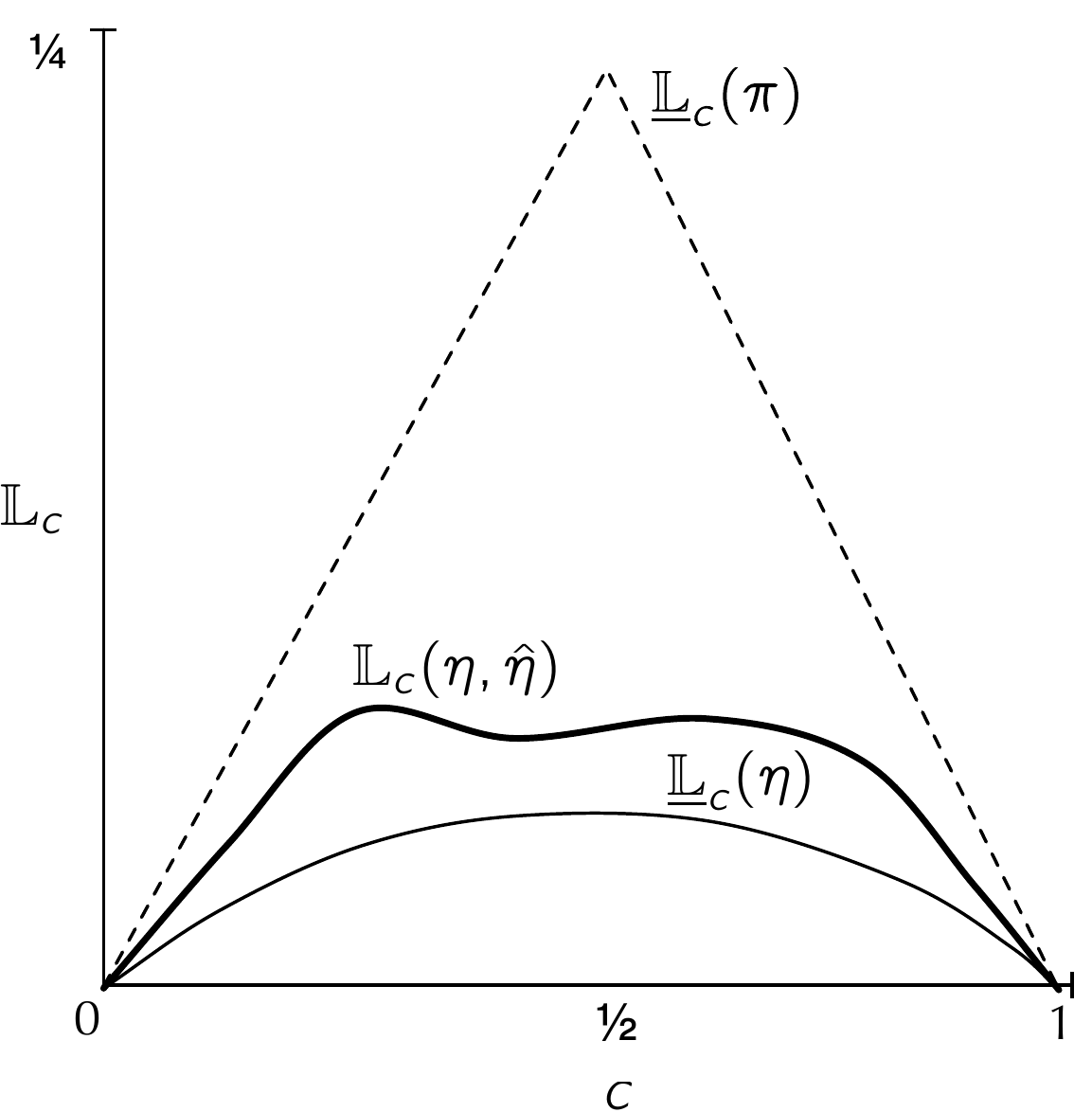}
    \end{center}
\caption{Example of a risk curve diagram showing risk curves for costs for the 
true posterior probability $\eta$ (bottom, solid curve), an estimate 
$\hat{\eta}$ (middle, bold curve) and the majority class or prior estimate (top, 
dashed curve).\label{fig:costcurve}}
\end{figure}

Figure~\ref{fig:costcurve} shows an example of a \emph{risk curve diagram}. 
On it is plotted the cost curves for an estimate 
$\hat{\eta}$ of a true posterior $\eta$ on the same graph. The ``tent'' function
also shown is the risk curve for the majority class predictor 
$\min((1-\pi)c, (1-c)\pi)$. Here $\pi = \thalf$. Other choices of $\pi\in(0,1)$
skew the tent and the curves under it towards 0 or 1.  

In light of the weighted integral representations described in 
Theorem~\ref{thm:choquet}, several of the quantities can be associated with 
properties of a cost curve diagram. The weight function $w(c)$ associated with
a loss $\ell$ can be interpreted as a weighting on the horizontal axis of a 
risk curve diagram. When the area under a risk curve is computed with respect
to this weighting the result is the full risk $\LL$ since 
$\LL(\eta,\hat{\eta}) = \int_0^1 \LL_c(\eta,\hat{\eta})\,w(c)\,dc$.

Furthermore, the weighted area between the risk curves for an estimate 
$\hat{\eta}$ and the true posterior $\eta$ is the regret 
$\LL(\eta,\hat{\eta}) - \minLL(\eta)$ and the statistical information 
$\SI(\eta,M) = \minLL(\pi,M) - \minLL(\eta,M)$ is the weighted area between 
the ``tent'' risk curve for $\pi$ and the risk curve for $\eta$.

The correspondence between ROC and risks curves is due to the relationship 
between the true class probability $\eta$ and the likelihood ratio $dP/dQ$
for a fixed $\pi$. As shown in Section~\ref{sub:generative-discriminative},
this relationship is
\[
	\frac{dP}{dQ}(x) = \lambda_\pi(\eta) 
	                 = \frac{1-\pi}{\pi}\frac{\eta}{1-\eta}.
\]
Each cost $c\in[0,1]$ can be mapped to a corresponding test statistic 
threshold $\tau_0 = \lambda_\pi(c)$ and \emph{vice versa}. 

\cite{DrummondHolte2006} show that their cost curves have a point-line dual 
relationship with ROC curves. The same result holds for our risk diagrams.
Specifically, for a given point $(FP, TP)$ on an ROC diagram the 
corresponding line in a risk diagram is 
\[
	\LL_c = (1-\pi)\,c\,FP + \pi\,(1-c)\,(1-TP).
\] 
Conversely, the line in ROC space corresponding to a point $(c, \LL_c)$ in risk 
space is
\[
	TP = \frac{(1-\pi)c}{\pi(1-c)}FP + \frac{(1-\pi)c - L}{\pi(1-c)}.
\]
An example of this relationship is shown graphically in 
Figure~\ref{fig:variational_roc_bound} between the point A and the line A*. 

 \begin{figure}[t]
     \begin{center}
 	\includegraphics[width=1.0\textwidth]{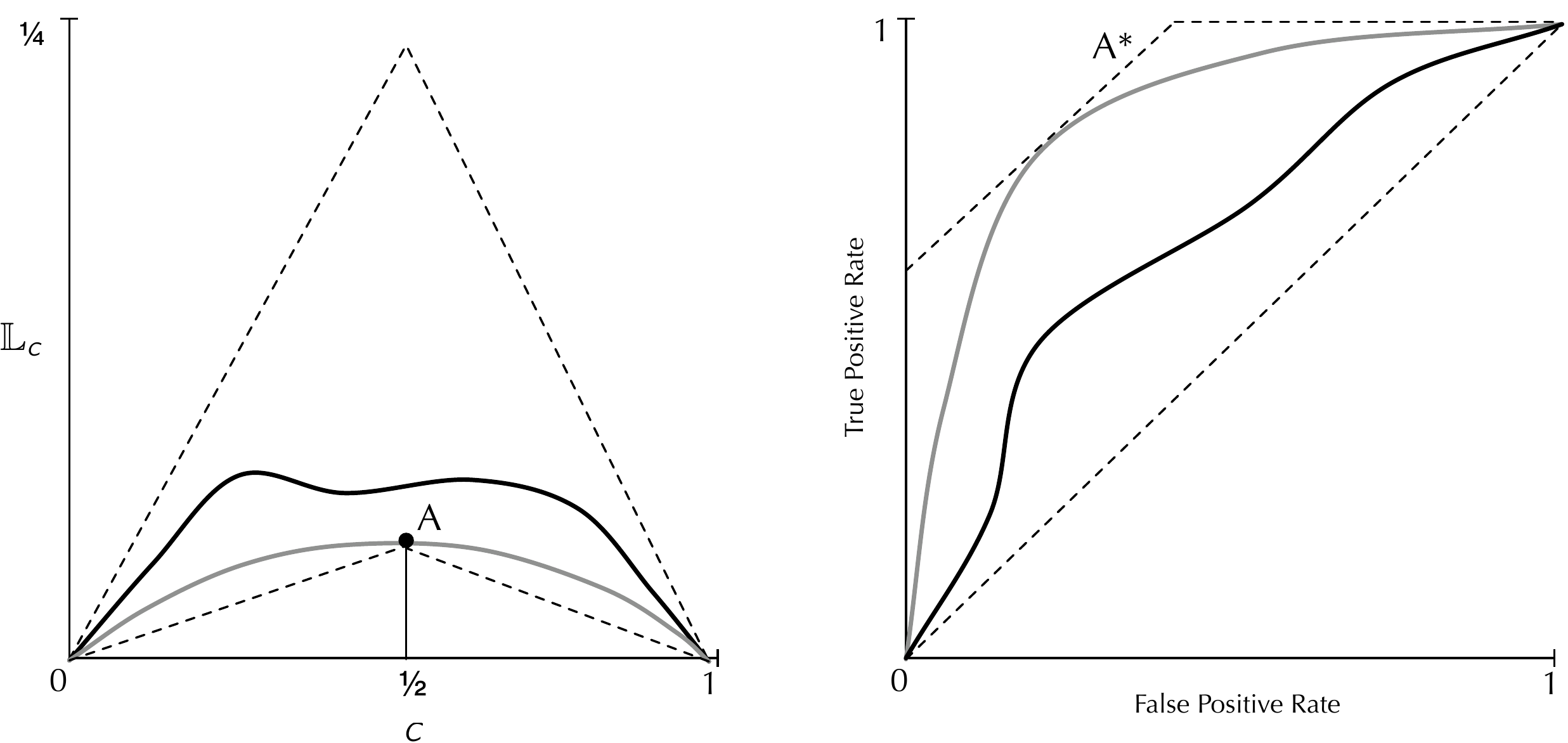}
     \end{center}
 \caption{Cost curve diagram (left) and corresponding ROC diagram (right).
 The black curves on the left and right represent risk and classification rates 
 of an example predictor. 
 The grey Bayes risk curve on the left corresponds to the dominating grey ROC
 curve on the right for the likelihood statistic. 
 Similarly, the dashed tent on the left corresponds to the 
 dashed diagonal ROC line on the right. 
 The point labelled A in the risk diagram corresponds to the line labelled 
 A* in the ROC diagram. 
 \label{fig:variational_roc_bound}}
 \end{figure}

As mentioned earlier, the Neyman-Pearson lemma guarantees the
ROC curve for $\eta$ is maximal. This corresponds to the cost curve being
minimal. In fact, these relationships are dual in the sense that there exists
an explicit transformation from one to the other.

%%%%%%%%%%%%%%%%%%%%%%%%%%%%%%%%%%%%%%%%%%
\subsection{Transforming from ROC to Risk curves and Back}
\label{sub:ROCandRisk}

Recall from Section~\ref{sub:NPLemma} the Neyman-Pearson function 
$\beta(\alpha,P,Q)$ for the binary experiment $(P,Q)$.
Since the true positive rate for $r$ is 
$TP_r = P(r^{-1}(1))$ and the false positive rate for $r$ is 
$FP_r = Q(r^{-1}(1))$ we have
\[
	\beta(\alpha,P,Q) 
	= \sup_{r\in\{-1,1\}^{\mathcal{X}}} 
		\{P(\Xcal_r^+)\ \colon\ Q(\Xcal_r^+)\le\alpha\}
\]
where $\Xcal_r^+ := r^{-1}(1)$.

Noting that the 0-1 loss of $r$ is simply its probability of error --- that is, 
the average of the false positive and false negative rates --- we have for each 
$\pi\in[0,1]$ that the Bayes optimal 0-1 loss is
\begin{equation}
   \minLL(\pi,P,Q) = \inf_{r} \{ (1-\pi) Q(\Xcal_r^+) + \pi(1-P(\Xcal_r^+)) \}.
\end{equation}
since the false negative rate $FN_r = P(\Xcal - \Xcal_r^+) = 1 - P(\Xcal_r^+)$.
Thus for all $\pi,\alpha\in[0,1]$, and all measurable functions
$r\colon\mathcal{X}\rightarrow \{-1,1\}$,
\begin{eqnarray*}
    \minLL(\pi,P,Q) &\le & (1-\pi) Q(\Xcal_r^+) + \pi(1-P(\Xcal_r^+))\\
    &\le & (1-\pi) \alpha + \pi(1-P(\Xcal_r^+))\\
    &\le & (1-\pi)\alpha + \pi(1-\beta(\alpha,P,Q)).
\end{eqnarray*}
Thus, we see that $\minLL(\pi,P,Q)$ is the largest number $\minLL$ such that 
$(1-\pi)\alpha+\pi(1-\beta(\alpha) ) \ge \minLL$ for all $\alpha\in[0,1]$
and hence one can set 
\begin{equation}
    \label{eq:minLL-in-terms-of-beta}
    \minLL(\pi,P,Q) = \minLL = \min_{\alpha\in[0,1]} ( (1-\pi)\alpha +
    \pi(1-\beta(\alpha) )
\end{equation}
for each $\pi\in[0,1]$.

Conversely, we can express the Neyman-Pearson function $\beta$ in terms of the
Bayes risk. 
That is, for any $\alpha\in[0,1]$, $\beta(\alpha,P,Q)$ is the largest number
$\beta$ such that 
\begin{eqnarray*}
    & \forall\pi\in[0,1]\ \ & (1-\pi)\alpha+\pi(1-\beta) \ge \minLL(\pi)\\
   \Leftrightarrow &  \forall\pi\in[0,1] \ \ &(1-\pi)\alpha -\minLL(\pi) \ge
   \pi(\beta-1)\\
   \Rightarrow  & \forall\pi\in(0,1] \ \  & \frac{1}{\pi} ( (1-\pi)\alpha
   -\minLL(\pi)) \ge \beta -1\\
   \Leftrightarrow &  \forall\pi\in(0,1]\ \ & \beta \le \frac{1}{\pi}
   ( (1-\pi)\alpha +\pi-\minLL(\pi) ).
\end{eqnarray*}
Thus we can set 
\begin{equation}
    \label{eq:beta-in-terms-of-minLL}
    \beta(\alpha) = \inf_{\pi\in(0,1]} \frac{1}{\pi} (
    (1-\pi)\alpha+\pi-\minLL(\pi) ), \ \ \alpha\in[0,1].
\end{equation}

The expressions (\ref{eq:beta-in-terms-of-minLL}) and 
(\ref{eq:minLL-in-terms-of-beta}) are due to
\cite{torgersen1991cse}. When $\beta(\cdot)$ and $\minLL(\cdot)$ are smooth,
explicit closed form formulas can be found:
\begin{theorem}
     Suppose $\beta$ and $\minLL$ are differentiable on $(0,1]$ and $[0,1]$
     respectively. Then
     \begin{equation}
	 \minLL(\pi) = (1-\pi) \check{\beta}(\pi) +
	 \pi(1-\beta(\check{\beta}(\pi))),\ \ \pi\in[0,1],
	 \label{eq:minLL-in-terms-of-beta-explicit}
     \end{equation}
     where
     \[
     \check{\beta}(\pi):= {\beta'}^{-1}
       \left(\frac{1-\pi}{\pi}\right)
     \]
     and
     \begin{equation}
	 \beta(\alpha) = \frac{1}{\check{\minLL}(\alpha)}
	 \left[(1-\check{\minLL}(\alpha))\alpha +\check{\minLL}(\alpha) -
	 \minLL(\check{\minLL}(\alpha))\right], \ \ \alpha\in(0,1],
	 \label{eq:beta-in-terms-of-minLL-explicit}
     \end{equation}
     where
     \begin{eqnarray*}
	  \check{\minLL}(\alpha) &:=& \tilde{\minLL}^{-1}(\alpha) \wedge 1,\\
	  \tilde{\minLL}(\pi) &:=& \minLL(\pi) - \pi \minLL'(\pi).
     \end{eqnarray*}
\end{theorem}
\begin{proof}
     Consider the right side of (\ref{eq:minLL-in-terms-of-beta}) and
     differentiate with respect to $\alpha$:
     \[
     \frac{\partial}{\partial\alpha} (1-\pi)\alpha +\pi(1-\beta(\alpha)) =
     (1-\pi)-\pi\beta'(\alpha).
     \]
     Setting this to zero we have $(1-\pi)=\pi\beta'(\alpha)$ and thus
     $\beta'(\alpha)=\frac{1-\pi}{\pi}$. Since $\beta$ is
     monotonically increasing and concave, $\beta'$ is monotonically 
     decreasing and non-negative. Thus we can set
     \[
     \alpha = {\beta'}^{-1}\left(\frac{1-\pi}{\pi}\right) \in [0,1].
     \]
	 Substituting back into $(1-\pi)\alpha+\pi(1-\beta(\alpha))$ we obtain
	 (\ref{eq:minLL-in-terms-of-beta-explicit}).

     Now consider the right side of (\ref{eq:beta-in-terms-of-minLL}):
     \begin{equation}
	 \frac{1}{\pi} ( (1-\pi)\alpha+\pi-\minLL(\pi)) .
	 \label{eq:beta-expression}
     \end{equation}
     Differentiating with respect to $\pi$ we have
     \(
     \frac{-\alpha}{\pi} - \frac{\minLL'(\pi)}{\pi} +
     \frac{\minLL(\pi)}{\pi^2}.
     \)
     Setting this equal to zero we obtain
     \begin{eqnarray*}
     & & \frac{-\alpha}{\pi} - \frac{\minLL'(\pi)}{\pi} + \frac{\minLL(\pi)}{\pi^2}
     = 0,\ \ \pi\in(0,1]\\
     \Rightarrow & & \alpha+\pi\minLL'(\pi)-\minLL(\pi)=0.
     \end{eqnarray*}
     Observing the definition of $\tilde{\minLL}$ we thus have that 
     $\tilde{\minLL}(\pi)=\alpha$. Now 
     \begin{eqnarray*}
	 \tilde{\minLL}'(\pi) & =& \frac{\partial}{\partial\pi} (-\pi
	 \minLL'(\pi) +\minLL(\pi))\\
	 &=& -\pi\minLL''(\pi)-\minLL'(\pi) +\minLL'(\pi)\\
	 &=& -\pi\minLL''(\pi)\\
	 &\ge & 0
     \end{eqnarray*}
     since $\minLL$ is concave. Thus $\tilde{\minLL}(\cdot)$ is monotonically
     non-decreasing and we can write  $\pi=\tilde{\minLL}^{-1}(\alpha)$.
     In order to ensure $\pi\in[0,1]$ we substitute
     $\pi=\check{\minLL}(\alpha)$ into (\ref{eq:beta-expression}) to obtain
	 (\ref{eq:beta-in-terms-of-minLL-explicit}).
\end{proof}
 Using (\ref{eq:beta-in-terms-of-minLL-explicit}) we present an
	 example. Consider (for $\gamma\in[0,1]$)
	 $
	 \minLL(\pi)=\gamma \pi(1-\pi).
	 $
	 One can readily check that $\tilde{\minLL}(\pi)= \gamma\pi^2$. Hence
	 $
	 \tilde{\minLL}^{-1}(\alpha) =
	 \sqrt{\frac{\alpha}{\gamma}}\in\left[0,\frac{1}{\gamma}\right].
	 $
	 Thus $\check{\minLL}(\alpha)= 0\vee \tilde{\minLL}^{-1}(\alpha)\wedge
	 1 = \sqrt{\alpha/\gamma} \wedge 1$.  Substituting and rearranging we
	 find that the corresponding $\beta$ is given by
	 \[
	 \beta_\gamma(\alpha) = \frac{\alpha+\gamma+(\sqrt{\alpha/\gamma}\wedge
	 1)(1-\alpha-\gamma)}{\sqrt{\alpha/\gamma}\wedge 1}.
	 \]
	 A graph of this $\beta(\cdot)$ is given in
	 figure~\ref{figure:beta-family}.
	 \begin{figure}[t]
	     \begin{center}
		\includegraphics[width=0.7\textwidth]{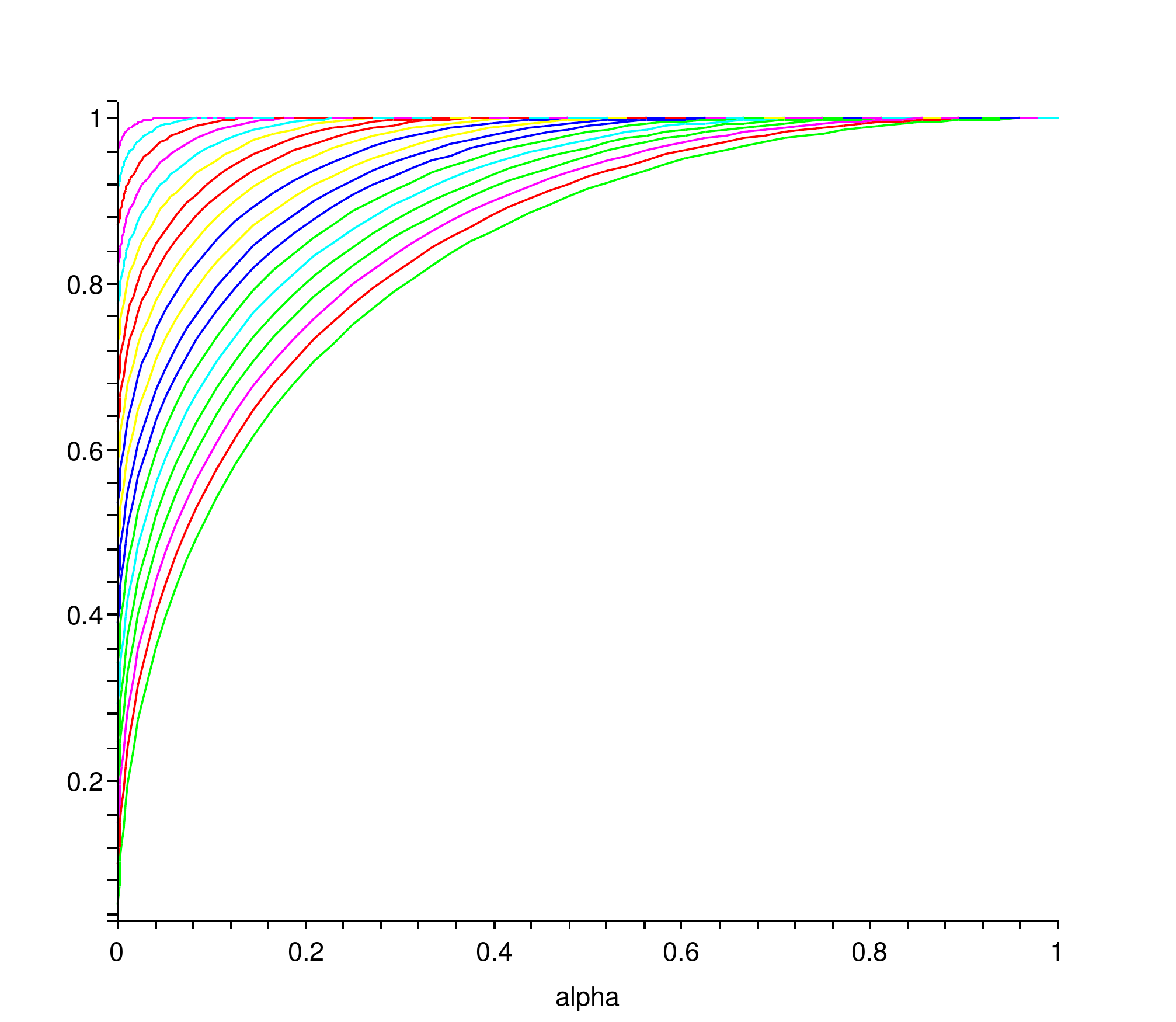}
	    \end{center}
	    \caption{Graph of $\alpha\mapsto\beta_\gamma(\alpha,P,Q)$ for
	    $\gamma=i/20$, $i=1,\ldots,20$.
	    \label{figure:beta-family}}
	 \end{figure}

By construction $\beta(1)=1$ and $\beta$ is concave and continuous on $(0,1]$.
The following lemma is due to \cite{torgersen1991cse}.

\begin{lemma}
    Suppose $\mathcal{X}$ contains a connected component $\mathcal{C}$.
     Let $\phi\colon[0,1]\rightarrow [0,1]$ be an arbitrary function that is
     concave and continuous on $(0,1]$ such that $\phi(1)=1$. Then there exists
     $P$ and $Q$ such that $\beta(\alpha,P,Q)=\phi(\alpha)$ for all
     $\alpha\in[0,1]$.
\end{lemma}
\begin{proof}
    Let $\mathcal{X'}=[0,1]$ and $P$ be the uniform distribution on
    $\mathcal{X'}$. Overload $P$ and $Q$ to also denote the respective
    cumulative distribution functions (i.e. $P(x)=P([0,x])$). Thus
    $P(\pi)=\pi)$. Set $Q(\pi)=\phi(\pi)$. Since $\phi(\cdot)$ is increasing it
    suffices to consider $r(\cdot)$ of the form $r_\pi(x)=\test{x<\pi}$. Hence
    \[
    \beta(\alpha)= \max\{\phi(\pi)\colon 0\le \pi\le 1,\ \pi\le\alpha\}, \
    \alpha\in[0,1].
    \]
    The maximum will always be obtained for $\pi=\alpha$ and thus
    $\beta(\alpha)=\phi(\alpha)$ for $\alpha\in[0,1]$. Finally, a pair of
    distributions on $\mathcal{X}$ can be constructed by embedding the
    connected component $\mathcal{C}\subset\mathcal{X}$ into $\mathcal{X}'$.
    Choose $g\colon\mathcal{C}\rightarrow\mathcal{X}'$ such that $g$ is
    invertible. Such a $g$ always exists since $\mathcal{C}$ is connected. Then
    $g^{-1}$ induces distributions $P'$ and $Q'$ on $\mathcal{C}$ and thus on
    $\mathcal{X}$ by subsethood.
\end{proof}
\begin{corollary}
    Let $\psi\colon [0,1]\rightarrow[0,1]$ be an arbitrary concave function
    such that  for all $\pi\in[0,1]$, $0\le\psi(\pi)\le \pi\wedge(1-\pi)$. Then
    there exists $P$ and $Q$ such that $\minLL(\pi,P,Q)=\psi(\pi)$ for all
    $\pi\in[0,1]$.
    \label{corollary:all-bayes-risk-curves-possible}
\end{corollary}
\begin{proof}
    Choose a $\psi$ satisfying the conditions and substitute into 
    (\ref{eq:beta-in-terms-of-minLL}). This gives a corresponding
    $\phi(\cdot)$. We 
    know from the preceeding lemma that there exist $P$ and $Q$ such that 
    $\beta(\cdot,P,Q)=\phi(\cdot)$ which corresponds to $\minLL(\cdot,P,Q)$.
    Thus it remains to show that the function $\phi$ defined by 
    \[
    \phi(\alpha)=\inf_{\pi\in(0,1]} \frac{1}{\pi}( (1-\pi)\alpha +\pi -
    \psi(\pi))
    \]
    is concave and satisfies $\phi(1)=1$. Observe that 
    $
    \beta(1) = \inf_{\pi\in(0,1]} \frac{1-\psi(\pi)}{\pi}.
    $
    Now by the upper bound on $\psi$, we have
    $
    \frac{1-\psi(\pi)}{\pi} \ge \frac{1-1+\pi}{\pi}=\frac{1}{\pi}\ge 1.
    $
    But $\lim_{\pi\rightarrow 1} \frac{1-\psi(\pi)}{\pi} =1$ and thus
    $\beta(1)=1$. Finally note that 
    \[
    \beta(\alpha)=\inf_{\pi\in(0,1]} \left(\frac{1-\pi}{\pi}\right) \alpha +
    (1-\psi(\pi)).
    \]
    This is the lower envelope of  a parametrized (by $\pi$) family of affine
    functions (in $\alpha$) and is thus concave.
\end{proof}

%%%%%%%%%%%%%%%%%%%%%%%%%%%%%%%%%%%%%%%%%%%%%%%%%%%%%%%%%%%%%%%%%%%%%%%%%%%%%%
\section{Bounding General Objects in Terms of Primitives}\label{sec:inequalities}

All of the above results are exact --- they are exact representations of
particular primitives or general objects in terms of other primitives. 
Another type of
relationship is an inequality. In this section we consider how we can (tightly)
bound the value of a general object ($\II_f$ or $B_w$) in terms of primitive
objects ($V_\pi$ --- defined below --- or $B_c$). Bounding $\II_f(P,Q)$ in terms of $V_\pi(P,Q)$ is a
generalisation of the classical Pinsker inequality. Bounding
$B_w(\eta,\heta)$ in terms of $B_c(\eta,\heta)$ is a generalisation of the
so-called ``surrogate loss bounds.''

As explained previously, we work with the \emph{conditional} Bregman divergence
$B_w(\eta,\heta)$. Results in terms of $B_w(\eta,\heta)$, $\eta,\heta\in[0,1]$
immediately imply results for
$\BB_w(\eta,\heta)$, $\eta,\heta\colon\mathcal{X}\rightarrow[0,1]$
by taking expectations with respect to $\Xsf$.

%%%%%%%%%%%%%%%%%%%%%%%%%%%%%%%%%%%%%%%%%%%%%
\subsection{Surrogate Loss Bounds}

\label{section:surrogate-loss-bounds}
Suppose for some fixed $c_0\in(0,1)$ that $B_{c_0}(\eta,\heta)=\alpha$. What
can be said concerning the value of $B_w(\eta,\heta)$ for an arbitrary weight
function $w$? This is known as a surrogate loss bound. Previous works on this
problem are summarised in Appendix \ref{section:surrogate}. Apart from its
theoretical interest, the question has direct practical implications: it
can often be much simpler to minimise $B_w(\eta,\heta)$ over $\heta$ than to
minimise $B_c(\eta,\heta)$. The bounds below will tell the user of such a
scheme the maximum price they will have to pay in terms of statistical
performance.
\begin{theorem}
    Let $c_0\in(0,1)$ and suppose it is known that 
    $B_{c_0}(\eta,\heta)=\alpha\in(0,c_0)$. 
    Let $w$, $W$ and $\Wb$ be as in Theorem
    \ref{theorem:B-w}. Then
    \begin{equation}
	B_w(\eta,\heta) \ge [\Wb(c_0-\alpha)+\alpha W(c_0)] \wedge
	[\Wb(c_0+\alpha)-\alpha W(c_0)] - \Wb(c_0).
	\label{eq:B-w-bound}
    \end{equation}
    If $c_0=\frac{1}{2}$ and $w$ is
    symmetric about $\thalf$ ($w(\thalf+\alpha)=w(\thalf-\alpha)$ for
    $\alpha\in(0,\thalf)$)
    then
    \begin{eqnarray}
	B_w(\eta,\heta) &\ge & \Wb(\thalf+\alpha)-\alpha W(\thalf)-\Wb(\thalf)
	\label{eq:B-w-bound-c-1-2}\\
	&=&\Wb(\thalf-\alpha)+\alpha W(\thalf)-\Wb(\thalf).
    \end{eqnarray}
    Furthermore (\ref{eq:B-w-bound}) and (\ref{eq:B-w-bound-c-1-2}) are 
    the best possible.
    \label{theorem:surrogate-loss}
\end{theorem}
\begin{proof}
    By hypothesis $B_{c_0}(\eta,\heta)=\alpha$ and thus from (\ref{eq:B-c-0})
    it must be true that either ($\eta\le c_0$ and $\heta > c_0$) or
    ($\eta>c_0$ and $\heta\le c_0$). Suppose for now it is the former. We need
    to determine the minimum possible value of $B_w(\eta,\heta)$. From
    (\ref{eq:B-w-general-form}) we thus seek
    \begin{equation}
    \min_{\begin{array}{l} \eta\in[0,c_0],\ \heta\in[c_0,1]\\
	B_{c_0}(\eta,\heta)=\alpha \end{array}}
	\Wb(\eta)-\Wb(\heta)-(\eta-\heta)W(\heta) .
	\label{eq:surrogate-argmin}
    \end{equation}
    From case 1 of (\ref{eq:B-c-0}) we know $c_0-\eta=\alpha$ and hence
    $\eta=c_0-\alpha$ and the problem is reduced to determining
    \[
    \min_{\heta\in [c_0,1]}
	\Wb(c_0-\alpha)-\Wb(\heta)-(c_0-\alpha-\heta)W(\heta).
    \]
    Differentiating the above expression with respect to  $\heta$ we obtain
    \begin{equation}
	\frac{\partial}{\partial\heta}
	\Wb(c_0-\alpha)-\Wb(\heta)-(c_0-\alpha-\heta)W(\heta)
	= -(c_0-\alpha-\heta)w(\heta) =:\gamma.\label{eq:gamma-temp}
    \end{equation}
    By assumption (for now) we have $\heta>c_0$ and thus
    \begin{equation}
	c_0-\alpha-\heta\in[c_0-\alpha-1,c_0-\alpha-c_0]=
	[\underbrace{c_0-1}_{\le 0} -\alpha,\underbrace{-\alpha}_{<0}].
	\label{eq:c-0-interval}
    \end{equation}
    Equations \ref{eq:gamma-temp} and \ref{eq:c-0-interval}
    together imply that $\gamma\ge 0$. Clearly $\gamma$ can only equal zero if
    $w(\heta)=0$ for some $\heta\in [c_0,1]$.
    Since the derivative is consequently everywhere non-negative, the minimum
    occurs at the minimum value $\heta$; that is at $\heta=c_0$. Subsituting
    for this value of $\heta$ into (\ref{eq:B-w-general-form}) we obtain 
    \begin{equation}
	 B_w(\eta,\heta) \ge \Wb(c_0-\alpha)-\Wb(c_0)+\alpha W(c_0).
	 \label{eq:B-w-bound-case-1}
    \end{equation}

    If instead we have $\eta>c_0$ and $\heta<c_0$, we have (from case 2 of
    (\ref{eq:B-c-0})) that $\alpha=\eta-c_0$ and thus $\eta=\alpha+c_0$ and we
    need to determine
    \[
    \min_{\heta\in [0,c_0]}
	\Wb(\alpha+c_0)-\Wb(\heta)-(\alpha+c_0-\heta)W(\heta).
    \]
    Again differentiating with respect to $\heta$ we obtain 
    \[
    \frac{\partial}{\partial\heta}
	\Wb(\alpha+c_0)-\Wb(\heta)-(\alpha+c_0-\heta)W(\heta) =
	-(\alpha+c_0-\heta)w(\heta)=:\gamma.
    \]
    Furthermore we have $\heta\in[0,c_0]$ and so
    $
    (\alpha+c_0-\heta)\in[\alpha+c_0-c_0,\alpha+c_0]
    $
    and thus $\gamma\ge 0$ and can only equal zero if $w(\heta)=0$.
    Since the derivative is consequently everywhere
    non-positive, the minimum occurs at the maximum possible value of $\heta$
    namely $\heta=c_0$. Substituting
    for this value of $\heta$ into (\ref{eq:B-w-general-form}) we obtain 
    \begin{equation}
	B_w(\eta,\heta) \ge \Wb(c_0+\alpha)-W(c_0)-\alpha W(c_0) .
	 \label{eq:B-w-bound-case-2}
    \end{equation}
    Combining (\ref{eq:B-w-bound-case-1}) and (\ref{eq:B-w-bound-case-2}) gives
	(\ref{eq:B-w-bound}).

    If $c_0=\thalf$ and $w$ is symmetric about $\thalf$ then for 
    $\alpha\in[0,\thalf]$ we have
    \begin{eqnarray*}
	& & w(\thalf-\alpha) = w(\thalf+\alpha)\\
	&\Rightarrow & \int w(\thalf-\alpha) d\alpha = \int w(\thalf+\alpha)d\alpha\\
	&\Rightarrow & W(\thalf) -W(\thalf-\alpha) = W(\thalf+\alpha)-W(\thalf)\\
	&\Rightarrow & \int  W(\thalf) -W(\thalf-\alpha)d\alpha
	   = \int W(\thalf+\alpha)-W(\thalf) d\alpha\\
       &\Rightarrow & \Wb(\thalf-\alpha) +\alpha W(\thalf) 
        = \Wb(\thalf+\alpha) -\alpha W(\thalf) ,
    \end{eqnarray*}
    in which case (\ref{eq:B-w-bound}) reduces to 
    (\ref{eq:B-w-bound-c-1-2}).  

    We finally demonstrate the tightness of the bound.  
    Since $\eta=\pi\frac{dP}{dM}$, by choosing and arbitrary $\pi\in(0,1)$ and
    $M$ uniform on $\mathcal{X}$ we
    have $\eta(x)=\eta$ for all $x\in\Xcal$. Furthermore given any desired
    $\eta\colon\Xcal\rightarrow [0,1]$ there exists a $P$ and $Q$ that 
    generates $\eta(\cdot)$. Furthermore $\heta(\cdot)$ can be an arbitrary
    function on $\Xcal$. Thus one can take $(\eta,\heta)$ to be 
    induced by the arg min in (\ref{eq:surrogate-argmin}). By the above
    construction there exists $\eta(\cdot)$ and $\heta(\cdot)$ such that the
    constructed $(\eta,\heta)$ are the corresponding values conditioned on
    $x\in\Xcal$. Thus there exists $(\pi,P,Q)$ such that 
    (\ref{eq:B-w-bound}) is obtained and thus the bound is tight.
\end{proof}
So far we have glossed over the constants of integration implicit in defining
$W$ and $\Wb$ in terms of $w$. Replacing $W(c)$ by $W(c)+k_1$ and $\Wb(c)$ by
$\Wb(c)+ck_1 +k_2$ and substituting into (e.g) (\ref{eq:B-w-bound-c-1-2}) we
obtain
\begin{eqnarray}
& & \Wb(\thalf+\alpha)+k_1(\thalf+\alpha) +k_2 
    -\alpha W(\thalf)-\alpha k_1
  -\Wb(\thalf)-k_\thalf-k_2\\
&=& \Wb(\thalf+\alpha)-\alpha W(\thalf)-\Wb(\thalf)
+\underbrace{k_1/2 -k_1/2 + k_1\alpha -k_1\alpha  +k_2 -k_2}_{=0}.
\label{eq:constants-of-integration}
\end{eqnarray}
Thus the choice of the constants of integration do not affect the bound. 

One can take a Taylor series expansion of 
(\ref{eq:B-w-bound-c-1-2}) in $\alpha$ about zero to obtain
\[
B_w(\eta,\heta) \ge \frac{w(\thalf)}{2}  \alpha^2 +\frac{w''(\thalf)}{24}
    \alpha^4 +\cdots
\]
There is no third order term since for (\ref{eq:B-w-bound-c-1-2}) to hold $w$
is symmetric and thus $w'(\thalf)=0$. This corresponds to the second order result
presented in \cite{Buja:2005}.

%%%%%%%%%%%%%%%%%%%%%%%%%%%%%%%%%%%%%%%%%%%%%%%%%%%%%%%%%%%%%%%%%%%%%%%
\subsection{General Pinsker Inequalities for Divergences}
The many different $f$ divergences are single number summaries of the
relationship between two distributions $P$ and $Q$. Each $f$-divergence
emphasises different aspects. Merely considering the functions $f$ by which
$f$-divergences are traditionally defined makes it hard to understand these
different aspects, and harder still to understand how knowledge of $\II_{f_1}$
constrains the possible values of $\II_{f_2}$. When $\II_{f_1}=V$ (a special
primitive for $\II_f$) and
$\II_{f_2}=\mathrm{KL}$, this a classical problem that has been studied for
decades.

\cite{Vajda1970} posed the question of a \emph{tight lower bound} on
KL-divergence in terms of variational divergence. This ``best possible Pinsker
inequality'' takes the form
\begin{equation}
L(V) :=\inf_{V(P,Q)=V} \mathrm{KL}(P,Q), \ \ \ \ V\in[0,2)
\label{eq:vajda-tight-lower-bound}
\end{equation}
Recently Fedotov et al.~\cite{FedotovTopsoe2003}
presented an {\em implicit} (parametric) version of the form 
\begin{eqnarray}
     & & (V(t),L(t))_{t\in\reals^+} \label{eq:implicit-pinsker}\\
    & & V(t) = t\left(1-\left(\coth(t)-\frac{1}{t}\right)^2\right),\ \ \ \ \ 
    L(t) = \log\left(\frac{t}{\sinh(t)}\right)+t\coth(t)
    -\frac{t^2}{\sinh^2(t)}\nonumber.
\end{eqnarray}
We will now show how viewing $f$-divergences in terms of their weighted
integral representation simplifies the problem of understanding the
relationship between different divergences and leads, amongst other things, to
an explicit formula for~(\ref{eq:vajda-tight-lower-bound}).

Fix a positive integer $n$. 
Consider a sequence $0<\pi_1 <\pi_2 < \cdots <\pi_n < 1$. Suppose we
``sampled'' the value of $\SI(\pi,P,Q)$ at these discrete values of $\pi$. This
is equivalent to knowing the values of the ``narrowband'' (so called because
its weight is $\gamma(\pi)=4\delta(\pi-\pi_i)$) primitive generalised 
variational divergence $V_{\pi_i}(P,Q):=V_{[-1,1]^{\mathcal{X}},\pi_i}(P,Q)$.

Since $\pi\mapsto\minLL(\pi,P,Q)$ is concave, the piece-wise linear  
concave function passing through points $\{(\pi_i,
\minLL^{0-1}(\pi_i,P,Q))\}_{i=1}^n$ is guaranteed to be an upper bound on the
Bayes risk curve $(\pi, \minLL^{0-1}(\pi,P,Q))_{\pi\in(0,1)}$.  This therefore
gives a lower bound on the statistical information for a task with loss given
by a weight function $\gamma$ and therefore a lower bound on the $f$-divergence
$\II_f(P,Q)$ corresponding to the statistical information. This observation forms
the basis of the theorem stated below.

\begin{theorem}
    \label{theorem:general-pinsker}
For a  positive integer $n$
consider a sequence $0<\pi_1<\pi_2 <\cdots <\pi_n< 1$. Let $\pi_0:=0$ and
$\pi_{n+1}:=1$. Let
\[
\psi_i:=\minLL^{0-1}(\pi_i,P,Q),\ \ \  i=0,\ldots,n+1
\]
(observe that consequently
$\psi_0=\psi_{n+1}=0$). Let 
\begin{equation}
    A_n:=\left\{\abold=(a_1,\ldots,a_n)\in\reals^n
   \colon \frac{\psi_{i+1}-\psi_i}{\pi_{i+1}-\pi_i} \le
   a_i \le\frac{\psi_i-\psi_{i-1}}{\pi_i-\pi_{i-1}},\ i=1,\ldots,n\right\}
   . \label{eq:A-n}
\end{equation}
The set $A_n$ defines the allowable slopes of a piecewise linear function
majorizing $\pi\mapsto\minLL^{0-1}(\pi,P,Q)$ at each of $\pi_1,\ldots,\pi_n$.
For $\abold=(a_1,\ldots,a_n)\in A_n$, let 
\begin{eqnarray}
\tilde{\pi}_i &:= &\frac{\psi_i-\psi_{i+1}+a_{i+1}
    \pi_{i+1}-a_i\pi_i}{a_{i+1}-a_i},\ \  i=0,\ldots,n,\\
j &:= & \{k\in\{1,\ldots,n\}: \tilde{\pi}_{k} <\textstyle\frac{1}{2}\le 
	\tilde{\pi}_{k+1}\}.  \label{eq:j-def}\\
\bar{\pi}_i &:=&\test{i<j}\tilde{\pi}_i +
    \test{i=j} \textstyle\frac{1}{2} + \test{j<i}\tilde{\pi}_{i-1},
      \label{eq:bar-pi}\\
\alpha_{\abold,i} &:= & \test{i\le
j}(1-a_i)+\test{i>{j}}(-1-a_{i-1}),\label{eq:alpha-abold}\\
\beta_{\abold,i}&:=&\test{i\le{j}}(\psi_i-a_i\pi_i)+
\test{i>{j}}(\psi_{i-1}-a_{i-1}\pi_{i-1})\label{eq:beta-abold}
\end{eqnarray}
for $i=0,\ldots,n+1$ and let $\gamma_f$ be the weight corresponding to  
$f$ given by
(\ref{eq:gamma-f-relationship}).

For arbitrary $\II_f$ and 
    for all distributions $P$ and $Q$ the following bound holds.
    If in addition $\mathcal{X}$ contains a connected component, it is tight.
    \begin{eqnarray}
	\II_f(P,Q) &\ge & \min_{\abold\in A_n} \sum_{i=0}^{n}
	\int_{\bar{\pi}_i}^{\bar{\pi}_{i+1}} 
	(\alpha_{\abold,i}\pi +\beta_{\abold,i}) \gamma_f(\pi) d\pi 
	\label{eq:general-pinsker}\\
	&=& \min_{\abold\in A_n} \sum_{i=0}^{n} \left[
	    \left(\alpha_{\abold,i}\bar{\pi}_{i+1}+
	    \beta_{\abold,i}\right)\Gamma_f(\bar{\pi}_{i+1})
	    -\alpha_{\abold,i}\bar{\Gamma}_f(\bar{\pi}_{i+1}) \right.\nonumber\\
	&& \ \ \ \ \ \ \ \ \ \ \ \left. -\left(\alpha_{\abold,i}\bar{\pi}_{i} +
	    \beta_{\abold,i}\right)\Gamma_f(\bar{\pi}_i)
	+\alpha_{\abold,i}\bar{\Gamma}_f(\bar{\pi}_i)\right] ,
	\label{eq:general-pinsker-evaluated}
\end{eqnarray}
where $\Gamma_f(\pi):=\int^\pi \gamma_f(t) dt$ and 
$\bar{\Gamma}_f(\pi):=\int^\pi \Gamma_f(t) dt$.
\end{theorem}
Equation \ref{eq:general-pinsker-evaluated} follows from
(\ref{eq:general-pinsker}) by integration by parts. The remainder of the 
proof is in appendix \ref{section:pinsker-proofs}. 
Although (\ref{eq:general-pinsker}) looks daunting, we observe: 
(1) the
constraints on $\abold$ are convex (in fact they are a box constraint); and (2)
the objective is a relatively benign function of $\abold$.  
By theorem
\ref{pro:c_vs_pi} and the fact that $2\ell_{\frac{1}{2}}=\ell^{0-1}$,
$\psi_i=\minLL^{0-1}(\pi_i,P,Q)
=2\minLL_{\pi_i}(\frac{1}{2},P,Q)$.

When $n=1$ the result simplifies considerably. If in addition $\pi_1=\thalf$  then
$V_{\frac{1}{2}}(P,Q)=V(P,Q)=2-4\minLL^{0-1}(\frac{1}{2},P,Q)$. It is
then a straightforward exercise to explicitly evaluate
(\ref{eq:general-pinsker}), especially when $\gamma_f$ is symmetric. The
following theorem expresses the result in terms of $V(P,Q)$ for comparability
with previous results. The result for $\mathrm{KL}(P,Q)$ is a
(best-possible) improvement on the classical Pinsker
inequality. The various divergences are
defined in Table~\ref{table:symmetric-divergences}.
\begin{theorem}\label{theorem:special-cases-pinsker}
    For any distributions $P,Q$ on $\mathcal{X}$, let $V:=V(P,Q)$. 
    Then the following bounds hold and, if in 
    addition $\mathcal{X}$ has a connected component, are tight.

    When $\gamma$ is symmetric about $\thalf$ and convex,
    \begin{equation}
	\II_f(P,Q) \ge 2 \left[ 
	\bar{\Gamma}_f\left(\textstyle\frac{1}{2}-\textstyle\frac{V}{4}\right)+
	\textstyle\frac{V}{4}\Gamma_f\left(\textstyle\frac{1}{2}\right)-
	\bar{\Gamma}_f\left(\textstyle\frac{1}{2}\right)\right]
	\label{eq:general-symmetric-pinsker}
    \end{equation}
    and $\Gamma_f$ and $\bar{\Gamma}_f$ are as in 
    Theorem~\ref{theorem:general-pinsker}.
    The following special cases hold.
    \begin{eqnarray}
	h^2(P,Q) &\!\!\!\!\!\ge \!\!\!\! &2 - \sqrt{4-V^2};\hspace*{6mm}
	\mathrm{J}(P,Q) \ge \textstyle 2V\ln
	\left(\frac{2+V}{2-V}\right);\hspace*{6mm}
	\Psi(P,Q)\ge \frac{8V^2}{4-V^2}\nonumber\\
	\mathrm{I}(P,Q) &\!\!\!\!\!\ge\!\!\!\! &\textstyle
	    \left(\frac{1}{2}-\frac{V}{4}\right)\ln(2-V)
	    +\left(\frac{1}{2}+\frac{V}{4}\right)\ln(2+V)-\ln(2)\nonumber\\
	\mathrm{T}(P,Q)&\!\!\!\!\!\ge\!\!\!\! &\textstyle
	   \ln\left(\frac{4}{\sqrt{4-V^2}}\right)-\ln(2). \nonumber
	\end{eqnarray}

	When $\gamma$ is not symmetric, the following special cases hold:
	\begin{eqnarray}
	\chi^2(P,Q) &\!\!\!\!\!\ge\!\!\!\! & \textstyle\test{V<1} V^2 + \test{V\ge
	    1}\frac{V}{(2-V)}\label{eq:my-chi-squared-bound}\\
	\mathrm{KL}(P,Q) &\!\!\!\!\ge &\!\! \min_{\beta\in[V-2,2-V]}\textstyle 
    \left(\frac{V+2-\beta}{4}\right)\ln\left(\frac{\beta-2-V}{\beta-2+V}\right)
    +
    \left(\frac{\beta+2-V}{4}\right)\ln\left(\frac{\beta+2-V}{\beta+2+V}\right).
    \label{eq:my-pinsker-KL}
\end{eqnarray}
\end{theorem}
This theorem gives the first explicit representation of 
the optimal Pinsker bound.\footnote{
     A summary of existing results and their relationship to those presented
     here is given in appendix \ref{section:history}.
}
By plotting both (\ref{eq:implicit-pinsker}) and
(\ref{eq:my-pinsker-KL}) one can confirm 
that  the two bounds (implicit and explicit) coincide; see Figure
\ref{figure:pinsker-curves}. Equation \ref{eq:general-symmetric-pinsker} should
be compared with (\ref{eq:B-w-bound-c-1-2}).
\begin{figure}[t]
    \begin{center}
	\includegraphics[width=0.4\textwidth]{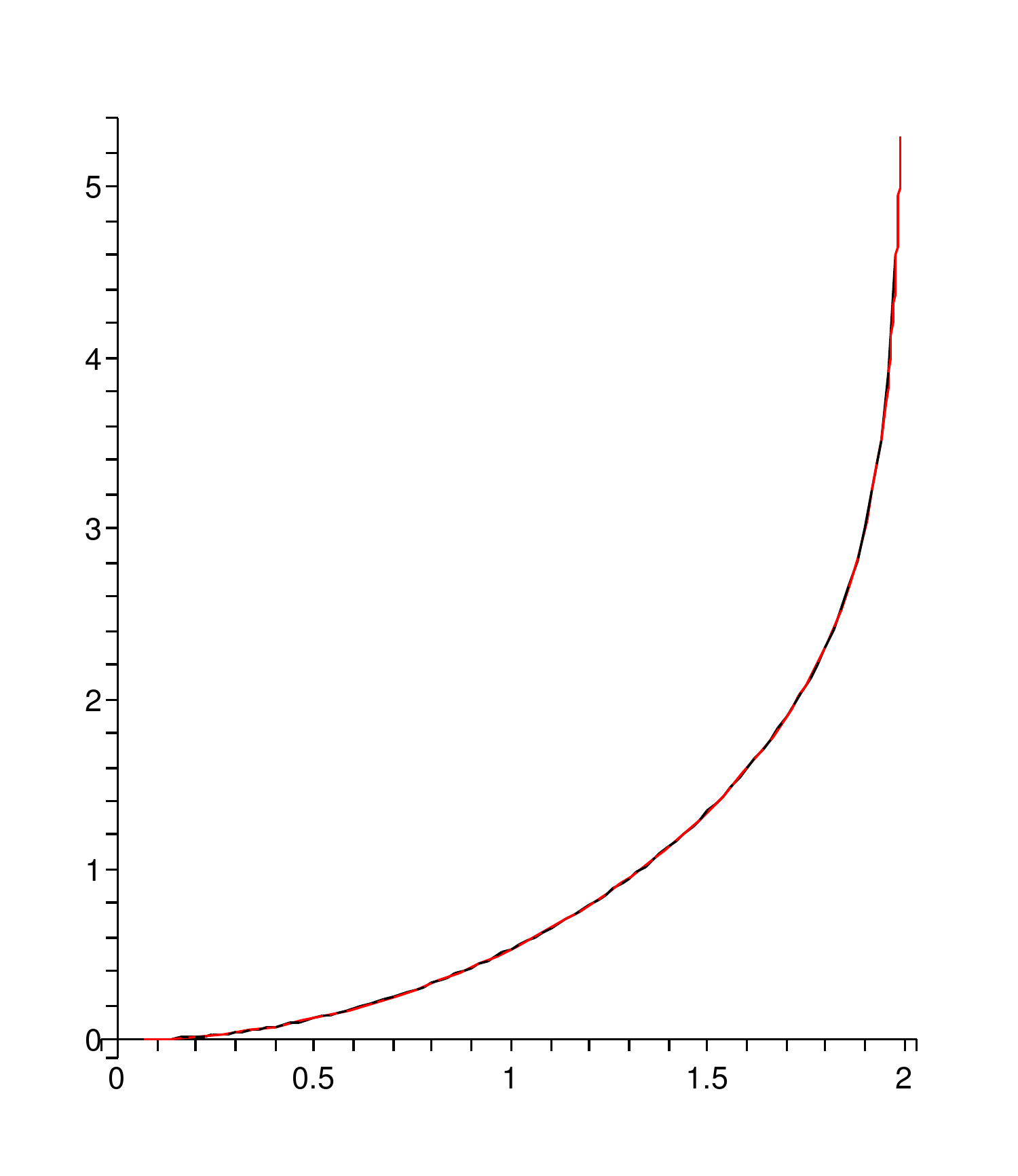}
    \end{center}
\caption{Lower bound on  $\mathrm{KL}(P,Q)$ 
as a function of the variational divergence $V(P,Q)$. Both the explicit bound
(\ref{eq:my-pinsker-KL}) and Fedotorev et al.'s implicit bound
(\ref{eq:implicit-pinsker}) are plotted.
\label{figure:pinsker-curves}}
\end{figure}

The above theorem suggests a means by which one can \emph{estimate} an
$f$-divergence by estimating a sequence $(\minLL_{c_i}(\pi,P,Q))_{i=1}^n$. A
simpler version of such an idea (more directly using the representation 
(\ref{eq:LV-integral-rep})) has been studied by
\citet{SongReidSmolaWilliamson2008}

%%%%%%%%%%%%%%%%%%%%%%%%%%%%%%%%%%%%%%%%%%%%%%%%%%%%%%%%%%%%%%%%%%%%%%

\section{Variational Representations}\label{sec:variational}

We have already seen a number of connections between the Bayes risk
\begin{equation}
    \minLL(\pi,P,Q) = \inf_{\hat{\eta}\in[0,1]^{\mathcal{X}}} 
    \E{\Xsf\sim M}{\ell(\eta(\Xsf),\hat{\eta}(\Xsf))}
    \label{eq:classical-minLL-def}
\end{equation}
and the $f$-divergence 
\begin{equation}
    \II_f(P,Q) = \E{Q}{f\left(\frac{dP}{dQ}\right)}.
    \label{eq:classical-I-f-def}
\end{equation}
Comparing these definitions leads to an obvious and intriguing point: the
definition of $\minLL$ involves an optimisation, whereas that for $\II_f$ does
not. Observe that the normal usage of these quantities is that one normally
wishes to not just know the real number $\minLL(\pi,P,Q)$, but one would like
the estimate $\hat{\eta}\colon\mathcal{X}\rightarrow[0,1]$ that attains the 
minimal risk.  In this section we
will explore two views of $\II_f$ --- relating the standard definition to a {\em
variational} one that explains where the optimisation is hidden in
(\ref{eq:classical-I-f-def}). The easiest place to start, unsurprisingly!, is
with the variational divergence. Below we derive a straight-forward  extension
of the classical result relating $\minLL^{\mathrm{0-1}}(\frac{1}{2},P,Q)$ to
$V(P,Q)$. We then explore variational representations for general
$f$-divergences and consequently develop some new generalisations.

%%%%%%%%%%%%%%%%%%%%%%%%%%%%%%%%%%%%%%%%%%%%%%%
\subsection{Generalised Variational Divergence}

Let $\mathcal{C} \subseteq \{-1,1\}^{\mathcal{X}}$ denote a collection of 
binary \emph{classifiers} on $\mathcal{X}$. Consider the (constrained\footnote{
\cite{TongKoller2000} call this the {\em restricted} Bayes risk.}) Bayes 
risk for 0-1 loss
minimised over this set
\begin{equation}\label{eq:minLLF-def}
    \minLL_{\mathcal{C}}^{0-1}(\pi,P,Q)
    = \inf_{r\in\mathcal{C}}\EE_{(\Xsf,\Ysf)\sim\PP}
    [\loss^{0-1}(r(\Xsf),\Ysf)].
\end{equation}
The variational divergence is so called because it
can be written 
\begin{equation}
    V(P,Q)=2\sup_{A\subseteq\mathcal{X}} |P(A)-Q(A)|,
\label{eq:classic-variational-divergence-def}
\end{equation}
where the supremum is over all measurable
subsets of $\mathcal{X}$. 
Since $V(P,Q)=\sup_{r\in[-1,1]^{\mathcal{X}}} |\EE_P r - \EE_Q r|$, 
consider 
the following generalisation of $V$:
\begin{equation}
    V_{{\mathcal{R}},\pi}(P,Q) := 2\sup_{r\in\mathcal{R}\subseteq[-1,1]^{\mathcal{X}}}
    |\pi\EE_P r - (1-\pi)\EE_Q r| ,
    \label{eq:VGpi-def}
\end{equation}
where $\pi\in(0,1)$.  When $\pi=\thalf$ this is a
scaled version of what  \cite{Muller1997,Muller1997a} 
calls an {\em integral 
probability metric}.\footnote{
	\cite{Zolotarev1984} calls this a  {\em probability metric
	with $\zeta$-structure}.  There are probability metrics that
	are neither $f$-divergences nor integral probability metrics. A large
	collection is due to \cite{Rachev1991}.   A recent survey on
	relationships (inequalities and some representations) has been given by
	\cite{GibbsSu2002}.
}

If $\mathcal{R}$ is {\em symmetric about zero} ($r\in\mathcal{R}\Rightarrow
-r\in\mathcal{R}$),  then the absolute value signs in (\ref{eq:VGpi-def}) can
be removed. To see this, suppose the supremum was attained at $\overline{r}$
and that $\alpha:=\pi\EE_P \overline{r} -(1-\pi)\EE_Q \overline{r} <0$.
Choose $\overline{r}':= -\overline{r}$ and observe that $\pi\EE_P
\overline{r}'-(1-\pi)\EE_Q \overline{r}'= -\alpha>0$.  Thus
$V_{\mathcal{R},\pi}(P,Q)=2\sup_{r\in\mathcal{R}\subseteq[-1,1]^{\mathcal{X}}}
(\pi\EE_P r -(1-\pi)\EE_Q r)$. 

Let $\sgn\mathcal{R}:=\{\sgn r\colon r\in\mathcal{R}\}$ and for $a,b\in\reals$, 
let $a\mathcal{R}+b:=\{ar+b\colon r\in\mathcal{R}\}$.
\begin{theorem}
    Suppose $\mathcal{R}\subseteq [-1,1]^{\mathcal{X}}$ is symmetric about 
    zero and
    $\sgn\mathcal{R}\subseteq\mathcal{R}$.
    For all $\pi\in(0,1)$ and all $P$ and $Q$
    \begin{equation}
	\textstyle\minLL_{(\sgn\mathcal{R} +1)/2}^{0-1}(\pi,P,Q)
	= \frac{1}{2} - \frac{1}{4} V_{\mathcal{R},\pi}(P,Q)
	\label{eq:LL-Vpi}
    \end{equation}
    and the infimum in (\ref{eq:minLLF-def}) corresponds to the supremum in
    (\ref{eq:VGpi-def}).
    \label{theorem:minLL-V-pi}
\end{theorem}
\begin{proof}
    Let $\mathcal{C} :=(\sgn\mathcal{R}+1)/2 \subseteq\{0,1\}^{\mathcal{X}} $ and
    so $\sgn\mathcal{R}=2\mathcal{C} -1$. Then
    \begin{eqnarray}
	\minLL_{\mathcal{C}}^{0-1}(\pi,P,Q) &=& \inf_{r\in\mathcal{C}} 
	\EE_{(\Xsf,\Ysf)\sim\PP}
	\ell^{0-1}(r(\mathsf{X}),\mathsf{Y})\nonumber\\
	&=& \inf_{r\in\mathcal{C}} \left(\pi\EE_{\Xsf\sim P} 
	\ell^{0-1}(r(\mathsf{X}),0) +
	(1-\pi) \EE_{\Xsf\sim Q} \ell^{0-1}(r(\mathsf{X}),1)\right)\nonumber\\
	&=& \inf_{r\in\mathcal{C}} \left(\pi\EE_{\Xsf\sim P}
	\test{r(\mathsf{X})=1}
	+(1-\pi)\EE_{\Xsf\sim Q}\test{r(\mathsf{X})=0}\right)\nonumber\\
	&=& \inf_{r\in\mathcal{C}} \left(\pi\EE_P r +(1-\pi)\EE_Q
	(1-r)\right)\nonumber
    \end{eqnarray}
    since $\ran r=\{0,1\}\Rightarrow \EE_{\Xsf\sim P}\test{r(\mathsf{X})=1}=
    \EE_{\Xsf\sim P} r(\mathsf{X})$ and 
    $\EE_{\Xsf\sim Q}\test{r(\mathsf{X})=0}=\EE_{\Xsf\sim Q}(1-r(\mathsf{X}))$. 
    Let $\rho=2r-1\in 2\mathcal{C}-1$. Thus $r=\frac{\rho+1}{2}$. Hence
    \begin{eqnarray*}
    \minLL_{\mathcal{C}}^{0-1}(\pi,P,Q) 
    &=& \inf_{\rho\in 2\mathcal{C}-1}\left(\pi\EE_P \left(\frac{\rho+1}{2}\right) - 
	   (1-\pi)\EE_Q \left(1-\frac{\rho+1}{2}\right)\right)\\
	&=& \frac{1}{2}\inf_{\rho\in 2\mathcal{C}-1} (\pi\EE_P (\rho+1)
	   +(1-\pi)\EE_Q(1-\rho))\\
	&=&\frac{1}{2}\inf_{\rho\in 2\mathcal{C}-1} 
	   (\pi\EE_P\rho+(1-\pi)\EE_Q(-\rho)+\pi+(1-\pi))\\
	&=&\frac{1}{2}+\frac{1}{2}\inf_{\rho\in2\mathcal{C}-1}(\pi\EE_P\rho-
	  (1-\pi)\EE_Q\rho) \\
        &=& \frac{1}{2}-\frac{1}{2}\sup_{\rho\in 2\mathcal{C}-1} (\pi\EE_P
	  (-\rho) -(1-\pi)\EE_Q(-\rho)) .
    \end{eqnarray*}
    Since $\mathcal{R}$ is symmetric about zero, 
    $\sgn(\mathcal{R})=2\mathcal{C}-1$, 
    $\mathcal{C}\subseteq\{0,1\}^{\mathcal{X}}$
    is symmetric about $\thalf$; i.e. $\rho\in\mathcal{C}\Rightarrow
    (1-\rho)\in\mathcal{C}$. Thus
    \begin{eqnarray}
	\minLL_{\mathcal{C}}^{0-1}(\pi,P,Q)  &=& \frac{1}{2}-\frac{1}{2}
	\sup_{\rho\in 2\mathcal{C}-1} (\pi\EE_P\rho-(1-\pi)\EE_Q\rho)\nonumber\\
	&=&\frac{1}{2}-\frac{1}{4} V_{2\mathcal{C} -1,\pi}(P,Q)\nonumber\\
	&=& \frac{1}{2}-\frac{1}{4} V_{\sgn\mathcal{R},\pi}(P,Q).
	\label{eq:LLC-expression}
    \end{eqnarray}
    Since by assumption $\sgn\mathcal{R}\subseteq\mathcal{R}$, the supremum 
    in (\ref{eq:VGpi-def}) will be $\pm 1$-valued everywhere. 
    Thus $V_{\sgn\mathcal{R},\pi}(P,Q)=V_{\mathcal{R},\pi}(P,Q)$.
    Combining this fact with
    (\ref{eq:LLC-expression}) leads to (\ref{eq:LL-Vpi}).

    Finally observe that by replacing $\inf$ and $\sup$ by $\argmin$ and
    $\argmax$ the final part of the theorem is apparent.
\end{proof}

%%%%%%%%%%%%%%%%%%%%%%%%%%%%
\subsubsection{The Linear ``Loss''}
\label{section:linear-loss}
This theorem shows that computing $V_{\mathcal{R},\pi}$ involves an
optimisation problem equivalent to that arising in the determination 
of~$\minLL$. The $\argmin$
in the definition of $\minLL$ is usually called the {\em hypothesis} (or 
{\em Bayes optimal hypothesis}). Following 
\citet{BorgwardtGrettonRaschKriegelScholkopfSmola2006}  we will call the
$\argmax$ in (\ref{eq:VGpi-def}) the {\em witness}.

When $\mathcal{R}=[-1,1]^{\mathcal{X}}$ and $\pi=\thalf$,
$\sgn\mathcal{R}\subseteq\mathcal{R}$ and furthermore
$\mathcal{C}=(\sgn\mathcal{R}+1)/2=\{0,1\}^{\mathcal{X}}$ and so
Theorem \ref{theorem:minLL-V-pi} reduces to the classical result that
$\minLL^{0-1}(\frac{1}{2},P,Q)=\frac{1}{2}-\frac{1}{4} V(P,Q)$ 
\citep{DevGyoLug96}.

The requirement that
$\sgn\mathcal{R}\subseteq\mathcal{R}$ is unattractive. It is necessitated by
the use of 0-1 loss.  It can be removed by instead considering the linear 
loss.

\iffalse
considering the {\em hinge loss} $\ell^{\mathrm{hinge}}(y,\heta):=
\max((1-y\heta), 0)$. Considering the
general case of $\pi\ne\frac{1}{2}$ we can show (assuming $\mathcal{R}$ is
symmetric and $|r(x)|\le 1$ for all $x\in\mathcal{X}$ and all $r\in\mathcal{R}$)
that
\begin{eqnarray}
    \minLL_{\mathcal{R}}^{{\mathrm{hinge}}}(\pi,P,Q) &=&
    \inf_{r\in\mathcal{R}\subseteq[-1,1]^{\mathcal{X}}} \pi \EE_P (1-r)
    +(1-\pi)\EE_Q(1+r)\label{eq:inf-hinge-loss}\\
    &=&  1 + \sup_{r\in\mathcal{R}}\ \  (1-\pi)\EE_Q r -\pi \EE_P r .\nonumber
\end{eqnarray}
Observing that the Maximum Mean Discrepancy (a new method for measuring
distances between distributions \cite{GrettonBorgwardtRaschScholkopfSmola2007})
$\mathrm{MMD}_{\mathcal{R}}[P,Q]=\sup_{r\in\mathcal{R}} (\EE_P r -\EE_Q r) =
\sup_{r\in\mathcal{R}} (\EE_Q r - \EE_P r) = 2 V_{\mathcal{R},\frac{1}{2}}(P,Q)$ (since $\mathcal{R}$ is symmetric),
we have shown:
\begin{corollary}
    \label{cor:hinge-correspondence}
    Suppose $\mathcal{R}$ is symmetric and $\|r\|_\infty\le 1$ for all
    $r\in\mathcal{R}$. Then 
    \begin{equation}
	\minLL_{\mathcal{R}}^{\mathrm{hinge}}(\textstyle\frac{1}{2},P,Q)=1+\frac{1}{2}
	\mathrm{MMD}_{\mathcal{R}}[P,Q]
    \end{equation}
    and the infimum in (\ref{eq:inf-hinge-loss}) corresponds to the supremum 
    in the definition of MMD.
\end{corollary}
\fi
%%%%%% new Stuff

It is convenient to temporarily switch conventions so that the labels
$y\in\{-1,1\}$. Consider the {\em linear loss}
\[
\loss^{\mathrm{lin}}(r(x),y):=1-yr(x),\ \ \ y\in\{-1,1\}.
\]
If $r$ is unrestricted, then there is no guarantee that
$\loss^{\mathrm{lin}}>-\infty$ and is thus a legitimate loss function. Below we will always consider $r\in\mathcal{R}$
such that the linear loss is bounded from below.
Observe that the common hinge loss
\citep{Steinwart2008}
is simply $\loss^{\mathrm{hinge}}(f(x),y)=
0\vee \loss^{\mathrm{lin}}(f(x),y)$.  
\begin{theorem}
    \label{theorem:linear-loss-equivalence}
    Assume that $\mathcal{R}\subseteq [-a,a]^{\mathcal{X}}$ for some $a>0$ and
    is symmetric about zero.
    Then for all $\pi\in(0,1)$ and all
distributions $P$ and $Q$ on $\mathcal{X}$
\begin{equation}
    \minLL_{\mathcal{R}}^{\mathrm{lin}}(\pi,P,Q) =
    1- \frac{1}{2} V_{\mathcal{R},\pi}(P,Q)\label{eq:minLL-equals-V-R-pi}
\end{equation}
and the $r$ that attains $\minLL_{\mathcal{R}}^{\mathrm{lin}}(\pi,P,Q) $
corresponds to the $r$ that obtains the supremum in the definition of
$V_{\mathcal{R},\pi}(P,Q)$.
\end{theorem}
\begin{proof}
\begin{eqnarray}
    \minLL_{\mathcal{R}}^{\mathrm{lin}}(\pi,P,Q) &=&\inf_{r\in\mathcal{R}}\left(
    \pi\EE_{\Xsf\sim P} \ell^{\mathrm{lin}}(r(\mathsf{X}),-1) + (1-\pi)
    \EE_{\Xsf\sim Q}
	\ell^{\mathrm{lin}}(r(\mathsf{X}),+1)\right) \label{eq:lin-loss-def}\\
	&=& \inf_{r\in\mathcal{R}}\left(\pi\EE_{\Xsf\sim P}(1+r(\mathsf{X})) +
	(1-\pi)\EE_{\Xsf\sim Q}(1-r(\mathsf{X}))\right)\nonumber\\
	&=& \inf_{r\in\mathcal{R}}\left(\pi+\pi\EE_P r +(1-\pi)-(1-\pi)\EE_Q
	r\right)\nonumber\\
	&=& 1+\inf_{r\in\mathcal{R}} \left(\pi\EE_P r -(1-\pi)\EE_Q r\right)\\
	&=& 1-\sup_{r\in\mathcal{R}} \left(\pi\EE_P (-r)
	-(1-\pi)\EE_Q(-r)\right)\nonumber\\
	&=& 1-\sup_{r\in\mathcal{R}} \left(\pi\EE_P r -(1-\pi)\EE_Q
	r\right)\nonumber\\
	&=& 1- \frac{1}{2} V_{\mathcal{R},\pi}(P,Q),\nonumber
\end{eqnarray}
where the penultimate step exploits the symmetry of $\mathcal{R}$. 
\end{proof}

Now suppose that $\mathcal{R}=B_{\mathcal{H}}:=\{r\colon \|r\|_{\mathcal{H}}\le
1\}$, the unit ball in  $\mathcal{H}$, a Reproducing Kernel Hilbert Space 
(RKHS) \citep{ScholkopfSmola2002}. Thus
for all $r\in\mathcal{R}$ there exists a {\em feature map}
$\phi\colon\mathcal{X}\rightarrow\mathcal{H}$ such that $r(x)=\langle
r,\phi(x)\rangle_{\mathcal{H}}$ and
$\langle\phi(x),\phi(y)\rangle_{\mathcal{H}}=k(x,y)$, where $k$ is a positive
definite {\em kernel} function. 
\cite{BorgwardtGrettonRaschKriegelScholkopfSmola2006} show that 
\begin{equation}
    \label{eq:borgwardt}
    V_{B_{\mathcal{H}},\frac{1}{2}}^2(P,Q)=
    \frac{1}{4}\|\EE_P\phi-\EE_Q\phi\|_{\mathcal{H}}^2.
\end{equation}
Thus 
\begin{equation}
    \label{eq:LLlin}
\minLL_{\mathcal{R}}^{\mathrm{lin}}(\pi,P,Q)=
   1-\frac{1}{4}\|\EE_P\phi-\EE_Q\phi\|_{\mathcal{H}}.
\end{equation}
Empirical estimators derived from the correspondence between
(\ref{eq:borgwardt}) and (\ref{eq:LLlin}) 
lead to the $\nu$-Support Vector Machine and Maximum Mean Discrepancy;
see appendix \ref{sec:appendix-svm-mmd}.

Let $\aco\mathcal{R}$ denote the absolute convex hull of $\mathcal{R}$:
\[
\aco \mathcal{R} :=\left\{\sum_i \alpha_i r_i \colon r_i\in\mathcal{R},\
\sum_i |\alpha_i|\le1,\ \alpha_i\in\reals\right\}
\]
The following is a minor generalisation of a result due to \citet{Muller1997}.
\begin{theorem}
    For all $P,Q$ and $\pi\in(0,1)$,
    $
    V_{\aco\mathcal{R},\pi}(P,Q) = V_{\mathcal{R},\pi}(P,Q).
    $
\end{theorem}
\begin{proof} Let $B_1:=\{(\alpha_i)_i\colon\sum_i |\alpha_i|\le 1\}$. Then
    \begin{eqnarray*}
	V_{\aco\mathcal R,\pi}(P,Q) &=& 2\sup_{r\in\aco\mathcal{R}} \pi\EE_P r
	-(1-\pi)\EE_Q r\\
	&=& 2\sup_{(\alpha_i)_i\in B_1}
	\sup_{\{r_i\}_i\subset\mathcal{R}} \pi\EE_P\sum_i\alpha_i r_i
	-(1-\pi)\EE_Q\sum_i \alpha_i r_i\\
	&=& 2\sup_{(\alpha_i)_i\in B_1}
	\sup_{\{r_i\}_i\subset\mathcal{R}} \sum_i \alpha_i \left(\pi\EE_P r_i
	-(1-\pi)\EE_Q r_i\right)\\
	&=& 2\sup_{(\alpha_i)_i\in B_1} \sum_i\alpha_i
	\sup_{r_i\in\mathcal{R}} \left(\pi\EE_P r_i -(1-\pi)\EE_Q r_i\right)\\
	&=& 2\sup_{(\alpha_i)_i\in B_1} \sum_i \alpha_i
	V_{\mathcal{R},\pi}(P,Q)\\
	&=& V_{\mathcal{R},\pi}(P,Q) .
    \end{eqnarray*}
\end{proof}
Combining this theorem with Theorem \ref{theorem:linear-loss-equivalence}
shows that for all $P,Q$ and all $\pi\in(0,1)$,
\begin{equation}
    \label{eq:aco-bayes-lin}
	\minLL_{\aco(\mathcal{R})}^{\mathrm{lin}}(\pi,P,Q)=
	\minLL_{\mathcal{R}}^{\mathrm{lin}}(\pi,P,Q);
\end{equation}
that is, taking the absolute convex hull does not change the Bayes risk when
using linear loss.  
Let ${S}(P,Q,\mathcal{F},\epsilon):=\Pr\{
\minLL_{\mathcal{F}}^{\mathrm{lin}}(\pi,P,Q) -
\minLL_{\mathcal{F}}^{\mathrm{lin}}(\pi,P_{n},Q_{n}) > \epsilon\}$ denote the
probability of being misled by more than $\epsilon$ on a sample of size $n$,
where $P_{n}$ and $Q_{n}$ are
the respective empirical distributions induced by the empirical distribution
$\PP_n$.
Since $P,Q$ are arbitrary in (\ref{eq:aco-bayes-lin}) we conclude that
\(
    {S}(P,Q,\aco(\mathcal{F}),\epsilon)=
    {S}(P,Q,\mathcal{F},\epsilon),
\)
which is hinted at by the Rademacher average upper bounds on sample complexity
and invariance to forming absolute convex hulls \citep{BartlettMendelson2002},
but as far as we are aware has never been stated as above. Note that the use of
linear loss is essential here and it is only well defined for suitable
$\mathcal{F}$. Appendix \ref{sec:appendix-svm-mmd} shows that
the standard SVM can be derived using linear loss.

%%%%%%%%%%%%%%%%%%%%%%%%%%%%%%%%%%%%%%%%%%%%%%%%%%%%%%%%%%%%%%%%%%%%%%%%%%%
\subsection{Variational Representation of $\II_f$ and its Generalizations}

The variational representation of the Variational divergence
(\ref{eq:classic-variational-divergence-def})
suggests the question of whether there is a variational representation for a
general $f$-divergence. This has been considered previously. We briefly
summarise the approach, and then explore some (new) implications of the
representation.

One can obtain a variational representation for $\II_f$ by substituting
a variational representation for $f$ into the definition of $\II_f$
\citep{Keziou2003,Keziou2003a,Broniatowski:2004a,Broniatowski:2005}.
Let $p$ and $q$ denote the densities corresponding to $P$ and $Q$ and assume
for now they exist. Recall from Section~\ref{sub:lfdual} above, that the 
Legendre-Fenchel conjugate of $f$ is given by
\(
\lfdual{f}(s) = \sup_{u\in{\mathrm{Dom}}f} us- f(u)  .
\)
In general $\ran \lfdual{f}=\lfdual{\reals}:=\reals\cup\{+\infty\}$.
Since $f(u)=\sup_{\rho\in\reals} u\rho -\lfdual{f}(\rho)$, we can write
\begin{eqnarray*}
    \II_f(P,Q) & = &\int_{\mathcal{X}} q(x) \sup_{\rho\in\reals} \left(\rho
    \frac{p(x)}{q(x)} -\lfdual{f}(\rho)\right) dx\\
    &=&\sup_{\rho\in\reals^{\mathcal{X}}} \int_{\mathcal{X}} \rho(x)p(x)
    -\lfdual{f}(\rho(x))q(x) dx.\\
    &=& \sup_{\rho\in\reals^{\mathcal{X}}} 
    (\EE_P \rho - \EE_Q \lfdual{f}(\rho)) .
    \label{eq:variational-f-div}
\end{eqnarray*}
We make this concrete by considering the variational
divergence. 
The corresponding $f$ is given by $f(t)=|t-1|$  and (adopting the 
convention that $0\cdot\infty=0$)	
$\lfdual{f}(x)=\test{x\not\in[-1,1]}\infty+\test{x\in[-1,1]} x$.
Since the supremum in (\ref{eq:variational-f-div}) will not be attained if 
the second term is infinite, one can restrict the supremum to be over
$\mathcal{F}=\{\rho\in\reals^{\mathcal{X}}\colon \|\rho\|_\infty\le 1\}$.
Thus
\begin{eqnarray*}
V(P,Q) &=&\sup_{\rho\colon \|\rho\|_\infty\le 1}(\EE_P \rho - \EE_Q \rho)
= \sup_{\rho\in\{-1,1\}^\mathcal{X}}(\EE_P \rho - \EE_Q \rho)\\
&=& \sup_{\rho\in\{0,2\}^\mathcal{X}}(\EE_P \rho - \EE_Q \rho) = 
   2\sup_{\rho\in\{0,1\}^\mathcal{X}}(\EE_P \rho - \EE_Q \rho)
= 2\sup_{A}  |P(A) -Q(A)|
\end{eqnarray*}
since the supremum will be attained for functions  $\rho$ taking on
values only  in $\{-1,1\}$ and the remaining steps are simply a shift and 
rescaling (to $\{0,2\}$ by adding 1, and then to $\{0,1\}$).

The representation (\ref{eq:variational-f-div}) suggests the generalisation 
\begin{eqnarray*}
    \II_{f,\mathcal{F}}(P,Q) &:=& 
     \sup_{\rho\in\mathcal{F}\subseteq\reals^{\mathcal{X}}}
	\int_{\mathcal{X}} \rho(x)p(x) -\lfdual{f}(\rho(x))q(x) dx\\
    &=& \sup_{\rho\in\mathcal{F}} (\EE_P \rho - \EE_Q \lfdual{f}(\rho)) .
\end{eqnarray*}
Observing this is not symmetric in $p$ and $q$ suggests a further
generalisation:
\begin{eqnarray}
    \II_{f,g,\mathcal{F}}(P,Q) &:=&
     \sup_{\rho\in\mathcal{F}\subseteq\reals^{\mathcal{X}}}
    \int_{\mathcal{X}} -\lfdual{g}(\rho(x))p(x) 
			-\lfdual{f}(\rho(x))q(x) dx \label{eq:IfgF}\\
    &=& \sup_{\rho\in\mathcal{F}} (-\EE_P \lfdual{g}(\rho) 
			- \EE_Q \lfdual{f}(\rho)) .
\end{eqnarray}
Here $\lfdual{g}$ is the $\lfdual{\reals}$-valued LF conjugate of 
a convex function $g$. Set
$\II_{f,g}:=\II_{f,g,\reals^{\mathcal{X}}}$.  

An alternative generalisation of $\II_f$  is
\begin{equation}
    \tilde{\II}_{f,g,\mathcal{F}} (P,Q) :=
    \sup_{\rho\in\mathcal{F}} \left(\EE_P \lfdual{g}(\rho) -
    \EE_Q \lfdual{f}(\rho)\right)
\end{equation}
which is identical to (\ref{eq:IfgF}) except for removal of the minus sign
preceeding $\lfdual{g}$. Set
$\tilde{\II}_{f,g}:=\tilde{\II}_{f,g,\reals^{\mathcal{X}}}$. 
If $\rho\in\mathcal{F}$
are such that $\|\rho\|_\infty$ is unbounded, then in general
$\tilde{I}_{f,g,\mathcal{F}} (P,Q)$ will be infinite.  
Properties of the alternative definition relate to the extended infimal 
convolution between two convex functions.
\begin{definition}
    Suppose $f,g\colon\reals^+\rightarrow\reals^*$ are convex.
    The extended infimal convolution is 
    \[
    (f\Box g)(\tau):=\inf_{x\in\reals^+} f(x) +\tau g(x/\tau),\ \ 
    \tau\in\reals^+.  
    \]
\end{definition}
Note that the second term in this convolution is the perspective function
(Section~\ref{sub:perspective}) applied to $g$, that is, $I_g(x, \tau)$.
\begin{theorem}
Suppose $f,g\colon\reals^+\rightarrow\reals^*$ are convex. Then
\begin{enumerate}
    \item $\II_f(P,Q)=\II_{f,\reals^{\mathcal{X}}}(P,Q)$,
	$\tilde{\II}_{f,\mathrm{id},\mathcal{F}}(P,Q)=\II_{f,\mathcal{F}}(P,Q)$,
	and $\II_{t\mapsto |t-1|,\mathcal{F}}(P,Q)=
	2 V_{\mathcal{F},\thalf}(P,Q)$.
	\label{part:easy-relations}
    \item $\tilde{\II}_{f_1,g_1,\mathcal{F}}= \II_{f_2,g_2,\mathcal{F}}$ only 
	if $f_1-f_2=f_a$ and $g_1-g_2=g_a$ and $f_1,f_2,f_a,g_1,g_2,g_a$
        are affine.
	\label{part:affine}
 \item $\II_{f,f,\mathcal{F}} = 
		\II_{\mathrm{id},\mathrm{id},\lfdual{f}(\mathcal{F})}(P,Q)$.
     \label{part:Iff}
 \item $\tilde{\II}_{f,f,\mathcal{F}}= \tilde{\II}_{\mathrm{id},\mathrm{id},
     \lfdual{f}(\mathcal{F})}(P,Q) = 2V_{\lfdual{f}(\mathcal{F})} (P,Q)$.
     \label{part:tilde-Iff}
    \item $\II_{f,g}=\II_{f\Box g}$.
     \label{part:I-f-g-extended-convolution}
\end{enumerate}
\label{theorem:generalised-f-div}
\end{theorem}
\begin{proof}
    Part \ref{part:easy-relations} 
    follows immediately from the various definitions.
Since affine functions are the only functions that are
simultaneously convex and concave, $\tilde{\II}_{f_1,g_1,\mathcal{F}}=
\II_{f_2,g_2,\mathcal{F}}$ only if $f_1,f_2$ (resp.~$g_1,g_2$) are affine and 
their differences are affine (since an affine offset will not change 
$\tilde{\II}$).  This proves part \ref{part:affine}.

We have by change of variables
\begin{equation}
    \tilde{\II}_{f,f,\mathcal{F}}(P,Q)
	=\sup_{\rho\in\mathcal{F}}(\EE_P \lfdual{f}(\rho) -\EE_Q
    \lfdual{f}(\rho))  = \sup_{\psi\in \lfdual{f}(\mathcal{F})} (\EE_P \psi
    -\EE_Q\psi)
    = \tilde{\II}_{\mathrm{id},\mathrm{id},\lfdual{f}(\mathcal{F})} (P,Q),
\end{equation}
where 
$\lfdual{f}(\mathcal{F}):=\{\lfdual{f}\circ\rho\colon\rho\in\mathcal{F}\}$.
(Exactly the same argument applies to $I_{f,f,\mathcal{F}}$ although 
$\sup_{\psi\in \lfdual{g}(\mathcal{F})} (-\EE_P \psi-\EE_Q \psi)$ does not 
correspond to a generalised variational divergence.) This proves
parts \ref{part:Iff} and \ref{part:tilde-Iff}.

\end{proof}
The proof of part \ref{part:I-f-g-extended-convolution}
is in Appendix \ref{section:proof-of-extended-convolution}. 
It suggests the question: given a suitable convex $f$, does there always exist
$g$ such that $f=g\Box g$? This is analogous to the question of spectral
factorisation \citep{SayedKailath2001} for ordinary linear convolution. We do
not know the answer to this question, but have collected a few examples in
Appendix \ref{section:examples} that demonstrates it is certainly true for {\em
some} $f$. There does not appear to be a result analogous to part
\ref{part:I-f-g-extended-convolution} of Theorem
\ref{theorem:generalised-f-div} for $\tilde{\II}_{f,g}$.

We have seen how $f$-divergences are related to integral probability metrics
$V_{\mathcal{F}}$. It turns out that the variational divergence is special in
being both. Many integral probability metrics are true metrics
\citep{Muller1997,Muller1997a}. The only $f$-divergence that is a metric is the
variational divergence. Whether there exist $\mathcal{F}$ such that
$V_{\mathcal{F}}(\cdot,\cdot)$ is not a metric but equals $\II_f(\cdot,\cdot)$
for some $f\ne t\mapsto|t-1|$ (or affine transformation thereof) is left as an
open problem.

We end with another open problem. We have seen how $\minLL_{\mathcal{F}}$ and
$V_{\mathcal{F}}$ are related. This begs the question whether there is a
representation of the form
\[
\II_{f,\mathcal{F}}(P,Q) \stackrel{?}{=} \int_0^1 
\Delta\minLL_{\mathcal{F}}^{0-1}(\pi,P,Q)
\gamma_f(\pi) d\pi.
\]
%%%%%%%%%%%%%%%%%%%%%%%%%%%%%%%%%%%%%%%%%%%%%%%%%%%%%%%%%%%%%%%%%%%%%%
\section{Conclusions}

There are several existing concepts that can be used to quantify the amount of
information in a task and its difficulty: Uncertainty, Bregman information,
statistical information, Bayes risk and regret, and $f$-divergences.
Information is a difference in uncertainty; regret is a difference in risk.  In
the case of supervised binary class probability estimation, we have connected
and extended several existing results in the literature to show how to
translate between these perspectives. The representations allow a precise
answer to the question of what are the primitives for binary experiments.

We have derived the integral representations in a simple and unified manner, and
illustrated the value of the representations.  Along the way we have drawn
connections to a diverse set of concepts related to binary experiments: risk
curves, cost curves, ROC curves and the area under them; variational
representations of $f$-divergences, risks and regrets.

Two key consequences are surrogate regret bounds that are at once more general
and simpler than those in the literature, and a generalisation of the
classical Pinkser inequality providing, \emph{inter alia}, an explicit form
for the best possible Pinsker inequality relating Kullback-Liebler divergence
and Variational divergence.  The parametrisation of regret in terms of weighted
integral representations also shows the connection with matching losses and
provides a simple proof of the convexity of the composite loss induced by
a proper scoring rule with its canonical link function.
We have also presented a
new derivation of support vector machines and their relationship to Maximum
Mean Discrepancy (integral probability metrics).

The key relationships between the basic objects of study are summarised in
Table~\ref{table:summary} and Figure~\ref{figure:big-picture}.
\begin{table}[t]
\begin{center}
    \begin{tabular}{cll}
	\hline
	Given & Assumed & Derived\\
	\hline
	$(P,Q)$ & $f\leftrightarrow\gamma$ &  $\II_f(P,Q)$\\
	\hline
	$(\pi,P,Q)$ & $U\leftrightarrow w, W,\Wb$ & $\JJ(U(\eta))=\SI(\pi,P,Q)$\\
	& & $\minLL(\eta)$\\
	\hline
	   & $\heta$ & $L_w(\eta,\heta)$, $B_w(\eta,\heta)$\\
	   \hline
    \end{tabular}
    \caption{Summary relationships between key objects arising in Binary
    Experiments. ``Given'' indicates the object is given or provided by the
    world; ``Assumed'' is something the user of assumes or imposes in order to
    create a well defined problem; ``Derived'' indicates quantities that are
    derived from the primitives. \label{table:summary}}
\end{center}
\end{table}

All of the results we present demonstrate the fundamental and elementary nature
of the cost-weighted misclassification loss, which is becoming increasingly
appreciated in the Machine Learning literature~\citep{BachHeckermanHorvitz2006,
BeygelzimerLangfordZadrozny2008}.

\begin{sidewaysfigure}
    \begin{center}
	\hspace*{-2.5cm}\includegraphics[width=26.5cm]{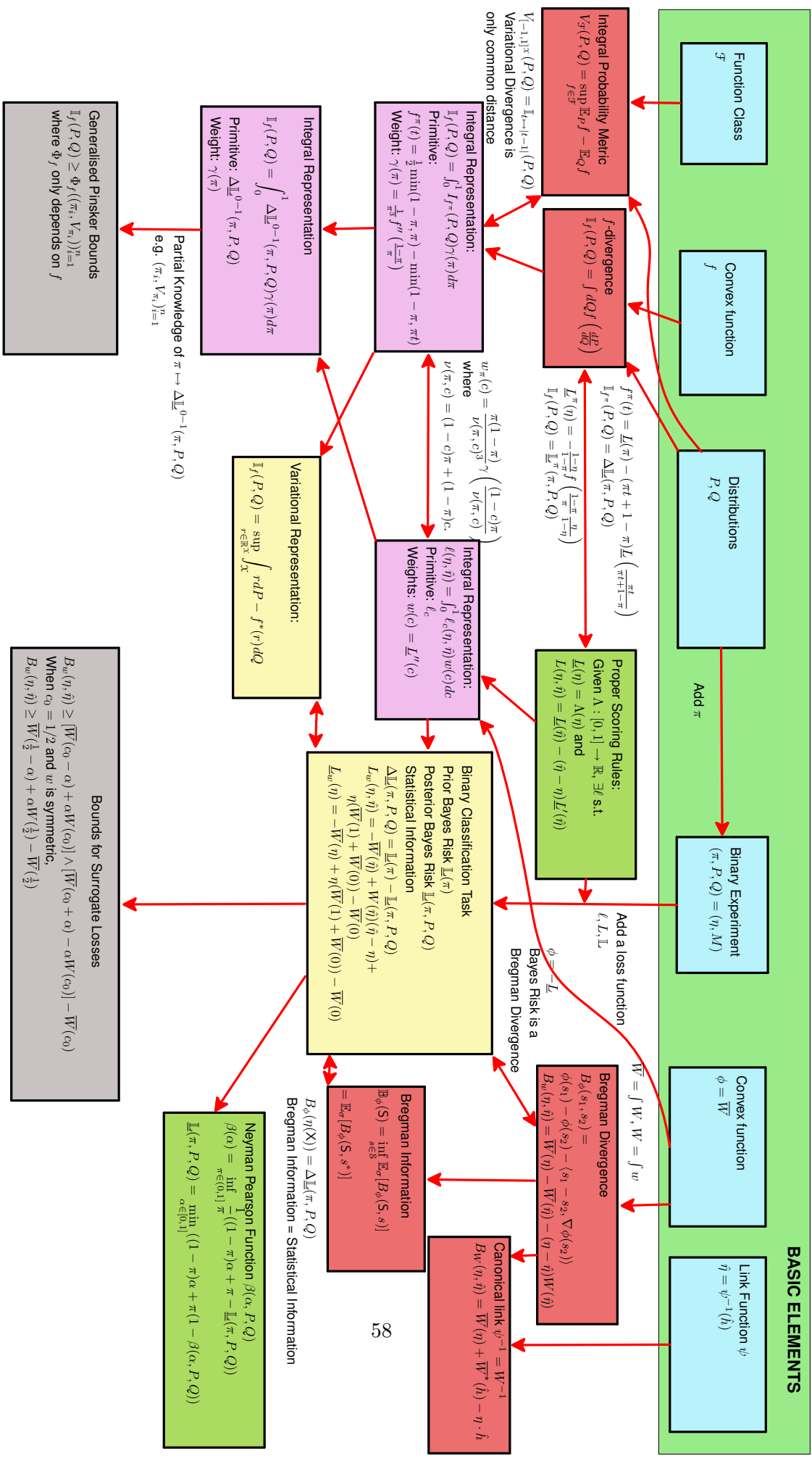}
	\hspace{3cm}\parbox{16cm}{\caption{Diagrammatic summary 
	   of key relationships developed in the paper.  
	   \label{figure:big-picture}}}
    \end{center}
\end{sidewaysfigure}

More generally, the present work is small part of a larger research agenda to
understand the whole field of machine learning in terms of \emph{relations}
between problems. We envisage these relations being richer and more powerful
than the already valuable \emph{reductions} between learning problems.
Much of the present literature on machine learning is highly
solution focussed. Of course one does indeed like to \emph{solve} problems, and
we do not suggest otherwise. But it is hard to see structure in the panoply of
solutions which continue to grow each year. The present paper is a first step
to a pluralistic unification of a diverse set of machine learning problems.
The goal we have in mind can be explained by analogy:

Within the field of
computational complexity (especially NP-completeness):
\cite{GareyJohnson1979,Johnson1982+} lead to a detailed and structured
understanding of the relationships between many fundamental
problems and consequently guides the search for solutions for new problems.
Compare Machine Learning problems with mathematical {\em functions}.  In the
19th century, each function was considered seperately.  Functional Analysis
\citep{Dieudonne1981}  {\em catalogued} them by considering {\em sets} of
functions and {\em relations} (mappings) between them and subsequently
developed many new and powerful tools. The increasing abstraction and focus
on relations has remained a powerful force in mathematics
\citep{GrothendieckRelative}.  A systematic {\em cataloging} (taxonomy)
resonates with Biology's Linnean past --- and taxonomies can indeed lead to
standardisation and efficiency \citep{BowkerStar99}. But taxonomies alone are
inadequate --- it seems necessary to understand the relationships in a manner
analogous to {\em Systems Biology}  which ``is about putting together rather
than taking apart, integration rather than reduction\dots.  Successful
integration at the systems level must be built on successful reduction, but
reduction alone is far from sufficient.'' \citep{Noble2006}.  Finally, Lyell's
{\em Principles of Geology} \citep{Lyell1830}  was a watershed in Geology's
history \citep{bowker2005mps}; prior work is {\em pre}-historical. Lyell's key
insight was to explain the huge diversity of geological formations in terms of
a relative simple set of transformations applied repeatedly. 

These analogies encourage our aspiration that by more systematically
understanding the \emph{relationships} between machine learning problems and
how they can be \emph{transformed} into each other,
we will develop a better organised and more powerful toolkit for solving
existing and future problems, and will make progress along the lines suggested
by \citet{Hand1994}.

\subsubsection*{Acknowledgments}
 This work was motivated in part by discussions with Alekh Agarwal, Arthur
 Gretton, Ulrike von
 Luxburg and Bernhard Sch\"{o}lkopf in 2006 and was supported by the Australian
 Research Council and NICTA. NICTA is supported by the Australian Government
 through Backing Australia's Ability. The authors were also aided by helpful 
 conversations with Suvrit Sra and Bharath Sriperumbudur in T\"{u}bingen
 in 2008.

\appendix
%%%%%%%%%%%%%%%%%%%%%%%%%%%%%%%%%%%%%%%%%%%%%%%%%%%%%%%%%%%%%%%
\section{Proofs}
\label{section:proofs}

\subsection{Proof of Corollary \ref{cor:intrep2}}
\label{sub:proof-of-intrep2}

Integration by parts of $t\phi''(t)$ gives
$\int_0^1 t\,\phi''(t)\,dt = \phi'(1) - (\phi(1) - \phi(0))$
which can be rearranged to give 
\[
	\phi'(1) = \int_0^1 t\,\phi''(t)\,dt + (\phi(1) - \phi(0)).
\]
Substituting this into the Taylor expansion of $\phi(s)$ about 1 yields
\begin{eqnarray*}
	\phi(s) 
	& = & \phi(1) + \phi'(1)(s-1) + \int_s^1 (t-s)\,\phi''(t)\,dt \\
	& = & \phi(1) + \left[
						\int_0^1 t\,\phi''(t)\,dt + (\phi(1) - 
						\phi(0))
					\right](s-1)
				  + \int_0^1 (t-s)_+\,\phi''(t)\,dt \\
	& = & \phi(1) + (\phi(1) - \phi(0))(s-1)
				  + \int_0^1 t(s-1)\,\phi''(t)\,dt 
				  + \int_0^1 (t-s)_+\,\phi''(t)\,dt \\
	& = & \phi(0) + (\phi(1) - \phi(0))s - \int_0^1 \psi(s,t)\,\phi''(t)\,dt,
\end{eqnarray*}
where $\psi(s,t) := \min\{(1-t)s, (1-s)t\}$. This form of $\psi$ is valid since
\begin{eqnarray*}
	-(t(s-1) + (t-s)_+) 
	& = & \begin{cases}
		-ts + t - t + s , & t \geq s \\
		-ts + t 		, & t < s
	\end{cases}\\ 
	& = & \begin{cases}
		s - ts ,& t \geq s \\
		t - ts ,& t < s
	\end{cases} \\
	& = & \min\{(1-t)s, (1-s)t\}
\end{eqnarray*}
as required.

\subsection{Proof of Theorem \ref{thm:jensen-affine}}
\label{sub:proof-jensen-affine}
Expanding the definition of the Jensen gap using the definition of $\psi$ gives
	\begin{eqnarray*}
		\JJ_{\mu}[\psi(\Ssf)] 
		& = & \EE_{\mu}[\psi(\Ssf)] - \psi(\EE_{\mu}[\Ssf]) \\
		& = & \EE_{\mu}[\phi(\Ssf)+b\Ssf+a] 
			- ( \phi(\EE_\mu[\Ssf]) + b \EE_\mu[\Ssf] + a) \\
		& = & \EE_{\mu}[\phi(\Ssf)] + b\EE_\mu[\Ssf] + a
			- \phi(\EE_\mu[\Ssf]) - b \EE_\mu[\Ssf] - a \\
		& = & \JJ_\mu[\phi(\Ssf)]
	\end{eqnarray*}
	as required.

%%%%%%%%%%%%%%%%%%%%%%%%%%%%%%%%%%%%%%%
\subsection{Proof of Theorem~\ref{thm:duality}}
\label{app:infodiv-duality}
\begin{proof}
Given a task $(\pi,P,Q; \ell)$ we need to first check that
\begin{equation}\label{eq:proof-fpi}
	f^{\pi}(t) := \minL(\pi) - (\pi t + 1 - \pi)
	\minL\left(\frac{\pi t}{\pi t+1-\pi}\right)
\end{equation}
is convex and that $f^{\pi}(1) = 0$. This latter fact is obtained immediately by 
substituting $t=1$ into $f^\pi(t)$ yielding $\minL(\pi) - \minL(\pi) = 0$.
The convexity of $f^\pi$ is guaranteed by Theorem~\ref{pro:minrisk_concave},
which shows that $\minL$ is concave and the fact that the perspective 
transform of a convex function is always convex (see 
Section~\ref{sub:perspective}). Thus the function
\[
	t 
	\mapsto I_{-\minL}(\pi t, \pi t + 1 - \pi)
	= - (\pi t + 1 - \pi)
	\minL\left(\frac{\pi t}{\pi t+1-\pi}\right)
\]
is the composition of a convex function and an affine one and therefore convex.

Substituting (\ref{eq:proof-fpi}) into the definition of $f$-divergence in 
(\ref{eq:fdiv1}) yields
\begin{eqnarray*}
\E{Q}{f^{\pi}(dP/dQ)} 
	&=& \E{Q}{\minL(\pi)
		-\left(\pi\frac{dP}{dQ}+1-\pi\right)
		\minL\left(\frac{\pi dP}{\pi dP+(1-\pi)dQ}\right)}\\
	&=& \minL(\pi)-\int_{\Xcal} \minL\left(\pi \frac{dP}{dM}\right) dM
\end{eqnarray*}
since $dM = \pi dP + (1-\pi) dQ$. Recall that $\eta = \pi dP/dM$.
As $\minL(\pi)$ is constant we note that 
$\minL(\pi) = \E{M}{\minL(\pi)} = \minLL(\pi,M)$ and so
\begin{eqnarray*}
	\E{Q}{f^{\pi}(dP/dQ)} 
	&=& \minL(\pi) - \E{M}{\minL(\eta)} \\
	&=& \minLL(\pi,M) - \minLL(\eta,M) \\
	&=& \SI(\eta,M)
\end{eqnarray*}
as required for the forward direction.

Starting with 
\[
	\minL^\pi(\eta) 
	:= -\frac{1-\eta}{1-\pi} f\left(\frac{1-\pi}{\pi}\frac{\eta}{1-\eta}\right)
\]
and substituting into the definition of statistical information in 
(\ref{eq:statinfo}) gives us
\begin{eqnarray*}
	\SI^\pi(\eta,M) 
	&=& \E{M}{\minL^\pi(\pi)} - \E{M}{\minL^\pi(\eta)}\\
	&=& \int_{\Xcal} -\frac{1-\pi}{1-\pi}f(1)\,dM 
		- \int_{\Xcal} 
		-\frac{1-\eta}{1-\pi} 
		f\left(\frac{1-\pi}{\pi}\frac{\eta}{1-\eta}\right)\,dM\\
	&=& 0 + \int_{\Xcal} f\left(\frac{dP}{dQ}\right)\,dQ
\end{eqnarray*}
since $f(1) = 0$, $dQ = (1-\eta)/(1-\pi)dM$ and 
\[
	dP/dQ = \frac{1-\pi}{\pi}\frac{\eta}{1-\eta}
\] 
by the discussion in Section~\ref{sub:generative-discriminative}.
This proves the converse statement of the theorem.
\end{proof}
%%%%%%%%%%%%%%%%%%%%%%%%%%%%%%%%%%%%%%%
\subsection{Proof of part \ref{part:I-f-g-extended-convolution} of
Theorem \ref{theorem:generalised-f-div}}
\label{section:proof-of-extended-convolution}
We need the following lemma.
\begin{lemma}
	Let $f\colon\reals\rightarrow\reals$ and
	$K\colon\reals\times\reals\rightarrow\reals$ be convex and bounded
	from below. Then the extended infimal convolution
	\[
	(f\Box K)(x) = \inf_{y\in\reals} f(y) + K(x,y), \ \ \ x\in\reals
	\]
	is convex in $x\in\reals$.
    \end{lemma}
    Observe that if $K(x,y)=g(x-y)$ for convex $g$, then $f\Box K = f\oplus
    g$, the standard infimal convolution \citep{Hiriart-UrrutyLemarechal1993}. 
    This extended infimal convolution seems
    little studied apart from by \cite{Cepedello-Boiso1998}.
    \begin{proof}
	Let $\tilde{f}(x,y):=f(y)$, $x\in\reals$. Clearly $\tilde{f}$ is convex
	on $\reals\times\reals$. Let $\tilde{h}(x,y)=\tilde{f}(x,y)+K(x,y)$.
	\citet[Proposition 2.1.1]{Hiriart-UrrutyLemarechal1993} show that
	$\tilde{h}$ is convex on $\reals\times\reals$. Observe that
	$
	(f\Box K)(x) = \inf\{\tilde{h}(x,y)\colon y\in\reals\},
	$
	i.e.~the {\em marginal} function of $\tilde{h}$. Since by construction
	$\tilde{h}$ is bounded from below, using the result of 
	\citet[p.169]{Hiriart-UrrutyLemarechal1993} proves the result.
    \end{proof}
    \begin{corollary}
	For any convex $f$ and $g$, $f\Box g$ is convex.
	\label{cor:box-convex}
    \end{corollary}
    \begin{proof}
    Observe that $(f\Box g)(x)=\inf_{y\in\reals^+} f(y) + x g(y/x)=
    \inf_{y\in\reals^+} f(y) + I_g(x,y)$,
    $x\in\reals^+$, where $I_g$ is the perspective function 
    (\ref{eq:perspective}).
    \citet[Proposition 2.2.1]{Hiriart-UrrutyLemarechal1993} show that if
    $g\colon\reals^n\rightarrow\reals$ is convex then the { perspective}
    $I_g$ 
    is convex on $\reals^{n+1}$.  The corollary then follows from the lemma.
    \end{proof}

    \begin{proof} {\bf (part \ref{part:I-f-g-extended-convolution} of
	Theorem \ref{theorem:generalised-f-div})}
    Observe that if $h(x)=t\phi(x)$ then the LF conjugate $h^*(s)=t \phi(s/t)$.
    Thus using the Fenchel duality theorem \citep{Rockafellar:1970} we have, 
    using \cite[Theorem 14.60]{RockafellarWets2004} to justify the swapping the
    order of the supremum and integration,
    \begin{eqnarray*}
    \II_{f,g}(P,Q)   &=&
    \sup_{\rho\in\bar{\reals}^{\mathcal{X}}}
	\int_{\mathcal{X}} -g^\lf(\rho(x))p(x) -f^\lf(\rho(x))q(x) dx\\
	&=& \int_{\mathcal{X}} \sup_{\rho\in\bar{\reals}} -g^\lf(\rho)p(x)
	-f^\lf(\rho)q(x)dx\\
	&=& \int_{\mathcal{X}} \inf_{\rho\in\bar{\reals}}
	f\left(\frac{\rho}{q(x)}\right) + g\left(\frac{\rho}{p(x)}\right) dx\\
	&=& \int_{\mathcal{X}} \inf_{\rho\in\bar{\reals}} q(x)
	f\left(\frac{\rho}{q(x)}\right) +p(x)g\left(\frac{\rho}{p(x)}\right)
	dx,\\
	&=& \int_{\mathcal{X}} i_{f,g}(p,q)(x) dx
\end{eqnarray*}
where
\[
i_{f,g}(p,q)(\cdot):= \inf_{\rho\in\bar{\reals}} q(\cdot) 
f\left(\frac{\rho}{q(\cdot)}\right) +
p(\cdot) g\left(\frac{\rho}{p(\cdot)}\right).
\]
Let
$x:=\frac{\rho}{q}\in\bar{\reals}^+$. Thus $\rho=xq$ and 
\[
i_{f,g}(p,q)=\inf_{x\in\bar{\reals}^+} qf(x)+pg(xq/p).
\]
Let $\tau=\frac{p}{q}\in\bar{\reals}^+$. Thus
\begin{eqnarray}
    i_{f,g}(p,q)(\tau) &=& \inf_{x\in\reals^+} q f(x) + pg(x/\tau)\nonumber\\
    &=& q \left[\inf_{x\in\bar{\reals}^+} f(x) + \tau g(x/\tau)\right]\nonumber\\
    &=& q \cdot (f\Box g)(\tau).\label{eq:ibox}
\end{eqnarray}
Let $h:=f\Box g$. Observe from (\ref{eq:ibox}) that
$ i_{f,g}(p,q) = q h(p/q)$ and thus 
\[
\II_{f,g}(p,q) = \int_{\mathcal{X}} q(x) h\left(\frac{p(x)}{q(x)}\right) dx 
= \II_h(p,q)
\]
if $h$ is convex, which we know to be the case from Corollary
\ref{cor:box-convex}. 
\end{proof}

%%%%%%%%%%%%%%%%%%%%%%%%%%%%%%%%%%%%%%%
\subsection{Pinsker Theorems}
\label{section:pinsker-proofs}
\begin{proof} {\bf (Theorem \ref{theorem:general-pinsker})}
    Given a binary experiment $(P,Q)$ denote the corresponding
    statistical information as
    \begin{equation}
	 \phi(\pi)=\phi_{(P,Q)}(\pi):=\SI^{0-1}(\pi,P,Q)=
	 \pi\wedge(1-\pi)-\psi_{(P,Q)}(\pi),
    \end{equation}
    where $\psi_{(P,Q)}(\pi)=\psi(\pi)=\minLL^{0-1}(\pi,P,Q)$. We
    know that $\psi$ is non-negative and  concave and satisfies 
    $\psi(\pi)\le\pi\wedge(1-\pi)$
    and thus $\psi(0)=\psi(1)=0$.

    Since 
    \begin{equation}
	\II_f(P,Q) = \int_0^1 \phi(\pi) \gamma_f(\pi) d\pi,
	\label{eq:I-f-gamma}
    \end{equation}
    $\II_f(P,Q)$ is minimized by minimizing $\phi_{(P,Q)}$ over all
    $(P,Q)$ such that
    \[
	\phi(\pi_i)=\phi_i=\pi_i\wedge(1-\pi_i)-\psi_{(P,Q)}(\pi_i).
    \]
    Let
    $\psi_i:=\psi(\pi_i)=\frac{1}{2}-\frac{1}{4}V_{\pi_i}(P,Q)$. 
    The problem becomes:
    \begin{eqnarray}
	\mbox{Given\ } (\pi_i,\psi_i)_{i=1}^n
	& & \mbox{find the maximal\ }
    \psi\colon[0,1]\rightarrow[0,\thalf]\ \mbox{such that}
    \label{eq:S-problem}\\
    & &\ \ \psi(\pi_i) = \psi_i, \ \ 
           i=0,\ldots,n+1,\label{eq:psi-i-constraint}\\
	& &\ \ \psi(\pi) \le \pi\wedge (1-\pi), \ \
	\pi\in[0,1],\label{eq:psi-overbound-constraint}\\
	& &\ \  \psi  \mbox{\ is concave}.\label{eq:psi-concave-constraint}
    \end{eqnarray}
    This will tell us the optimal $\phi$ to use since optimising over $\psi$ is
    equivalent to optimizing over $\minLL(\cdot,P,Q)$. Under the additional
    assumption on $\mathcal{X}$, Corollary
    \ref{corollary:all-bayes-risk-curves-possible} implies that for any $\psi$
    satisfying (\ref{eq:psi-i-constraint}),
    (\ref{eq:psi-overbound-constraint})  and (\ref{eq:psi-concave-constraint})
    there exists $P,Q$ such that $\minLL(\cdot,P,Q)=\psi(\cdot)$.

    Let $\Psi$ be the set of piecewise linear concave functions on $[0,1]$
    having $n+1$ segments such that
    $\psi\in\Psi\Rightarrow\psi$ satisfies 
    (\ref{eq:psi-i-constraint}) and (\ref{eq:psi-overbound-constraint}). 
    We now show
    that in order to solve (\ref{eq:S-problem}) it suffices to consider
    $\psi\in\Psi$.

    \begin{figure}[t]
	\begin{center}
	\includegraphics[width=0.98\textwidth]{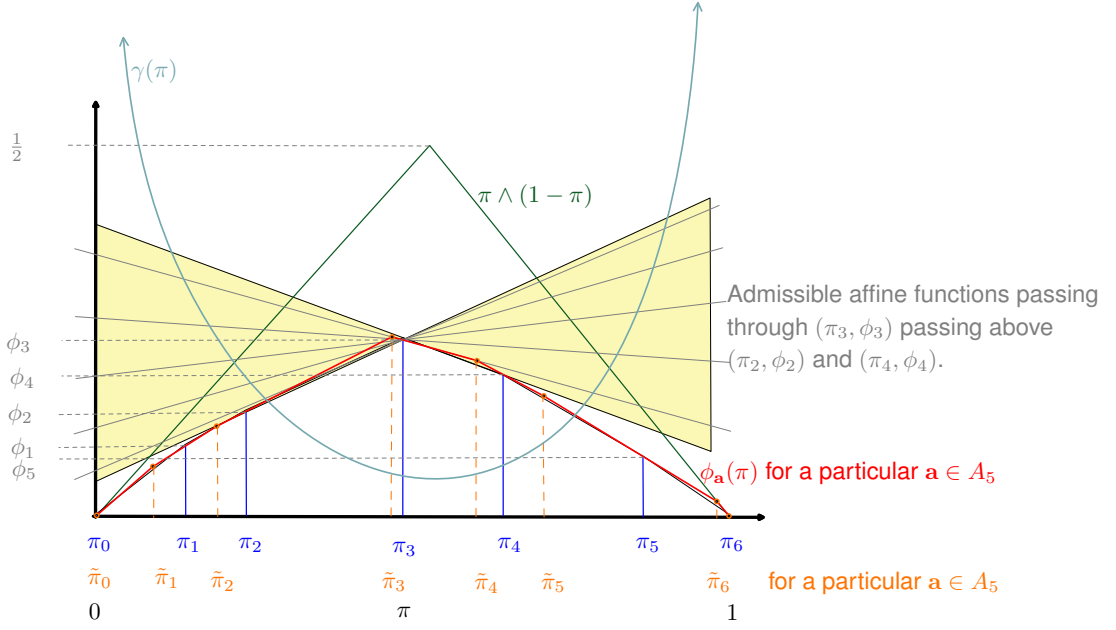}
    \end{center}
	\caption{Illustration of construction of optimal
	$\phi(\pi)=\minLL(\pi,P,Q)$. The optimal $\phi$ is piece-wise linear
	such that $\phi(\pi_i)=\phi_i$, $i=0,\ldots,n+1$. }
    \end{figure}

    If $g$ is a concave function on $\reals$, then
    \[
    \eth g(x):= \{s\in\reals\colon g(y) \le g(x) + \langle s, y-x\rangle, \
    y\in\reals\}
    \]
    denote the {\em sup-differential} of $g$ at $x$. (This is 
    the obvious analogue of the \emph{sub}-differential for convex functions
    \citep{Rockafellar:1970}.) 
    Suppose $\tilde{\psi}$ is a general concave function satisfying
    (\ref{eq:psi-i-constraint}) and (\ref{eq:psi-overbound-constraint}).
    For $i=1,\ldots,n$, let
    \[
    G_i^{\tilde{\psi}}:=\left\{
    [0,1]\ni g_i^{\tilde{\psi}}:\pi_i\mapsto \psi_i\in\reals  \ \ 
    \mbox{is linear and}\ \
    \textstyle\left.\frac{\partial}{\partial\pi}g_i^{\tilde{\psi}}(\pi)
    \right|_{\pi=\pi_i} \in\eth\tilde{\psi}(\pi_i)\right\}.
    \]
    Observe that by concavity, for all concave $\tilde{\psi}$ satisfying 
    (\ref{eq:psi-i-constraint}) and (\ref{eq:psi-overbound-constraint}), for
    all $g\in\bigcup_{i=1}^n G_i^{\tilde{\psi}}$, $g(\pi)\ge\psi(\pi)$,
    $\pi\in[0,1]$. 
    
    Thus given any such $\tilde{\psi}$, one can always
    construct
    \begin{equation}
	\psi^*(\pi)=\min(g_1^{\tilde{\psi}}
	(\pi),\ldots,g_n^{\tilde{\psi}}(\pi))
	\label{eq:psi-general}
    \end{equation}
    such that $\psi^*$ is concave,  satisfies
    (\ref{eq:psi-i-constraint})   and
    $\psi^*(\pi)\ge\tilde{\psi}(\pi)$, for all $\pi\in[0,1]$. It remains to
    take account of (\ref{eq:psi-overbound-constraint}). That is trivially done
    by setting
    \begin{equation}
	 \psi(\pi) = \min (\psi^*(\pi),\pi\wedge(1-\pi))
	 \label{eq:psi-general-2}
    \end{equation}
    which remains concave and piecewise linear (although with potentially one
    additional linear segment). Finally, the pointwise smallest concave 
    $\psi$
    satisfying 
    (\ref{eq:psi-i-constraint}) and (\ref{eq:psi-overbound-constraint}) is the
    piecewise linear function connecting the points 
    $(0,0), (\pi_1,\psi_1),\ldots, (\pi_m,\psi_m), (1,0)$.

    Let $g\colon [0,1]\rightarrow[0,\frac{1}{2}]$ be this function which can be written
    explicitly as
    \[
    g(\pi) =\left(\psi_i +\frac{(\psi_{i+1}-\psi)(\pi-\pi_i)}{\pi_{i+1}-\pi_i}
    \right)\cdot \test{\pi\in[\pi_i,\pi_{i+1}]},\ \ \ i=0,\ldots,n,
    \]
    where we have defined $\pi_0:=0$, $\psi_0:=0$, $\pi_{n+1}:=1$ and
    $\psi_{n+1}:=0$. 
    
    We now explicitly parametrize  this family of functions.
    Let $p_i\colon[0,1]\rightarrow\reals$ denote the
    affine segment the graph of which passes through 
    $(\pi_i,\psi_i)$, $i=0,\ldots,n+1$.
    Write $p_i(\pi)=a_i\pi +b_i$. We know that $p_i(\pi_i)=\psi_i$ and thus
    \begin{equation}
	 b_i=\psi_i -a_i\pi_i, \ \ \ i=0,\ldots,n+1 .
    \end{equation}
    In order to determine the constraints on $a_i$, since $g$ is concave and
    minorizes $\psi$, it suffices to only consider $(\pi_{i-1}, g(\pi_{i-1}))$
    and $(\pi_{i+1},g(\pi_{i+1}))$ for $i=1,\ldots,n$.  We have (for
    $i=1,\ldots,n$)
    \begin{eqnarray}
	  & p_i(\pi_{i-1}) &\ge g(\pi_{i-1})\nonumber\\
	 \Rightarrow & a_i \pi_{i-1}+b_i &\ge \psi_{i-1}\nonumber\\
	 \Rightarrow & a_i \pi_{i-1}+\psi_i-a_i\pi_i &\ge \psi_{i-1}\nonumber\\
	 \Rightarrow & a_i \underbrace{(\pi_{i-1}-\pi_i)}_{< 0}& \ge \psi_{i-1}
	 -\psi_i\nonumber\\[-3mm]
	 \Rightarrow & a_i &\le \frac{\psi_{i-1}-\psi_i}{\pi_{i-1}-\pi_i} .
	 \label{eq:ai-upperbound}
    \end{eqnarray}
    Similarly we have (for $i=1,\ldots,n$)
    \begin{eqnarray}
	 & p_i(\pi_{i+1}) &\ge g(\pi_{i+1})\nonumber\\
	 \Rightarrow & a_i\pi_{i+1}+b_i &\ge\psi_{i+1}\nonumber\\
	 \Rightarrow & a_i\pi_{i+1}+\psi_i-a_i\pi_i &\ge \psi_{i+1}\nonumber\\
	 \Rightarrow & a_i\underbrace{(\pi_{i+1}-\pi_i)}_{>0} &\ge
	 \psi_{i+1}-\psi_i\nonumber\\[-3mm]
	 \Rightarrow & a_i & \ge
	 \frac{\psi_{i+1}-\psi_i}{\pi_{i+1}-\pi_i} .\label{eq:ai-lowerbound}
    \end{eqnarray}
    We now determine the points at which $\psi$ defined by
    (\ref{eq:psi-general}) and (\ref{eq:psi-general-2}) change slope. That
    occurs at the points $\pi$ when
    \begin{eqnarray*}
      & p_i(\pi) &=p_{i+1}(\pi)\\
      \Rightarrow & a_i\pi+\psi_i-a_i\pi_i
      &=a_{i+1}\pi+\psi_{i+1}-a_{i+1}\pi_{i+1}\\
      \Rightarrow & (a_{i+1}-a_i)\pi &=
      \psi_i-\psi_{i+1}+a_{i+1}\pi_{i+1}-a_i\pi_i\\
      \Rightarrow &\pi &=
      \frac{\psi_i-\psi_{i+1}+a_{i+1}\pi_{i+1}}{a_{i+1}-a_i} \\
      &  & =:\tilde{\pi}_i
\end{eqnarray*}
    for $i=0,\ldots,n$. Thus
    \[
    \psi(\pi)=p_i(\pi),\ \ \ \pi\in[\tilde{\pi}_{i-1},\tilde{\pi}_i], \
    i=1,\ldots,n.
    \]
    Let $\abold=(a_1,\ldots,a_n)$. We explicitly denote the dependence of
    $\psi$ on $\abold$ by writing $\psi_{\abold}$.
    Let
    \begin{eqnarray*}
    \phi_{\abold}(\pi)&:=&\pi\wedge(1-\pi) -\psi_{\abold}(\pi)\\
    &=&\alpha_{\abold,i}\pi+\beta_{\abold,i}, \ \ \
    \pi\in[\bar{\pi}_{i-1},\bar{\pi}_{i}], \ i=1,\ldots,n+1,
    \end{eqnarray*}
    where $\abold\in A_n$ (see (\ref{eq:A-n})), $\bar{\pi}_i$, 
    $\alpha_{\abold,i}$ and
    $\beta_{\abold,i}$ are defined by (\ref{eq:bar-pi}), 
    (\ref{eq:alpha-abold}) and
    (\ref{eq:beta-abold}) respectively. The extra segment induced at index $j$
    (see (\ref{eq:j-def})) is needed since $\pi\mapsto\pi\wedge(1-\pi)$ has a
    slope change at $\pi=\thalf$. Thus in general, $\phi_{\abold}$ is
    piecewise linear with $n+2$ segments (recall $i$ ranges from $0$ to $n+2$);
    if $\tilde{\pi}_{k+1}=\frac{1}{2}$ for some $k\in\{1,\ldots,n\}$, then there will
    be only $n+1$ non-trivial segments.

    Thus  
    \[
    \left\{\pi\mapsto\sum_{i=0}^{n} \phi_{\abold}(\pi)
    \cdot\test{\pi\in[\bar{\pi}_{i},
    \bar{\pi}_{i+1}]}\colon \abold\in A_n\right\}
    \]
    is the set of $\phi$ consistent with the constraints and $A_n$ is defined
    in (\ref{eq:A-n}).
    Thus substituting into (\ref{eq:I-f-gamma}), interchanging the order of
    summation and integration and optimizing we have
    shown~(\ref{eq:general-pinsker}). The tightness has already been argued: 
    under the additional assumption on $\mathcal{X}$, since there is no
    slop in the argument above since every $\phi$ satisfying the
    constraints is the Bayes risk function for some $(P,Q)$.
\end{proof}
\begin{proof} {\bf (Theorem \ref{theorem:special-cases-pinsker})}
    In this case $n=1$  and the optimal $\psi$ function will be piecewise
    linear, concave, and its graph will 
    pass through $(\pi_1,\psi_1)$. Thus the optimal $\phi$
    will be of the form
    \[
    \phi(\pi) = \left\{
    \begin{array}{ll}
	 0, & \pi\in [0,L]\cup [U,1]\\
	 \pi -(a\pi+b), \ & \pi\in [L,\thalf]\\
	 (1-\pi)-(a\pi+b), & \pi\in[\thalf,U].
    \end{array}
    \right.
    \]
    where $a\pi_1+b=\psi_1\Rightarrow b=\psi_1-a\pi_1$ and $a\in
    [-2\psi_1,2\psi_1]$ (see Figure \ref{figure:SimplePinsker}).
    \begin{figure}[t]
	\begin{center}
	\includegraphics[width=0.7\textwidth]{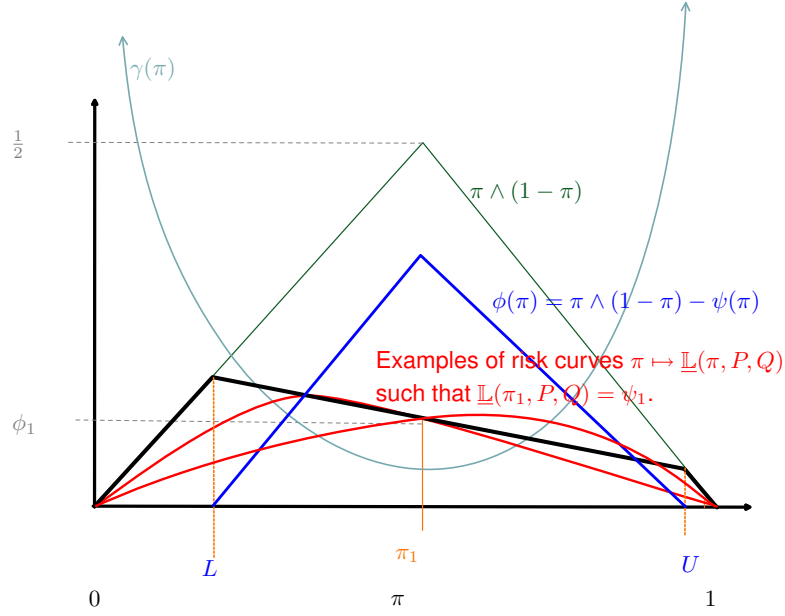}
	\end{center}
	\caption{The optimisation problem when $n=1$. Given $\phi_1$, there are
	many risk curves consistent with it. The optimisation problem involves
	finding the risk curve that maximises $\II_f$. $L$ and $U$ are defined
	in the text.
	\label{figure:SimplePinsker}}
    \end{figure}
    For variational divergence, $\pi_1=\thalf$ and thus
    \begin{equation}
	\psi_1=\pi_1\wedge(1-\pi_1) - \frac{V}{4}=\frac{1}{2}-\frac{V}{4}
	\label{eq:psi-1-variational}
    \end{equation}
    and so $\phi_1=V/4$. We can thus determine $L$ and $U$: 
    \begin{eqnarray*}
	 & aL +b & = L\\
	 \Rightarrow & aL+\psi_1-a\pi_1& =L\\
	 \Rightarrow  & L & = \frac{a\pi_1 -\psi_1}{a-1}.
    \end{eqnarray*}
    Similarly $ aU+b =1-U \Rightarrow U =\frac{1-\psi_1+a\pi_1}{a+1}$
    and thus
    \begin{equation}
	\II_f(P,Q) \ge\!\!\! \min_{a\in[-2\psi_1,2\psi_1]}\!\!\!
	\int\limits_{\frac{a\pi_1-\psi_1}{a-1}}^{\frac{1}{2}}\!\!\!\!\!
	[(1-a)\pi-\psi_1+a\pi_1]
	\gamma_f(\pi) d\pi +\!\!\!\!
	\int\limits_{\frac{1}{2}}^{\frac{1-\psi_1+a\pi_1}{a+1}}\!\!\!\!\!
	[(-a-1)\pi-\psi_1+a\pi_1+1]\gamma_f(\pi)d\pi.
	\label{eq:general-I-f-in-terms-of-psi-1}
    \end{equation}
If $\gamma_f$ is symmetric about $\pi=\thalf$ (so by
Corollary~\ref{corollary:symmetric-f-div} $\II_f$ is symmetric) and convex and
$\pi_1=\thalf$, then the optimal $a=0$.  Thus in that case, 
\begin{eqnarray}
    \II_f(P,Q) &\ge& 2 \int_{\psi_1}^{\frac{1}{2}} 
      (\pi-\psi_1) \gamma_f(\pi) d\pi \\
      &=& 2\left[(\thalf-\psi_1)\Gamma_f(\thalf) + \bar{\Gamma}_f(\psi_1)
      -\bar{\Gamma}_f(\thalf)\right]\nonumber\\
      &=& 2 \left[\textstyle\frac{V}{4} \Gamma_f(\thalf)
      +\bar{\Gamma}_f \left(\thalf-\textstyle\frac{V}{4}\right)
      -\bar{\Gamma}_f(\thalf)\right] . \label{eq:simpler-pinsker-Gamma}
\end{eqnarray}
Combining the above with (\ref{eq:psi-1-variational}) leads to a range of
Pinsker style bounds for symmetric $\II_f$:

\begin{description}
    \item[Jeffrey's Divergence] Since $J(P,Q)=\KL(P,Q)+\KL(Q,P)$ we have
	$\gamma(\pi)=\frac{1}{\pi^2 (1-\pi)^2}
	+\frac{1}{\pi(1-\pi)^2}=\frac{1}{\pi^2(1-\pi)^2}$. (As a check,
	$f(t)=(t-1)\ln(t)$, $f''(t)=\frac{t+1}{t^2}$ and so
	$\gamma_f(\pi)=\frac{1}{\pi^3}f'\left(\frac{1-\pi}{\pi}\right)=
	\frac{1}{\pi^2 (1-\pi)^2}$.)
	Thus
	\begin{eqnarray*}
	    J(P,Q) &\ge & 2 \int_{\psi_1}^{1/2} \frac{(\pi-\psi_1)}{\pi^2
	    (1-\pi)^2} d\pi\\
	    &=& (4\psi_1 -2)(\ln(\psi_1)-\ln(1-\psi_1)).
	\end{eqnarray*}
	Substituting $\psi_1=\frac{1}{2}-\frac{V}{4}$ gives
	\[
	    J(P,Q) \ge 
	    V\ln\left(\frac{2+V}{2-V}\right).
	\]

	Observe that the above bound behaves like $V^2$ for small $V$, and 
	$V\ln\left(\frac{2+V}{2-V}\right)\ge V^2$ for $V\in[0,2]$.
	Using
	the traditional Pinkser inequality ($\KL(P,Q)\ge V^2/2$) we have
	\begin{eqnarray*}
	    J(P,Q)&=& \KL(P,Q)+\KL(Q,P)\\
	    &  \ge & \frac{V^2}{2}+\frac{V^2}{2}\\
	    &=& V^2
	\end{eqnarray*}
    %%%%%%%%%%%%%%%%%%%%%%%%%
    \item[Jensen-Shannon Divergence] Here
	$f(t)=\frac{t}{2}\ln t - 
	     \frac{(t+1)}{2}\ln(t+1) +\ln 2$ 
	 and thus
	$\gamma_f(\pi)=\frac{1}{\pi^3}
	f''\left(\frac{1-\pi}{\pi}\right)=\frac{1}{2\pi(1-\pi)}$. Thus
	\begin{eqnarray*}
	    JS(P,Q) &=& 2 \int_{\psi_1}^{\frac{1}{2}}
	    \frac{\pi-\psi_1}{2\pi(1-\pi)}d\pi\\
	    &=& \ln(1-\psi_1)-\psi_1\ln(1-\psi_1)+\psi_1\ln\psi_1+\ln(2).
	\end{eqnarray*}
	Substituting $\psi_1=\frac{1}{2}-\frac{V}{4}$ leads to
	\[
	JS(P,Q) \ge \left(\frac{1}{2}-\frac{V}{4}\right)\ln(2-V) +
	\left(\frac{1}{2}+\frac{V}{4}\right)\ln(2+V) -\ln(2).
	\]
%%%%%%%%%%%%%%%%%%%
\item[Hellinger Divergence] Here $f(t)=(\sqrt{t}-1)^2$. Consequently
    $\gamma_f(\pi) = \frac{1}{\pi^3} f''\left(\frac{1-\pi}{\pi}\right)=
    \frac{1}{\pi^3} \frac{1}{2 \left( (1-\pi)/
    \pi\right)^{3/2}}=\frac{1}{2[\pi(1-\pi)]^{3/2}}$
    and thus
    \begin{eqnarray*}
	h^2(P,Q) &\ge & 2\int_{\psi_1}^{\frac{1}{2}}
	\frac{\pi-\psi_1}{2[\pi(1-\pi)]^{3/2}} d\pi\\
	&=& \frac{4 \sqrt{\psi_1} (\psi_1-1)
	+2\sqrt{1-\psi_1}}{\sqrt{1-\psi_1}}\\
	&=&
	\frac{4\sqrt{\frac{1}{2}-\frac{V}{4}}\left(\frac{1}{2}-\frac{V}{4}-1\right)+2\sqrt{1-\frac{1}{2}+\frac{V}{4}}}{\sqrt{1-\frac{1}{2}+\frac{V}{4}}}\\
	&=&2- \frac{(2+V)\sqrt{2-V}}{\sqrt{2+V}}\\
	&=& 2-\sqrt{4-V^2}.
    \end{eqnarray*}
    For small $V$, $2-\sqrt{4-V^2}\approx V^2/4$.
%%%%%%%%%%%%%%%%%%%%%%%%%%
\item[Arithmetic-Geometric Mean Divergence] Here $f(t)=\frac{t+1}{2}
    \ln\left(\frac{t+1}{2\sqrt{t}}\right)$. Thus $f''(t)=\frac{t^2+1}{4t^2
    (t+1)}$ and hence $\gamma_f(\pi)=\frac{1}{\pi^3}
    f''\left(\frac{1-\pi}{\pi}\right) = 
    \gamma_f(\pi)=\frac{2\pi^2-2\pi+1}{\pi^2(\pi-1)^2}$
    and thus
    \begin{eqnarray*}
	T(P,Q) &\ge & 2\int_{\psi_1}^{\frac{1}{2}} (\pi-\psi_1)
    \frac{2\pi^2-2\pi+1}{\pi^2(\pi-1)^2} d\pi\\
    &=& -\frac{1}{2}\ln(1-\psi)-\frac{1}{2}\ln(\psi)-\ln(2).
    \end{eqnarray*}
    Substituting $\psi_1=\frac{1}{2}-\frac{V}{4}$ gives
    \begin{eqnarray*}
    T(P,Q) &\ge &  -\frac{1}{2}\ln\left(\frac{1}{2}+\frac{V}{4}\right)
    -\frac{1}{2}\ln\left(\frac{1}{2}-\frac{V}{4}\right)-\ln(2)\\
    & = & \ln\left(\frac{4}{\sqrt{4-V^2}}\right)-\ln(2) .
\end{eqnarray*}
%%%%%%%%%%%%%%%%%%%%%%%%%%
\item[Symmetric $\chi^2$-Divergence] Here $\Psi(P,Q)=\chi^2(P,Q)+\chi^2(Q,P)$
    and thus (see below) $\gamma_f(\pi)=\frac{2}{\pi^3}+\frac{2}{(1-\pi)^3}$.
    (As a check, from $f(t)=\frac{(t-1)^2(t+1)}{t}$ we have
    $f''(t)=\frac{2(t^3+1)}{t^3}$ and thus
    $\gamma_f(\pi)=\frac{1}{\pi^3}f''\left(\frac{1-\pi}{\pi}\right)$ gives the
    same result.)
    \begin{eqnarray*}
	\Psi(P,Q) &\ge& 2\int_{\psi_1}^{\frac{1}{2}}
	(\pi-\psi_1)\left(\frac{2}{\pi^3}+\frac{2}{(1-\pi)^3}\right) d\pi\\
	&=& \frac{2(1+4\psi_1^2-4\psi_1)}{\psi_1(\psi_1-1)}.
    \end{eqnarray*}
    Substituting $\psi_1=\frac{1}{2}-\frac{V}{4}$ gives $\Psi(P,Q)\ge
    \frac{8V^2}{4-V^2}$.
%%%%%%%%%%%%%%%%%%%%%%%%%%
\end{description}
When $\gamma_f$ is not symmetric, one needs to use
(\ref{eq:general-I-f-in-terms-of-psi-1}) instead of the simpler
(\ref{eq:simpler-pinsker-Gamma}).  We consider two special cases.
\begin{description}
\item[$\chi^2$-Divergence] Here $f(t)=(t-1)^2$ and so $f''(t)=2$ and hence
    $\gamma(\pi)=f''\left(\frac{1-\pi}{\pi}\right)/\pi^3 = \frac{2}{\pi^3}$ which is not
    symmetric.
	Upon substituting $2/\pi^3$ for $\gamma(\pi)$ in
	(\ref{eq:general-I-f-in-terms-of-psi-1}) and evaluating the integrals
	we obtain
	\[
	\chi^2(P,Q) \ge 2 \min_{a\in[-2\psi_1,2\psi_1]} 
	\underbrace{\frac{1+4\psi_1^2-4\psi_1}{2\psi_1-a} -
	\frac{1+4\psi_1^2-4\psi_1}{2\psi_1-a-2}}_{=:J(a,\psi_1)} .
	\]
	One can then solve $\frac{\partial}{\partial a} J(a,\psi_1)=0$ for $a$
	and one obtains $a^*=2\psi_1-1$. Now $a^*> -2\psi_1$ only if
	$\psi_1>\frac{1}{4}$. One can check that when $\psi_1\le\frac{1}{4}$,
	then $a\mapsto J(a,\psi_1)$ is monotonically increasing for
	$a\in[-2\psi_1,2\psi_1]$ and hence the minimum occurs at
	$a^{*}=-2\psi_1$. Thus the value of $a$ minimising
	$J(a,\psi_1)$ is
	\[
	a^*= \test{\psi_1>1/4} (2\psi_1-1) + \test{\psi_1\le 1/4}(-2\psi_1).
	\]
	Substituting the optimal value of $a^*$ into $J(a,\psi_1)$ we obtain
	\[
	J(a^*,\psi_1)= \test{\psi_1>1/4}(2+8\psi_1^2-8\psi_1)    + 
	\test{\psi_1\le 1/4}\left(\frac{1+4\psi_1^2-4\psi}{4\psi} -
	\frac{1+4\psi_1^2-4\psi}{4\psi_1-2}\right) .
	\]
	Substituting $\psi_1=\frac{1}{2}-\frac{V}{4}$ and observing that $V<1
	\Rightarrow \psi_1>1/4$ we obtain
	\[
	\chi^2(P,Q) \ge \test{V<1} V^2 + \test{V\ge 1}\frac{V}{(2-V)}.
	\]
	Observe that the bound diverges to $\infty$ as $V\rightarrow 2$.

%%%%%%%%%%%%%%%%%%%%%%%%%%
\item[Kullback-Leibler Divergence] In this case 
    we have $f(t)=t\ln t$ and thus $f''(t)=1/t$ and 
    $\gamma_f(\pi)=\frac{1}{\pi^3}f''\left(\frac{1-\pi}{\pi}\right)=
    \frac{1}{\pi^2(1-\pi)}$ which is clearly not symmetric. From  
	(\ref{eq:general-I-f-in-terms-of-psi-1}) we obtain
	\[
	\KL(P,Q) \ge \min_{[-2\psi_1,2\psi_1]}
	\left(1-\frac{a}{2}-\psi_1\right)
	\ln\left(\frac{a+2\psi_1-2}{a-2\psi_1}\right) +
	\left(\frac{a}{2}+\psi_1\right)\ln\left(\frac{a+2\psi_1}{a-2\psi_1+2}\right).
	\]
	Substituting $\psi_1=\frac{1}{2}-\frac{V}{4}$ gives
	$
	    \KL(P,Q) \ge  \min_{a\in\left[\frac{V-2}{2},\frac{2-V}{2}\right]}
	    \delta_a (V),
	$
	where $\delta_a(V)=\left(\frac{V+2-2a}{4}\right)
	\ln\left(\frac{2a-2-V}{2a-2+V}\right)+
	\left(\frac{2a+2-V}{4}\right)\ln\left(\frac{2a+2-V}{2a+2+V}\right)$. 
	Set $\beta:=2a$ and we have (\ref{eq:my-pinsker-KL}).
\end{description}
\end{proof}
%%%%%%%%%%%%%%%%%%%%%%%%%%%%%%%%%%%%%%%%%%%%%%%%%%%%%%%%%%%%%%%%%%%%%%%%%%%%%

\section{Summary of Previous ``Monistic'' Approaches to Unification}
\label{section:monistic}
There are are range of different approaches to unifying machine learning from a
monistic perspective:

\emph{Low level data interchange:} There is a small amount of work on
    developing  standards for interchanging data sets
    \citep{GrossmanHornickMeyer2002, CareyMarGentleman2007,wettschereck2001edm}
    --- this is analogous to PDDL
    \citep{GhallabHoweKnoblockMcDermottRamVelosoWeldWilkins1998}.  There are
    also some limited higher level attempts such as 
    ontologies \citep{SoldatovaKing2006}
    and general frameworks \citep{FayyadSmyth1996}.

\emph{Modelling frameworks:}  To {\em solve} a machine learning
    problem, one needs models. There is a rich literature on
    graphical models\cite{Jordan1999},  factor 
    graphs \citep{KschischangFreyLoeliger2001} and Markov logic networks 
    \citep{Domingos:2004,RichardsonDomingos2006} which have allowed the 
    unification
    of sets of problems \citep{WorthenStark2001}, with a focus on
    the modelling and computational techniques for particular
    problems.

\emph{Comparison of frameworks:} There are several philosophical 
       frameworks/approaches to designing inference and learning algorithms.
       There are several works \citep{Barnett1999,BayarriBerger2004,Berger2003}
       that compare and contrast these. They are effectively comparing
       different monistic frameworks, not comparing problems.

       \emph{Overarching frameworks:} These include Bayesian 
       \citep{Robert1994}, information-theoretic
    \citep{Jenssen:2005,HarPhD93},
    game-theoretic \citep{VovkGammermanShafer2005, GrunwaldDawid2004},
    MDL \citep{Grunwald2007,Rissanen2007}, 
    regularised distance
    minimisation \citep{BorweinLewis1991,AltunSmola2006,
    Broniatowski:2004a}, and more 
    narrowly focussed
    ``unifying frameworks'' such as information
    geometry \citep{Dawid2007,Eguchi:2005}, exponential 
    families \citep{CanuSmola2006} and the information
    bottleneck \citep{TishbyPereiraBialek2000}.

%%%%%%%%%%%%%%%%%%%%%%%%%%%%%%%%%%%%%
\section{Examples and Prior Work on Surrogate Loss Bounds}
\label{section:surrogate}

Surrogate loss bounds have garnered increasing interest in the machine learning
community \citep{Zhang2004, BartlettJordanMcAuliffe2006, Steinwart:2007,
Steinwart2008}.  \citet[Chapter 3]{Steinwart2008} have presented a good summary
of recent work.

All of the recent work has been in terms of \emph{margin losses} of the form
\[
    L^\phi(\eta,\hh)=\eta\phi(\hh)+(1-\eta)\phi(-\hh).
\]
As \cite{Buja:2005} discuss, such margin losses can not capture the richness
of all possible proper scoring rules. 
\citet{BartlettJordanMcAuliffe2006} prove that for any $\hh$
\[
    \psi\left(L^{0-1}(\eta,\hh)-\minL^{0-1}(\eta)\right) \le  L^\phi(\eta,\hh)
    -\minL^\phi(\eta),
\]
where $\psi=\tilde{\psi}^{\lf\lf}$ is the LF biconjugate of $\tilde{\psi}$, 
\[
    \tilde{\psi}(\theta) = H^-\left(\frac{1+\theta}{2}\right) -
    H\left(\frac{1+\theta}{2}\right),
\]
$H(\eta)=\minL^\phi(\eta)$ and 
\[
    H^-(\eta)=\inf_{\alpha\colon \alpha(2\eta-1)\le 0}
    (\eta\phi(\alpha)+(1-\eta)\phi(-\alpha))
\]
is the optimal conditional risk under the constraint that the sign of the
argument $\alpha$ disagrees with $2\eta-1$.

We will consider two examples presented by \cite{BartlettJordanMcAuliffe2006}
and show that the bounds we obtain with the above theorem match the results we
obtain with Theorem~\ref{theorem:surrogate-loss}.
\begin{description}
    \item[Exponential Loss] Consider the  link
	$\hh=\psi(\heta)=\thalf\log\textstyle\frac{\heta}{1-\heta}$ with
	corresponding inverse link $\heta=\frac{1}{1+e^{-2\hh}}$.
	\citet{Buja:2005} showed that this link function combined with
	exponential margin loss $\phi(\gamma)=e^{-\gamma}$ results in a proper
	scoring rule
	\[
	L(\eta,\heta)=\eta\left(\frac{1-\heta}{\heta}\right)^{\frac{1}{2}} +
	(1-\eta) \left(\frac{\heta}{1-\heta}\right)^{\frac{1}{2}} .
	\]
	From (\ref{eq:w-in-terms-of-L}) we obtain
	\[
	w(\eta)= \frac{1}{2[\eta(1-\eta)]^{\frac{3}{2}}}.
	\]
	(Note \citet{Buja:2005} have missed the factor of $\thalf$.)
	Thus $W(\eta)=\frac{2\eta-1}{\sqrt{\eta(1-\eta)}}$ and
	$\Wb(\eta)=-2\sqrt{\eta(1-\eta)}$. Hence from 
	(\ref{eq:L-bar-general-form})  we obtain
	\begin{equation}
	      \minL(\eta)= 2\sqrt{\eta(1-\eta)} 
	      \label{eq:exp-loss-bayes-risk}
	\end{equation}
	and from (\ref{eq:B-w-bound-c-1-2})  we obtain that if
	$B_{\frac{1}{2}}(\eta,\heta)=\alpha$ then
	\begin{equation}
	    B(\eta,\heta)\ge 1-\sqrt{1-4\alpha^2} .
	    \label{eq:exp-loss-regret-bound}
	\end{equation}
	  Equations \ref{eq:exp-loss-bayes-risk} and
	    \ref{eq:exp-loss-regret-bound} match the results presented by 
	    \citet{BartlettJordanMcAuliffe2006}	upon noting that
	    $B_{\frac{1}{2}}(\eta,\heta)$ measures the loss in terms of
	    $\ell_{\frac{1}{2}}$ and \citet{BartlettJordanMcAuliffe2006} used
	    $\ell^{0-1}=2\ell_{\frac{1}{2}}$.
    \item[Truncated Quadratic Loss] 
	Consider the margin loss $\phi(\hh)=(1+\hh\vee 0)^2=(2\heta\vee 0)^2$
	with link function $\hh(\heta)=2\heta-1$. 
	From (\ref{eq:w-in-terms-of-L}) we obtain $\minL(\eta)=4\eta(1-\eta)$
	and from (\ref{eq:B-w-bound-c-1-2}) the regret bound 
	$B(\eta,\heta)\ge 4\alpha^2$. These match the results presented by
    \citet{BartlettJordanMcAuliffe2006} when again it is noted we used
    $\ell_{\frac{1}{2}}$ and they used $\ell^{0-1}$.
\end{description}
The above results are for $c_0=\thalf$. Generalisations of margin losses to the
case of uneven weights are presented by \citet[Section 3.5]{Steinwart2008}.
Nevertheless, since the same $\phi$ function is still used for both components
of the loss (albeit with unequal
weights) such a scheme can still not capture the full generality of all proper
scoring rules in the manner achieved by the results in 
Section~\ref{section:surrogate-loss-bounds}.

%%%%%%%%%%%%%%%%%%%%%%%%%%%%%%%%%%%%%%%%%%%%%%%%%%%%%%%%%%%%%%%%%%%%%%%%%
\section{A Brief History of Pinsker Inequalties}
\label{section:history}
    \cite{Pinsker1964} presented the first bound relating $\mathrm{KL}(P,Q)$ to
    $V(P,Q)$: $\mathrm{KL}\ge V^2/2$ and it is now known by his name  or 
    sometimes as the Pinsker-Csisz\'{a}r-Kullback
    inequality since \cite{Csiszar1967} presented another version
    and 
    \cite{Kullback1967} showed $\mathrm{KL}\ge V^2/2 +V^4/36$. Much later 
    \cite{Topsoe2001} showed $\mathrm{KL}\ge V^2/2 + V^4/36+V^6/270$.
    Non-polynomial bounds are due to \cite{Vajda1970}: $\mathrm{KL}\ge
    L_{\mathrm{Vajda}}(V):=\log\left(\frac{2+V}{2-V}\right)-\frac{2V}{2+V}$ and
    \cite{Toussaint1978} who showed $\mathrm{KL}\ge L_{\mathrm{Vajda}}(V) \vee
    (V^2/2+V^4/36+V^8/288)$. 

    Care needs to be taken when comparing results from the literature as
    different definitions for the divergences exist. For example
    \cite{GibbsSu2002} use a definition of $V$ that differs by a factor of 
    2 from ours.
    There are some isolated bounds relating $V$ to some other divergences,
    analogous to the classical Pinkser bound; \cite{Kumar2005} have presented
    a summary as well as new bounds for a wide range of \emph{symmetric}
    $f$-divergences by making assumptions on the likelihood ratio: $r\le
    p(x)/q(x)\le R<\infty$ for all $x\in\mathcal{X}$. 
    This line of reasoning has also been
    developed by \cite{DragomirGluvsvcevicPearce2001,Taneja2005,Taneja2005a}. 
    \cite{Topsoe:2000}  has presented some infinite series  representations for
    capacitory discrimination in terms of triangular discrimination which lead
    to inequalities between those two divergences.
    \citet[p.48]{Liese2008} give the inequality $V\le h\sqrt{4-h^2} $
    (which seems to be originally due to \cite{LeCam1986})
    which when rearranged corresponds exactly to the bound for $h^2$ in theorem
    \ref{theorem:special-cases-pinsker}.
    \cite{Withers1999} has also presented some inequalities between other
    (particular) pairs of divergences; his reasoning is also in terms of
    infinite series expansions. 

    \cite{ArnoldMarkowichToscaniUnterreiter} considered the case of $n=1$
    but arbitrary $\II_f$ (that is they bound an arbitrary $f$-divergence in
    terms of the variational divergence). Their argument is similar to the
    geometric proof of Theorem \ref{theorem:general-pinsker}.
    They do not compute any of the explicit
    bounds in theorem \ref{theorem:special-cases-pinsker} except they state
    (page 243)  $\chi^2(P,Q)\ge V^2$ which is looser than 
    (\ref{eq:my-chi-squared-bound}).

    \cite{Gilardoni:2006} showed (via an intricate argument) that if
    $f'''(1)$ exists, then $\II_f \ge \frac{f''(1) V^2}{2}$. He also showed some
    fourth order inequalities of the form $\II_f\ge c_{2,f} V^2 + c_{4,f} V^4$
    where the constants depend on the behaviour of $f$ at 1 in a complex way.
    \cite{Gilardoni2006,Gilardoni2006a} presented a completely different
    approach which obtains many of the results of 
    theorem~\ref{theorem:special-cases-pinsker}.\footnote{We were unaware
    of these two papers until completing the results presented in the main
    paper. 
    }
    \cite{Gilardoni2006a} improved Vajda's bound  slightly to $
    \mathrm{KL}(P,Q) \ge \ln\frac{2}{2-V} - \frac{2-V}{2}\ln\frac{2+V}{2}$.

    \cite{Gilardoni2006,Gilardoni2006a}
    presented a general tight lower bound for $\II_f(P,Q)$ in terms of
    $V(P,Q)$ which is difficult to evaluate explicitly in general:
    \[
    \II_f \ge
    \frac{V}{2}\left(\frac{f[g_R^{-1}(k(1/V))]}{g_R^{-1}(k(1/V))-1} +
    \frac{f[g_L^{-1}(k(1/V))]}{1-g_L^{-1}(k(1/V))}\right),
    \]
    where $k^{-1}(t)=\frac{1}{2}\left(\frac{1}{1-g_L^{-1}(t)} +
    \frac{1}{g_R^{-1}(t)-1}\right)$, $k(u)=(k^{-1})^{-1}(u)$, 
    $g(u)=(u-1)f'(u)-f(u)$, $g_R^{-1}[g(u)]=u$
    for $u\ge 1$ and $g_L^{-1}[g(u)]=u$ for $u\le 1$. 
    He presented a new parametric form for $\II_f=\mathrm{KL}$ in terms of
    Lambert's $W$ function.
    In general, the
    result is analogous to that of
    \cite{FedotovTopsoe2003} in that it is in a parametric form which, 
    if one wishes
    to evaluate for a particular $V$, one needs to do a one dimensional
    numerical search --- as complex as (\ref{eq:my-pinsker-KL}).
    However, when $f$ is such that
    $\II_f$ is symmetric, this simplifies to the elegant form $\II_f\ge
    \frac{2-V}{2}f\left(\frac{2+V}{2-V}\right)-f'(1)V$. 
    He presented explicit special cases for $h^2$, $J$,$\Delta$ and $I$
    identical to the results in Theorem \ref{theorem:special-cases-pinsker}.
    It  is not apparent how 
    the approach of \cite{Gilardoni2006,Gilardoni2006a} could be extended to
    more general situations such as that in Theorem
    \ref{theorem:general-pinsker} (i.e.~$n>1$).

    Finally \cite{BolleyVillani2005} have
    considered {\em weighted} versions of the Pinsker inequalities (bounds for
    a weighted generalisation of Variational divergence) in terms of
    KL-divergence that are related to transportation inequalities.

%%%%%%%%%%%%%%%%%%%%%%%%%%%%%%%%%%%%%%%%%%%%%%%%%%%%%%%%%%%%%%%%
\section{Examples of extended convolution factorisation}
\label{section:examples}

    In this section we present three examples of $f$ which can be written as
    $f=g\Box g$.

    If $g(t)=(t-1)^2$ (corresponding to Pearson $\chi^2$ divergence),
    $(g\Box g)(\tau)=\inf_{x\in\reals^+} (x-1)^2 +\tau(x/\tau
    -1)^2$. Differentiating the right-hand side with respect to $x$, setting to
    zero and solving for $x$ gives $x=\frac{4}{2(1+1/\tau)}$. Substituting we
    obtain $(g\Box g)(\tau)=\frac{(\tau-1)^2}{\tau-1}$ which is the $f$ for
    $\Delta(P,Q)$, the triangular discrimination.

    If $g(t)=t\ln(t)$, a similar straightforward calculation yields
    $(g\Box g)(\tau)= \frac{-2\sqrt{\tau}}{e}$.

    If $g(t)=(\sqrt{t}-1)^2$ (corresponding to Hellinger divergence) then a
    similar calculation yields $(g\Box g)(\tau) = \frac{1}{2}
    (\sqrt{\tau}-1)^2=g(\tau)/2$. Thus this $g$ plays a role analogous to a
    gaussian kernel in ordinary convolution. The significance of this is
    unclear.

    We summarise the results (and the associated $g^\lf$)
    in the following table.

    \begin{center}
    \begin{tabular}{ccl}
	\hline
	$g(t)$ & $(g\Box g)(\tau)$ & $g^\lf(s)$\\[1mm]
	\hline
	$(t-1)^2 $ &  $\frac{(\tau-1)^2}{\tau-1}$ & $\frac{s^2}{4}+s$\\[1mm]
	$t\ln t$  & $\frac{-2\sqrt{\tau}}{e}$ & $e^{s-1}$\\[1mm]
	$(\sqrt{t}-1)^2$ & $\frac{1}{2}(\sqrt{\tau}-1)^2 $ &
	$\frac{s}{1-s}\test{s<1} + \infty\test{s\ge 1}$\\[1mm]
	\hline
    \end{tabular}
    \end{center}

    Whilst it is indeed straightforward to compute $(g\Box g)$ given $g$
    (although a simple closed form is not always possible), it is far from
    obvious how to go from a given $f$ to a $g$ such that $f=g\Box g$.

\citet[page 69]{Hiriart-Urruty1993} show that for $f$ convex on $\reals^+$, $g$
convex and increasing on $\reals^+$,
\[
(g \circ f)^\lf(s) = \inf_{\alpha > 0} \alpha f^\lf(\textstyle\frac{s}{\alpha} ) +
g^\lf(\alpha) = f^\lf \Box g^\lf.
\]
This illuminates the difficulty of the above ``factorisation problem''. It is
equivalent to: given a convex increasing $f^\lf$, find a convex increasing 
$g^\lf$ such that $f^\lf=g^\lf \circ g^\lf$.

%%%%%%%%%%%%%%%%%%%%%%%%%%%%%%%%%%%%%%%%%%%%%%%%%%%%%%%%%%%%%%%%%%%%
\section{Empirical Estimators of $V_{B_{\mathcal{H}},\frac{1}{2}}(P,Q)$ and
SVMs}
\label{sec:appendix-svm-mmd}

This appendix further develops the observations made in
Section~\ref{section:linear-loss} regarding the relationship between divergence
and risk when $\mathcal{R}=B_{\mathcal{H}}$,
a unit ball in a reproducing kernel Hilbert space $\mathcal{H}$. In contrast to
the rest of the paper (which focussed on relationships involving the underlying
distributions), in this appendix we will consider the practical situation where
there is only an empirical sample. We will see how the general results have
interesting implications for sample based machine learning algorithms.
    
If we require an empirical estimate of $V_{\mathcal{R},\pi}(P,Q)$ we can
replace $P$ and $Q$ by empirical distributions. We will use {\em weighted} 
empirical distributions. Given an independent identically distributed sample
$\mathbf{w}=(w_1,\ldots,w_m)\in\mathcal{X}^m$ the $\bm{\alpha}$-weighted
empirical distribution $\hat{P}_{\mathbf{w}}^{\bm{\alpha}}$  with respect to
$\mathbf{w}$ is defined by
\[
d\hat{P}_{\mathbf{w}}^{\bm{\alpha}} := \sum_{i=1}^m \alpha_i
\delta(\cdot-w_i)
\]
where $\bm{\alpha}=(\alpha_1,\ldots,\alpha_m)$, $\alpha_i\ge 0$, $i=1,\ldots,m$
and  $\sum_{i=1}^m \alpha_i=1$.
We will write $\hat{\EE}_{\mathbf{w}}^{\bm{\alpha}}\phi:=
\EE_{\hat{P}_{\mathbf{w}}^{\bm{\alpha}}}\phi=\sum_{i=1}^m\alpha_i\phi(w_i)$.
Thus
\[
V_{\mathcal{R},\frac{1}{2}}^2(\hat{P}_{\mathbf{w}}^{\bm{\alpha}},
\hat{P}_{\mathbf{z}}^{\bm{\beta}}) =
\frac{1}{2}\|\hat{\EE}_{\mathbf{w}}^{\bm{\alpha}}\phi
-\hat{\EE}_{\mathbf{z}}^{\bm{\beta}}\|_{\mathcal{H}}^2.
\]
Suppose now that $P$ and $Q$ correspond to the positive and negative class
conditional distributions. Let $\mathbf{x}:=(x_1,\ldots,x_m)$ be a sample drawn
from $M=\pi P+(1-\pi) Q$ with corresponding label vector
$\mathbf{y}=(y_1,\ldots,y_m)$. Let $I:=\{1,\ldots,m\}$, 
$I^+:=\{i\in I\colon y_i=1\}$,
$I^-:=\{i\in I\colon y_i=-1\}$. Consider a weight vector
$\bm{\alpha}=(\alpha_1,\ldots,\alpha_m)$ over the whole sample. Thus
\[
\hat{\EE}_P \phi = \sum_{i\in I^+} \alpha_i \phi(x_i)\ \ \ \mbox{and}\ \ \
\hat{\EE}_Q \phi = \sum_{i\in I^-} \alpha_i \phi(x_i)
\]
where we also require
\begin{equation}
    \sum_{i\in I^+}\alpha_i =\frac{m^+}{m}\ \ \ \mbox{and}\ \ \ \sum_{i\in I^-}
    \alpha_i = \frac{m^-}{m}
\end{equation}
and hence
\[
\sum_{i\in I} \alpha_i y_i = \frac{m^+ - m^-}{m}.
\]
Substituting into (\ref{eq:borgwardt}) we have
\begin{eqnarray}
    2 V_{B_{\mathcal{H}},\frac{1}{2}} (\hat{P},\hat{Q}) &=&
       \left\langle \hat{\EE}_P
	\phi - \hat{\EE}_Q \phi, \hat{\EE}_P \phi -
	\hat{\EE}_Q\phi\right\rangle\nonumber\\
    &=&\left\langle\sum_{i\in I^+}\alpha_i\phi(x_i)-\sum_{i\in I^-}\alpha_i
	\phi(x_i), \sum_{j\in I^+}\alpha_j\phi(x_j)-\sum_{j\in
	I^-}\alpha_j\right\rangle\nonumber\\
    &=& \left\langle\sum_{i\in I}\alpha_i y_i\phi(x_i), 
        \sum_{j\in I}\alpha_j y_j \phi(x_j)\right\rangle\nonumber\\
    &=& \sum_{i\in I}\sum_{j\in I} \alpha_i\alpha_j y_i y_j \langle
	\phi(x_i),\phi(x_j)\rangle\nonumber\\
    &=&\sum_{i\in I}\sum_{j\in I} \alpha_i\alpha_j y_i y_j k(x_i,x_j) =:
	J(\bm{\alpha},\mathbf{x}). \label{eq:J-alpha-x-def}
\end{eqnarray}
We now consider three different choices of $\bm{\alpha}$.

{\bf Uniform weighting} If we set $\alpha_i=\frac{1}{m}$, $i=1,\ldots,m$,
	then (\ref{eq:J-alpha-x-def}) becomes
	\[
	\frac{1}{m^2}\sum_{i,j\in I} y_i y_j k(x_i,x_j) =
	\mathrm{MMD}_b^2[B_{\mathcal{H}},\bm{x}^+, \bm{x}^- ]
	\]
	where $\bm{x}^+:=(x_i)_{i\in I^+}$, $\bm{x}^-:=(x_i)_{i\in I^-}$ 
	and $\mathrm{MMD}_b$ is the biased estimator of the {\em Maximum Mean
	Discrepancy} \citep{GrettonBorgwardtRaschScholkopfSmola2008},  
	an alternate name for  $V_{\mathcal{R}}$.
	Observe that from theorem \ref{theorem:linear-loss-equivalence}, this
	case corresponds to using a Fisher linear discriminant in feature space
	\citep{DevGyoLug96} when it is assumed that the 
	within-class covariance matrices are 
	both the identity matrix. This follows by observing that the
	constructed hypothesis is identical in both cases.

 {\bf Pessimistic Weighting} Instead of weighting each sample equally, one
	can optimise over $\bm{\alpha}$. By theorem 
	\ref{theorem:linear-loss-equivalence}, minimizing
	$J(\bm{\alpha},\mathbf{x})$ over $\bm{\alpha}$ will maximize
	$\minLL^{\mathrm{lin}}$ and is thus the most pessimistic choice.
	Explicitly, we have 
	\begin{eqnarray}
	    \min_{\bm{\alpha}} & &
	     \sum_{i=1}^m \sum_{i=1}^m \alpha_i \alpha_j
	     y_i y_j k(x_i,x_j)\label{eq:svm-objective}\\
		\mbox{s.t.}\ & & \alpha_i\ge 0,\ \
		i=1,\ldots,m\label{eq:svm-constraint-1}\\
		& & \sum_{i=1}^m \alpha_i y_i =
		\frac{m^+-m^-}{m}\label{eq:svm-constraint-2}\\
		& & \sum_{i=1}^m \alpha_i =1\label{eq:svm-constraint-3}
	\end{eqnarray}
	which can be recognized as the support vector machine 
	\citep{CortesVapnik1995}. The SVM uses the sign of the ``witness''
	\citep{GrettonBorgwardtRaschScholkopfSmola2008},  
	$
	    x\mapsto \sum_{i=1}^m \alpha_i y_i k(x_i,x)
	$
	as its predictor.

    {\bf Interpolation between above two cases} A parametrized interpolation
	between the above two cases can be constructed by the addition of the
	constraints
	\begin{equation}
	    \alpha_i \le \frac{1}{\nu m}, \ \ i=1,\ldots,m,
	    \label{eq:nu-alpha-constraint}
	\end{equation}
	where $\nu\in(0,1]$ is an adjustable parameter. Observe that $\nu$
	controls the sparsity of $\bm{\alpha}$ since
	(\ref{eq:nu-alpha-constraint}), (\ref{eq:svm-constraint-1}) and
	(\ref{eq:svm-constraint-3}) together
	imply that $|\{i\in I\colon \alpha_i\ne 0\}|\ge \nu m$.
	\cite{CriBur00} have shown that
	(\ref{eq:svm-objective}),$\ldots$,(\ref{eq:nu-alpha-constraint}) 
	is equivalent to the $\nu$-SVM algorithm \citep{SchSmoWilBar00}.

While ``information-theoretic'' approaches to the SVM and weighted kernel
representations are hardly new\footnote{
    The use of kernel representations for classification is of course not new:
    from the classical kernel classifier (where $\alpha_i=1/m$ for all $i\in
    I$)~\cite[Chapter 10]{DevGyoLug96} to the Generalised
    Portrait~\citep{AizBraRoz64}, the Generalised
    Discriminant~\citep{BaudatAnouar2000} and the panoply of techniques inspired
    by Support Vector Machines~\citep{ScholkopfSmola2002,Herbrich2002}. None of
    these techniques is designed from the perspective of minimising a
    $f$-divergence.  

    \citet{PriXuFis99b} have developed an approach to machine learning problems
    based on information theoretic
    criteria~\citep{PriXuZhaFis00,JenErdPriElt04,XuErdJenPri05,
    Jen05,JenErdPriElt06,PaiXuPri06}.  \citet{JenErdPriElt04,JenErdPriElt06}
    considered kernel methods from the perspective of
    Renyi's quadratic entropy.  They do not exploit the formal 
    relationship between
    maximising divergence and minimising risk.  They interpret the SVM as being
    constructed from weighted Parzen windows density estimates.
    \cite{GrettonBorgwardtRaschScholkopfSmola2008} explained the relationship
    between their MMD estimators and those derived from (unweighted) Parzen
    windows estimates of the class-conditional distributions.
    Weighted Parzen windows estimates were used as a basis for building a
    classifier by \citet{BabichCamps1996}. Weighted empirical distributions are
    widely used in particle filtering \citep{CrisanDoucet2002}.

    \citet{McDermottKatagiri02} considered the direct
    optimisation of a classifier built on top of Parzen windows density
    estimates.  They showed that the minimum classification error criterion is
    equivalent to a Parzen windows estimate of the theoretical  Bayes risk.
    They re-derive the traditional approach of minimising an estimate of the
    expected loss. \cite{McDermottKatagiri03} extended their approach to the
    multi-class setting in a way that takes account of all the ``other''
    classes better in estimating the probability of error of a given class.

}, the results
presented here are novel and provide a simple and direct derivation of the SVM
via the generalised variational divergence.

If $V_{B_{\mathcal{H}},\frac{1}{2}}(\hat{P}_{\bm{w}},\hat{Q}_{\bm{z}})$ is used
as a test statistic to infer whether two samples $\bm{w}$ and $\bm{z}$ are
drawn from the same distribution (as
\citet{GrettonBorgwardtRaschScholkopfSmola2008} do), then when the distributions
from which $\bm{w}$ and $\bm{z}$ are drawn are close, the classification performance of
the corresponding classifier (i.e.~the classifier that uses the sign of the
witness function) will be close to the worst possible. Thus one will be
operating in a regime distinct from the normal situation, where the risk is
typically small.

Finally observe that the derivation of the SVM presented here could be viewed
as an application of an alternate ``inductive principle''  --- a general recipe 
for constructing
learning algorithms from learning task specification 
\citep{Vapnik1989,vapnik2006edb}.
The traditional Empirical Risk Minimization principle entails replacing 
$(P,Q)$ with $(\hat{P}_{\bm{x}^+},\hat{Q}_{\bm{x}^-})$ in the definition 
of $\minLL(\pi,P,Q)$. Then, in order to not overfit, one restricts the class of
functions from which hypotheses are drawn. That is, there are two
approximations:
\[
\minLL(\pi,P,Q) \ \ 
\underrightarrow{\mbox{{\tiny \ Empirical Approximation\ (uniform) \ }}}\ \ 
    \minLL(\pi,\hat{P}_{\bm{x}^+}, \hat{Q}_{\bm{x}^-}) \ \ 
\underrightarrow{\mbox{\tiny\ Restrict Class\ \ }}\ \ 
    \minLL_{\mathcal{R}}(\pi,\hat{P}_{\bm{x}^+}, \hat{Q}_{\bm{x}^-})  .
\]
Upon setting $\bm{\alpha}^+=(\alpha_i)_{i\in I^+}$ and
$\bm{\alpha}^-=(\alpha_i)_{i\in I^-}$, the derivation presented above, 
in contrast, can be summarised schematically by
\[
\mbox{``$\minLL(\pi,P,Q)$''} \ \ 
\underrightarrow{\mbox{\tiny\ Restrict Class\ \ }}\ \ 
    \minLL_{\mathcal{R}}(\pi,P,Q) \ \ 
    \underrightarrow{\mbox{{\tiny \ Empirical Approximation\ 
    ($\bm{\alpha}$-weighted) \ }}}\ \ 
\minLL_{\mathcal{R}}(\pi,\hat{P}_{\bm{x}^+}^{\bm{\alpha}^+}, 
\hat{Q}_{\bm{x}^-}^{\bm{\alpha}^-}) ,
\]
where a different loss (the ``linear'' loss)  was used at the start. 
With that loss function,
reversing the order of the two approximations would not work, and is
(thus) not equivalent to the ERM inductive principle. The first step makes
$\minLL$ well defined --- with no restriction it is not, hence the quotes; and will avoid overfitting in any case. The second step
is the more general ($\bm{\alpha}$-weighted) empirical approximation.

%%%%%%%%%%%%%%%%%%%%%%%%%%%%%%%%%%%%%%%%%%%%%%%%%%%%%%%%%%%%%%%%%%%%%%%%
%\small
\setlength{\bibsep}{1.3mm}\baselineskip=4.1mm
\bibliography{reductions}
\end{document}